\documentclass[twoside,11pt]{article}

\usepackage{jmlr2e}

\usepackage{bm}
\usepackage{url}
\usepackage{bbm}
\usepackage{tikz}
\usepackage{array}
\usepackage{dsfont}
\usepackage{natbib}
\usepackage{epsfig}
\usepackage{amsmath}
\usepackage{amssymb}
\usepackage{caption}
\usepackage{physics}
\usepackage{wrapfig}
\usepackage{mdframed}
\usepackage{multirow}
\usepackage{scalerel}
\usepackage{placeins}
\usepackage{booktabs}
\usepackage{rotating}
\usepackage{amsfonts}
\usepackage{nicefrac}
\usepackage{graphicx}
\usepackage{xr-hyper}
\usepackage{hyperref}
\usepackage{microtype}
\usepackage{subcaption}
\usepackage[abs]{overpic}
\usepackage[]{algpseudocode}
\usepackage[export]{adjustbox}
\usepackage[ruled]{algorithm2e}
\usepackage[T1]{fontenc}
\usepackage{geometry}
\usepackage{thmtools}
\usepackage{thm-restate}
\usepackage{cleveref}
\declaretheorem[name=Definition,numberwithin=section]{defn}
\declaretheorem[name=Lemma,numberwithin=section]{lem}

\numberwithin{equation}{section}

\definecolor{tableau_blue}{HTML}{1F77B4}
\definecolor{tableau2}{HTML}{FF7F0E}
\definecolor{tableau_green}{HTML}{2CA02C}
\definecolor{tableau_red}{HTML}{D62728}
\definecolor{tableau5}{HTML}{9467BD}
\definecolor{tableau6}{HTML}{8C564B}
\definecolor{tableau7}{HTML}{CFECF9}
\definecolor{tableau8}{HTML}{7F7F7F}
\definecolor{tableau9}{HTML}{BCBD22}
\definecolor{tableau10}{HTML}{17BECF}

\newtheorem{exmp}[theorem]{Example}

\newcommand{\CCM}{\mathcal{CCM}}

\newcommand{\zero}{\mathbf{0}}
\newcommand{\one}{\mathbbm{1}}
\newcommand{\onet}{\tilde{\one}}
\newcommand{\T}{\top}

\newcommand{\prevv}{\text{prev}}
\newcommand{\nextt}{\text{next}}
\newcommand{\nvec}{n_\text{vec}}

\DeclareMathOperator*{\Ave}{Ave }

\DeclareMathOperator{\vect}{vec}

\DeclareMathOperator{\diag}{diag}
\DeclareMathOperator{\blkdiag}{blkdiag}
\DeclareMathOperator{\Exp}{\mathbb{E}}

\newcommand{\Hess}{\text{Hess}}

\newcommand{\R}{{\mathbb R}}
\newcommand{\Lagr}{\mathcal{L}}

\renewcommand{\v}{\boldsymbol{v}}

\newcommand{\x}{\boldsymbol{x}}

\newcommand{\y}{\boldsymbol{y}}
\newcommand{\h}{\boldsymbol{h}}
\newcommand{\W}{\boldsymbol{W}}

\newcommand{\Zb}{\boldsymbol{Z}}
\newcommand{\Eb}{\boldsymbol{E}}
\newcommand{\U}{\boldsymbol{U}}
\newcommand{\Ub}{\bar{\U}}
\newcommand{\V}{\boldsymbol{V}}
\newcommand{\Q}{\boldsymbol{Q}}

\newcommand{\z}{\boldsymbol{z}}
\newcommand{\e}{\boldsymbol{e}}
\newcommand{\p}{\boldsymbol{p}}
\newcommand{\g}{\boldsymbol{g}}
\newcommand{\w}{\boldsymbol{w}}
\newcommand{\f}{\boldsymbol{f}}
\newcommand{\vv}{\boldsymbol{v}}
\newcommand{\ww}{\boldsymbol{w}}

\newcommand{\Tb}{\boldsymbol{T}}
\newcommand{\Y}{\boldsymbol{Y}}
\newcommand{\G}{\boldsymbol{G}}
\renewcommand{\H}{\boldsymbol{H}}
\newcommand{\E}{\boldsymbol{E}}
\newcommand{\A}{\boldsymbol{A}}
\newcommand{\I}{\boldsymbol{I}}
\newcommand{\X}{\boldsymbol{X}}
\newcommand{\B}{\boldsymbol{B}}

\newcommand{\C}{\boldsymbol{C}}
\newcommand{\D}{\boldsymbol{D}}

\newcommand{\KFAC}{\G_{\text{KFAC}}}
\newcommand{\CFAC}{\G_{\text{CFAC}}}

\newcommand{\VC}{\V_\text{class}}
\newcommand{\VCP}{\V_\text{cross}}
\newcommand{\VW}{\V_\text{within}}

\newcommand{\GC}{\G_\text{class}}
\newcommand{\GCP}{\G_\text{cross}}
\newcommand{\GCPc}{\G_{\text{cross},c}}

\newcommand{\GW}{\G_\text{within}}
\newcommand{\GWc}{\G_{\text{within},c}}
\newcommand{\GWccp}{\G_{\text{within},c,c'}}

\newcommand{\HC}{\H_\text{class}}
\newcommand{\HW}{\H_\text{within}}

\newcommand{\WC}{\W_\text{class}}
\newcommand{\WW}{\W_\text{within}}

\newcommand{\DeltaC}{\Deltaa_\text{class}}
\newcommand{\DeltaCP}{\Deltaa_\text{cross}}

\newcommand{\DeltaW}{\Deltaa_\text{within}}

\newcommand{\Deltaa}{\boldsymbol{\Delta}}

\newcommand{\deltaa}{\boldsymbol{\delta}}
\newcommand{\thetaa}{\boldsymbol{\theta}}

\newlength\mylen
\newcommand\myinput[1]{%
  \settowidth\mylen{\KwIn{}}%
  \setlength\hangindent{\mylen}%
  \hspace*{\mylen}#1\\}
\newcommand\myresult[1]{%
  \settowidth\mylen{\KwResult{}}%
  \setlength\hangindent{\mylen}%
  \hspace*{\mylen}#1\\}

\jmlrheading{?}{?}{?}{?}{?}{papyan20b}{Vardan Papyan}

\ShortHeadings{Class/Cross-Class Structure Pervade Deep Learning}{Papyan}
\firstpageno{1}

\begin{document}

\title{Traces of Class/Cross-Class Structure \\ Pervade Deep Learning Spectra}

\author{\name Vardan Papyan \email papyan@stanford.edu \\
\addr Department of Statistics\\
Stanford University\\
Stanford, CA 94305, USA}

\editor{?}

\maketitle

\begin{abstract}
Numerous researchers recently applied empirical spectral analysis to the study of modern deep learning classifiers. We identify and discuss an important formal \textit{class/cross-class structure} and show how it lies at the origin of the many visually striking features observed in deepnet spectra, some of which were reported in recent articles, others are unveiled here for the first time. These include spectral outliers, ``spikes'', and small but distinct continuous distributions, ``bumps'', often seen beyond the edge of a ``main bulk''.

The significance of the cross-class structure is illustrated in three ways: (i) we prove the ratio of outliers to bulk in the spectrum of the Fisher information matrix is predictive of misclassification, in the context of multinomial logistic regression; (ii) we demonstrate how, gradually with depth, a network is able to separate class-distinctive information from class variability, all while orthogonalizing the class-distinctive information; and (iii) we propose a correction to KFAC, a well-known second-order optimization algorithm for training deepnets.
\end{abstract}





\section{Introduction}
\subsection{Empirical measurements of deepnet spectra}
Recently there has been a surge of interest in measuring the spectra associated with deep classifying neural networks. \citet{lecun2012efficient}, \citet{dauphin2014identifying} and \citet{sagun2016eigenvalues,sagun2017empirical} measured the eigenvalues of the Hessian of the parameters averaged over the training data. They plotted histograms of eigenvalues and observed a bulk, together with a few large outliers. We define these somewhat informally (see also Figure \ref{fig:pattern}):
\begin{mdframed}
\begin{defn}\textup{{\textbf{(Bulk).}}}
    A collection of eigenvalues which, when displayed in a histogram form, seemingly follows a continuous distribution.
\end{defn}
\end{mdframed}
\begin{mdframed}
\begin{defn}\textup{{\textbf{(Outliers).}}}
    A collection of eigenvalues, each individually isolated away from the other eigenvalues.
\end{defn}
\end{mdframed}
\noindent Crucially, \citet{sagun2016eigenvalues,sagun2017empirical} observed that the number of outliers in the spectrum of the Hessian is often equal to the number of classes $C$. Their observation was supported by \citet{gur2018gradient} who noticed that the eigenvectors corresponding to these $C$ outliers span approximately the gradients of stochastic gradient descent (SGD). \citet{papyan2019measurements} developed a rigorous attribution methodology which attributed these $C$ outliers to $C$ unsubtracted class means of gradients. \citet{fort2019emergent} alluded to yet another related phenomenon--training deepnets is successful even when confined to low dimensional subspace of parameters \citep{li2018measuring,jastrzkebski2018relation, fort2019large, fort2019goldilocks}.

\citet{sagun2017empirical} experimented with: (i) two-hidden-layer networks, with $30$ hidden units each, trained on synthetic data sampled from a Gaussian Mixture Model data; and (ii) one-hidden-layer networks, with $70$ hidden units, trained on MNIST. Their exploration was limited to architectures with \textbf{thousands} of parameters--orders of magnitude smaller than state-of-the-art architectures such as VGG by \citet{simonyan2014very} and ResNet by \citet{he2016deep} that have \textbf{tens of millions, hundreds of millions or close to a billion parameters} \citep{mahajan2018exploring}. In the absence of other deeper insights, phenomena observed in such small-scale `academic' examples could not be expected to persist in large-scale real-world examples. In the last year, it became possible to study spectra of deepnet Hessians at full-scale. \citet{papyan2018full} used this to observe  that the patterns seen in previous small-scale examples persist even in state-of-the-art deepnets.

In parallel, \citet{ghorbani2019investigation} studied the evolution of the full spectrum throughout the epochs of SGD, investigating the effects of skip connections, batch normalization and learning rate drops on properties of outliers and bulk. In addition to studying the spectrum of the full Hessian, \citet{li2019hessian} measured the spectrum of the layer-wise Fisher Information Matrix (FIM). They observed a bulk-and-outliers structure, and a closer inspection of their results shows that there is in fact more than just one bulk. \citet{jastrzebski2020break}, in addition to studying the spectrum of the Hessian, also studied the spectrum of the covariance of gradients, observing a bulk-and-outliers structure in both cases. They showed how SGD hyperparameters affect the magnitude of the spectral norm and the condition number of both matrices.

There were also measurements in the literature of quantities other than the Hessian, FIM, and covariance of gradients. \citep{martin2018implicit,mahoney2019traditional} measured extensively the spectrum of deepnet weights throughout the layers. Their plots sometimes show a set of outliers isolated from a bulk and closer inspection suggests occasionally the presence of another small bulk beyond the main bulk. \citet{verma2018manifold} proposed a novel regularization scheme for training deepnets and investigated its effect on the spectra of features, which they show exhibit a bulk-and-outliers structure. \citet{oymak2019generalization} measured the spectrum of the backpropagated errors and showed again a set of outliers isolated from a bulk.

\subsection{Initial theoretical studies}
Mathematically oriented researchers tried to leverage Random Matrix Theory (RMT) to generate features similar to the ones observed in practice and study them. \citet{pennington2017geometry} decomposed the Hessian into two components, the FIM $\G$ and a residual $\E$, assumed that the eigenvalues of $\G$ are distributed according to the Marchenko-Pastur law and those of $\E$ according to the semi-circle law, and studied the predicted spectrum of the Hessian. \citet{pennington2018spectrum} calculated the Stieltjes transform of the spectral density of the FIM for a single hidden layer neural network with squared loss and normally distributed weights and inputs. \citet{Granziol2019} studied the deviation of the train Hessian from the population Hessian, as a function of the ratio of sample size to number of parameters. They assumed the spectrum of the Hessian has a bulk, originating from the Gaussian Orthogonal Ensemble, with several outliers.

The loss surface of deepnets changes depending on the width of the network \citep{geiger2019disentangling,geiger2020scaling}. Mathematically oriented researchers therefore tried to leverage large-width limits to prove claims about deepnet spectra. \citet{karakida2019universal,karakida2019pathological} calculated the mean, variance and maximum of the FIM eigenvalues. \citet{dyer2019asymptotics} and \citet{andreassen2020asymptotics} used Feynman diagrams to study the training dynamics of SGD, calculating the spectra of the Hessian and the Neural Tangent Kernel (NTK) \citep{jacot2018neural}. \citet{jacot2019asymptotic} calculated the moments of the Hessian throughout training and showed how the FIM and $\E$ are asymptotically mutually orthogonal.

At the present time, existing spectral measurements display a wide variety of features (bulk shapes, outliers, secondary mini-bulks, etc.). It seems fair to say that existing theoretical studies reproduce certain of these features. However, the connections between formal analysis and observed features are so far incomplete. In fact, it is an ongoing activity to propose generative models exhibiting the different observed phenomena.

\subsection{Why are many researchers measuring deepnet spectra?}
In doing spectral analysis of each of these fundamental objects--Hessian, FIM, features, backpropagated errors, and weights--researchers hope to gain deeper insights into deepnet behaviour. Many researchers believe that such spectral features, once better understood, will provide clues to improvements in deep learning training or classifier performance \citep{lecun1998efficient,dauphin2014identifying,sagun2016eigenvalues,sagun2017empirical,gur2018gradient,ghorbani2019investigation,yao2019pyhessian}.

\subsection{Open questions}
Our goal in this work is the answer the following fundamental questions:
\begin{mdframed}
\begin{enumerate}
    \item[] \textbf{Cause attribution:} Can we say what causes outliers, mini-bulk(s) and bulk(s) in various deepnet spectra? Can we explain the number of eigenvalue outliers? Can we explain why the largest outlier is much farther out?
    \item[] \textbf{Ubiquity:} Why are these patterns pervasive in spectra across a variety of deepnets and variety of objects (features, backpropagated errors, gradients, weights, FIM, Hessian)?
    \item[] \textbf{Significance:} Are these patterns mere artifacts or do they convey meaningful clues? If meaningful, how can we best use the hints they give?
\end{enumerate}
\end{mdframed}

\subsection{Insights from three-level hierarchical structure}
In previous work, \cite{papyan2019measurements} introduced a three-level hierarchical structure for deepnet gradients. He then introduced its connection to some of the spectral patterns in the FIM mentioned above. This work shows how this three-level hierarchical structure can be utilized to explain the spectra of all the fundamental quantities in deep learning, not just the gradients.

\subsection{A pattern covering all cases}
We now make clear the main spectral features we will be discussing and explaining, through a pattern schematized in Figure \ref{fig:pattern}. This pattern applies to any of the spectral settings mentioned earlier or any of the several new settings to be discussed below. The pattern consists of:
\begin{itemize}
    \item A bulk;
    \item $C(C-1)$ eigenvalue outliers, i.e., eigenvalues outside the main bulk (for the sake of brevity we will refer to them as $C^2$ outliers);
    \item $C-1$ eigenvalue outliers situated at still larger amplitudes; and
    \item A single isolated outlier larger still.
\end{itemize}
The $C^2$ outliers may appear either as separated spikes, or alternatively as what we call a {\it mini-bulk}: an approximately continuous distribution rather than a series of separated spikes. We emphasize the schematic nature of the above description; the exact appearance of a spectral plot will differ from situation to situation.

\begin{figure}[t]
    \centering
    \includegraphics[trim=3.5cm 4.5cm 3.25cm 4.5cm,clip,width=0.8\textwidth]{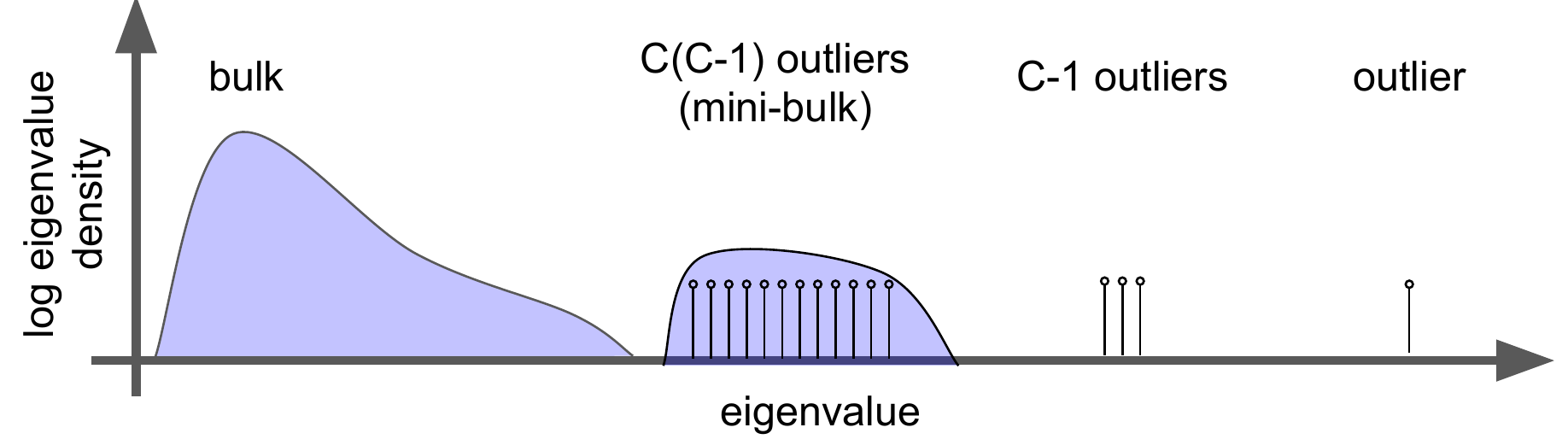}
    \caption{\textbf{Schematic typical spectrum.} In the above schematic, with $C{=}4$ classes, we see the presence of one isolated outlier on the far right, of $C-1$ secondary outliers that are less separated, of a mini-bulk consisting of $C(C-1) \approx C^2$ outliers, and of a ``main bulk'' on the left. These important spectral features are explained further in the body of the text.}
    \label{fig:pattern}
\end{figure}

\subsection{Class block structure}
Our first goal in this work is to explain what causes this ubiquitous pattern to emerge. To this end, we need the following definitions.

\begin{mdframed}
\begin{defn}\textup{{\textbf{(Class block structure).}}}
    An array of vectors $\{ \v_I \}_I$ exhibits a (balanced\footnote{Imbalanced structure would result if different classes $c$ had different numbers of examples per class $N_c$. We only study the balanced case, where $N_c = N_{c'}, \ \forall c,c'$.}) \textit{class block structure} when the indices have the form $I=(i, c)$, where $1 \leq i \leq N$ runs across the indices of examples in a certain class, and $1 \leq c \leq C$ runs across the class indices.
\end{defn}
\end{mdframed}
\begin{exmp}[\textbf{Training examples}]
    \textit{Training examples} in standard class-balanced machine learning datasets such as CIFAR10 exhibit class block structure. For example, CIFAR10 has a total of $50000$ training vectors, which include $5000$ examples in class `cat', $5000$ examples in class `dog', etc. We denote the $i$'th example in the $c$'th class by $\x_{i,c}$.
\end{exmp}
Let $f(\cdot)$ be some fixed function. Consider an array $\{ \v_{i,c} \}_{i,c}$, where $\v_{i,c} = f(\x_{i,c})$. Such an array inherits the class block structure from the train examples $\x_{i,c}$.
\begin{exmp}[\textbf{Features}]
    Consider the post-activations (also called features) at some fixed layer $l$ of a deepnet. They exhibit a class block structure. Indeed, they are functions of the examples and hence they inherit the class organization. The concatenation of such features across the layers also exhibits such structure. We denote the $l$'th layer features of $\x_{i,c}$ by $\h_{i,c}^l$ and their cross-layer concatenation by $\h_{i,c}$.
\end{exmp}
\begin{exmp}[\textbf{Gradients}]
    The gradients of the loss $\Lagr(\thetaa)$ with respect to the parameters of the model $\thetaa$,
    \begin{equation*}
        \pdv{\ell( f(\x_{i,c};\thetaa), \y_c )}{\thetaa},
    \end{equation*}
    inherit the class block structure from the examples $\x_{i,c}$.
\end{exmp}

\subsection{Cross-class block structure}
Assume we are using cross-entropy loss,
\begin{equation*}
    \Lagr(\theta) = \Ave_{i,c,c'} \Lagr_{i,c,c'} = - \Ave_{i,c,c'} \ y_{i,c,c'} \log(p_{i,c,c'}),
\end{equation*}
where $p_{i,c,c'}$ is the probability under the `logistic' or `softmax' model, that the $i$'th example in the $c$'th class belongs to $c'$. Similarly, $y_{i,c,c'}$ is the ground truth probability of the $i$'th example in the $c$'th class belonging to $c'$, which is equal to the Kronecker delta function, $\delta_{c=c'}$. The interpretation of the second subscript ($c'$) is different than that of the first subscript ($c$). The first denotes the actual class of that observation; while the second denotes the classes enumerated in applying the cross-entropy loss. In what follows, $c'$ will generally denote such a cross-entropy class, or cross-class. $c'$ generally represents a would-be class, as distinguished from $c$, the actual observed label class.

\begin{mdframed}
\begin{defn}\textup{{\textbf{(Cross-class block structure).}}}
    An $N \times C \times C$ array of vectors  $\v_{i,c,c'}$ exhibits a (balanced) \textit{cross-class block structure} when it is indexed by a three-tuple, $(i, c, c')$, where $1 \leq i \leq N$ runs across the indices of examples in a certain class, $1 \leq c \leq C$ runs across the \textit{class indices}, and $1 \leq c' \leq C$ runs across the \textit{cross-class indices}.
\end{defn}
\end{mdframed}
\begin{exmp}[\textbf{Losses}]
    The losses $\Lagr_{i,c,c'}$ exhibit cross-class structure.
\end{exmp}
\begin{exmp}[\textbf{Extended gradients}]
    The ``ordinary'' gradients of the loss, associated with an example $\x_{i,c}$, have the form:
    \begin{equation*}
        \g_{i,c,c} = \pdv{\ell( f(\x_{i,c};\thetaa), \y_c )}{\thetaa} = \pdv{f(\x_{i,c};\thetaa)}{\thetaa}^\T (\p_{i,c} - \y_c);
    \end{equation*}
    they exhibit class block structure but \textbf{not} cross-class block structure. We define an \textbf{extended gradient}, denoted $\g_{i,c,c'}$, which \textbf{does} exhibit cross-class block structure. In its definition, we replace in the above equation the actual observed one-hot vector $\y_c$ with a counterfactual one-hot vector corresponding to a would-be observation $\y_{c'}$:
    \begin{equation*}
        \g_{i,c,c'} = \pdv{\ell( f(\x_{i,c};\thetaa), \y_{c'} )}{\thetaa} = \pdv{f(\x_{i,c};\thetaa)}{\thetaa}^\T (\p_{i,c} - \y_{c'}).
    \end{equation*}
    Later we will see that the Fisher Information Matrix is a (weighted) second moment of extended gradients.
\end{exmp}
\begin{exmp}[\textbf{Backpropagated errors}]
    The derivative of the loss with respect to the output of some fixed layer $l$, also known as the backpropagated error, provides yet another array of vectors exhibiting cross-class block structure, provided we consider derivatives associated with all possible cross-class labels. The concatenation of backpropagated errors across layers also exhibits such structure. The $l$'th layer backpropagated error induced by example $\x_{i,c}$ will be denoted $\deltaa_{i,c,c'}^l$; the cross-layer concatenation will be denoted $\deltaa_{i,c,c'}$.
\end{exmp}

\subsection{Global mean, \texorpdfstring{$C$}{C} class means and \texorpdfstring{$C^2$}{C2} cross-class means}
Arrays exhibiting class/cross-class block structure permit various averages to be compactly expressed.
\begin{mdframed}
\begin{defn}\textup{{(\textbf{$C^2$ cross-class means).}}}
    For an array of vectors $\{ \v_{i,c,c'} \}_{i,c,c'}$ exhibiting cross-class structure, we denote their $C^2$ \textit{cross-class means} by $\{ \v_{c,c'} \}_{c,c'}$; they are obtained by averaging, for a fixed class $c$, and cross-class $c'$, across the replication index $i$, i.e.,
    \begin{equation*}
        \v_{c,c'} = \Ave_i \v_{i,c,c'}.
    \end{equation*}
\end{defn}
\end{mdframed}
\begin{exmp}
    Denote the $C^2$ cross-class means of gradients by $\{ \g_{c,c'} \}_{c,c'}$.
\end{exmp}
\begin{exmp}
    Denote the $C^2$ cross-class means of the $l$'th layer backpropagated errors by $\{ \deltaa_{c,c'}^l \}_{c,c'}$, and their layer-wise concatenation by $\{ \deltaa_{c,c'} \}_{c,c'}$.
\end{exmp}

\begin{mdframed}
\begin{defn}\textup{{\textbf{($C$ class means).}}}
    For an array of vectors $\{ \v_{i,c,c'} \}_{i,c,c'}$ exhibiting cross-class block structure, we denote their $C$ \textit{class means} by $\{ \v_c \}_c$; each is obtained by averaging, for a fixed class $c$, the cross-class means associated with that class, i.e.,
    \begin{equation} \label{def:v_c}
        \v_c = \Ave_{c'} \v_{c,c'} = \Ave_{i,c'} \v_{i,c,c'}.
    \end{equation}
    Moreover, an array $\{ \v_{i,c} \}_{i,c}$ exhibiting class block structure has $C$ \textit{class means}, $\{ \v_c \}_c$; each is obtained by averaging, for a fixed class $c$, across the replication index $i$, i.e.,
    \begin{equation*}
        \v_c = \Ave_i \v_{i,c}.
    \end{equation*}
\end{defn}
\end{mdframed}
\begin{exmp}
    Denote the $C$ feature class means at layer $l$ by $\{ \h_c^l \}_c$ and their layer-wise concatenation by $\{ \h_c \}_c$.
\end{exmp}
\begin{exmp}
    Denote the $C$ backpropagated error class means at layer $l$ by $\{ \deltaa_c^l \}_c$ and their layer-wise concatenation by $\{ \deltaa_c \}_c$.
\end{exmp}
\begin{exmp}
    Denote the $C$ class means of gradients by $\{ \g_c \}_c$.
\end{exmp}

\begin{mdframed}
\begin{defn}\textup{{\textbf{(Global mean).}}}
    An array of vectors exhibiting class/cross-class block structure has a \textit{global mean}, given by
    \begin{equation*}
        \v_G = \Ave_c \v_c.
    \end{equation*}
\end{defn}
\end{mdframed}

\subsection{Second moment matrices and covariances in the class/cross-class structure}
It is very natural to express second moment matrices for arrays with class/cross-class structure:

\begin{mdframed}
\begin{defn}\textup{{\textbf{(Second moment matrix).}}}
    An array of vectors exhibiting class or cross-class block structure with $D$-dimensional vectors has a \textit{second moment matrix} $\V \in \R^{D \times D}$ given by
    \begin{equation*}
        \V = \Ave_{i,c} \v_{i,c} \v_{i,c}^\T,
    \end{equation*}
    or
    \begin{equation*}
        \V = \Ave_{i,c,c'} \v_{i,c,c'} \v_{i,c,c'}^\T,
    \end{equation*}
    respectively.
\end{defn}
\end{mdframed}

\begin{mdframed}
\begin{defn}\textup{{\textbf{(Second moment of global mean).}}}
    Associated with an array of vectors $\{\v_{i,c,c'}\}_{i,c,c'}$ exhibiting class/cross-class structure is the \textit{second moment matrix of the global mean},
    \begin{equation*}
        \v_G \v_G^\T.
    \end{equation*}
\end{defn}
    \end{mdframed}

\begin{mdframed}
\begin{defn}\textup{{\textbf{(Between-class second moment).}}}
    Associated with an array of vectors $\{\v_{i,c,c'}\}_{i,c,c'}$ exhibiting class/cross-class structure is the \textit{between-class second moment}, $\VC \in \R^{D \times D}$, 
    \begin{equation*}
        \VC = \Ave_c \v_c \v_c^\T.
    \end{equation*}
\end{defn}
\end{mdframed}

\begin{mdframed}
\begin{defn}\textup{{\textbf{(Between-cross-class covariance).}}}
    The \textit{between-cross-class covariance}, $\VCP \in \R^{D \times D}$ associated with an array of vectors $\{\v_{i,c,c'}\}_{i,c,c'}$ exhibiting cross-class structure, is given by
    \begin{equation*}
        \VCP = \Ave_{c,c'} \z_{c,c'} \z_{c,c'}^\T,
    \end{equation*}
    where the \textit{cross-class mean deviations} $\z_{c,c'} \in \R^p$ are defined as follows:
    \begin{equation*}
        \z_{c,c'} = \v_{c,c'} - \v_{c}.
    \end{equation*}
\end{defn}
\end{mdframed}

\begin{mdframed}
\begin{defn}\textup{{\textbf{(Within-cross-class covariance).}}}
    The \textit{within-cross-class covariance}, \\ $\VW \in \R^{D \times D}$, associated with an array of vectors exhibiting cross-class structure, is given by
    \begin{equation*}
        \VW = \Ave_{i,c,c'} \z_{i,c,c'} \z_{i,c,c'}^\T,
    \end{equation*}
    where the \textit{replication deviations} $\z_{i,c,c'} \in \R^D$ are defined as follows:
    \begin{equation*}
        \z_{i,c,c'} = \v_{i,c,c'} - \v_{c,c'}.
    \end{equation*}
\end{defn}
\end{mdframed}
\noindent For simplicity, the notations of mean, covariance and second moment matrix discussed so far involved averages rather than weighted averages. Below, those notations will be extended to include certain weights $w_{i,c,c'}$ associated with corresponding terms $\v_{i,c,c'}$. Moreover, we will distinguish between vectors $\v_{i,c,c'}$, where $c = c'$ and $c \neq c'$. \footnote{In effect, we are introducing into deepnets constructs familiar in Multivariate Analysis of Variance (MANOVA), where the class/cross-class index structure would be called a two-way categorical layout. See reference \citep{huberty2006applied} for further details.}

\subsection{Cause attribution}
As the introduction has shown, various spectral features have been observed in the literature.
By proper use of our definitions, we are able to attribute causes for all the observed features as well as new ones. The cross-entropy loss induces a three-index structure of class, cross-class and replication. This index structure--inherited by all fundamental entities in deepnets, including features, backpropagated errors, and (extended) gradients--allows us to easily express certain second moment and covariance matrices. We shall demonstrate empirically that these matrices cause various spectral features:
\begin{itemize}
    \item The second moment matrix of the global mean causes the top outlier;
    \item The between-class covariance causes the leading cluster of $C-1$ outliers;
    \item The between-cross-class covariance causes the mini-bulk of $C(C-1)$ outliers; and
    \item The within-cross-class covariance causes the main bulk.
\end{itemize}
We will prove these assertions data-analytically by ``knocking out'' each of these matrices, and showing that such knockout eliminates the corresponding visual feature in the spectrum under study. We will formalize this notion of ``knockout'' into a formal attribution procedure.

\subsection{Ubiquity}
The effects of the class/cross-class structure permeate the spectra of deepnet features, backpropagated errors, gradients, weights, Fisher Information matrix, and Hessian, whether these are considered in the context of an individual layer or the concatenation of several layers. Specifically, we will show:
\begin{itemize}
    \item For a fixed layer $l$, the Kronecker product of the $c$'th class mean in the features, $\h_c^l$, and the $c$'th class mean in the backpropagated errors, $\deltaa_c^l$, approximates the $c$'th class mean in the gradients, $\g_c^l$.
    \item For a fixed layer $l$, the Kronecker product of the $c$'th class mean in the features, $\h_c^l$, and the $(c,c')$ cross-class mean in the backpropagated errors, $\deltaa_{c,c'}^l$, approximates the $(c,c')$ cross-class mean in the gradients, $\g_{c,c'}^l$.
    \item Similar relations hold between the class/cross-class means of features, $\h_c$; backpropagated errors $\deltaa_c$, $\deltaa_{c,c'}$; and gradients $\g_c$, $\g_{c,c'}$, once these are concatenated across the layers. However, now the Kronecker product is replaced by the Khatri-Rao product of the associated quantities.
    \item The $C$ class means and $C^2$ cross-class means in the layer-concatenated gradients induce $C$ and $C^2$ outliers in the spectra of the FIM.
    \item Outliers in the FIM also induce outliers in the spectrum of the Hessian, as the Hessian can be written as a summation of two components, one of them being the FIM.
\end{itemize}
These insights are summarized in Figure \ref{fig:patterns_relates}, as well as Table \ref{summary}.

\begin{figure}[t]
    \centering
    \includegraphics[width=1\textwidth]{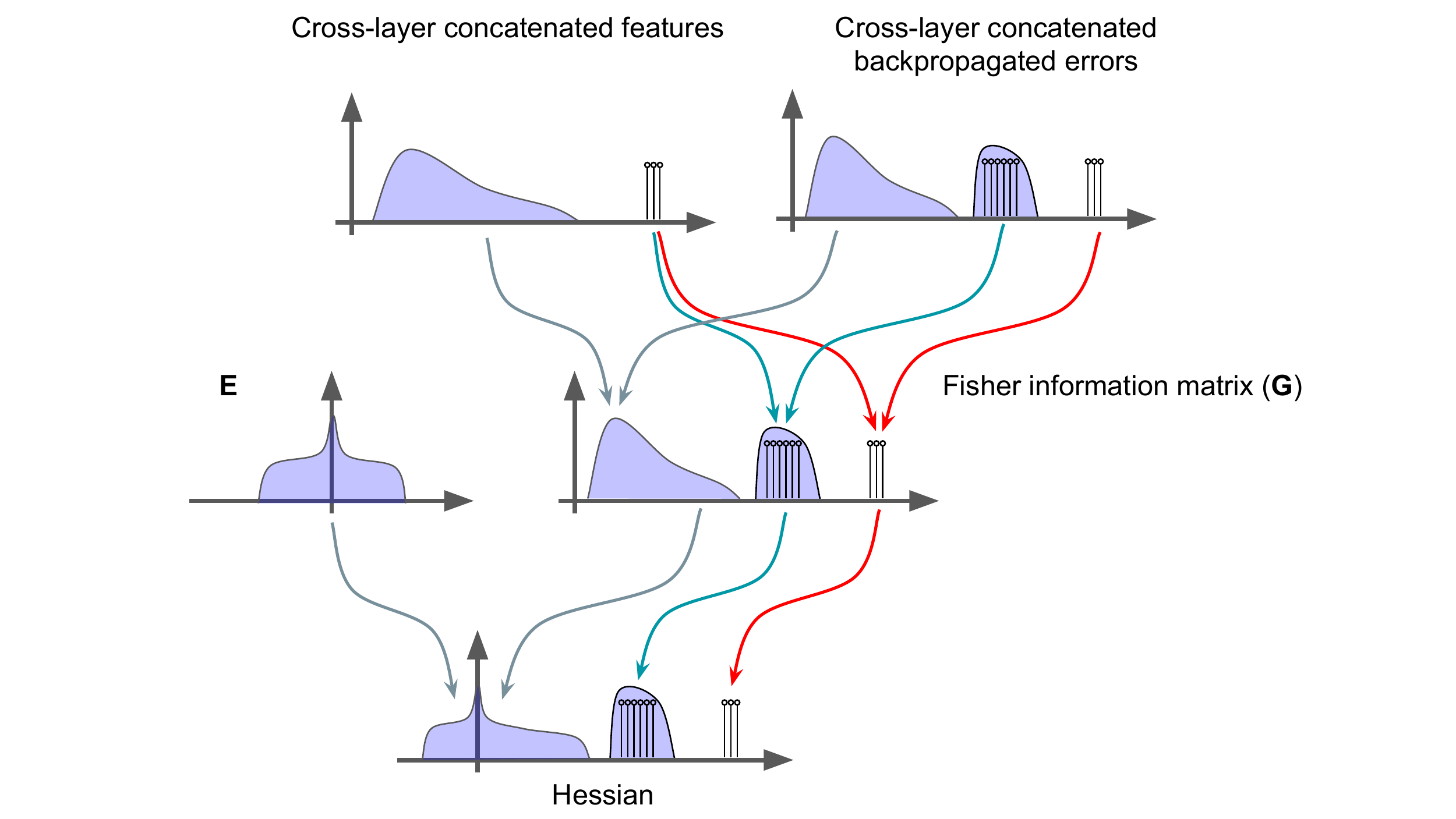}
    \caption{\textbf{Class/cross-class structure permeates all deepnet spectra ($C=3$ classes).}
    Class means in features are Khatri-Rao multiplied by class means in backpropagated errors to create class-means in FIM. Class means in features are Khatri-Rao multiplied by cross-class means in backpropagated errors to create cross-class-means in FIM. Class/cross-class means in FIM are inherited by Hessian. This important inheritance mechanism is further explained in the body of the text.}
    \label{fig:patterns_relates}
\end{figure}

\begin{table}[]
\centering
\begin{tabular}{|c|c|c|c|c|c|c|}
\hline
\multicolumn{1}{|c|}{\multirow{2}{*}{\textbf{Quantity}}}            & \multicolumn{1}{c|}{\multirow{2}{*}{\textbf{Section}}} & \multicolumn{1}{c|}{\multirow{2}{*}{\textbf{\begin{tabular}[c]{@{}c@{}}Second moment\end{tabular}}}} & \multicolumn{3}{c|}{\textbf{Attribution}}                                                                                                                                                     & \multicolumn{1}{c|}{\multirow{2}{*}{\textbf{Figures}}}                         \\ \cline{4-6}
\multicolumn{1}{|c|}{}                                              & \multicolumn{1}{c|}{}                               & \multicolumn{1}{c|}{}                                                                                  & \textbf{$C$ outliers}                                                              & \textbf{\begin{tabular}[c]{@{}l@{}}$C^2$ outliers\\ (mini-bulk)\end{tabular}}            & \textbf{Bulk} & \multicolumn{1}{c|}{}                                                          \\ \hline \hline
Hessian                                                    & \ref{sec:Hess_decomp}                               & \multicolumn{1}{c|}{-}                                                                                 & $\G$                                                                               & \multicolumn{1}{c|}{-}                                                                   & $\E$          & \ref{VGG11_spectrum_train_test_with_SSI}, \ref{attribution}                    \\ \hline
Gradients                                                  & \ref{decomp_G}                                      & $\G$                                                                                                   & \begin{tabular}[c]{@{}c@{}}$\GC$ \\ $\g_c \approx \h_c \otimes \deltaa_c$\end{tabular} & \begin{tabular}[c]{@{}c@{}}$\GCP$ \\ $\g_{c,c'} \approx \h_c \otimes \deltaa_{c,c'}$\end{tabular} & $\GW$         & \ref{log_G_SGD_ss} \\ \hline
Features                                                   & \ref{sec:features}                                  & $\H$                                                                                                   & $\HC$                                                                              & \multicolumn{1}{c|}{-}                                                                   & $\HW$         & \ref{fig:features_knockouts}, \ref{features_FCNet_bn_i}                        \\ \hline
\begin{tabular}[c]{@{}c@{}}Backprop.\\ errors\end{tabular} & \ref{sec:errors}                                    & $\Deltaa$                                                                                              & $\DeltaC$                                                                          & $\DeltaCP$                                                                               & $\DeltaW$     & \ref{fig:backprop_errors_knockouts}, \ref{jacobians_FCNet_bn_i}                \\ \hline
Weights                                                   & \ref{sec:weights}                                  & $\W$                                                                                                   & $\WC$                                                                              & \multicolumn{1}{c|}{-}                                                                   & $\WW$         & \ref{fig:weights_knockouts}                        \\ \hline
\end{tabular}
\caption{\textbf{Summary of conclusions from knockout experiments.} Each row corresponds to a different quantity of interest. The column ``Section'' references the section in which this quantity is described and possibly decomposed into its constituent components. The column ``Second moment'' indicates the notation for the second moment of this quantity. The attribution columns summarize the conclusions from the knockout experiments, which attribute spectral features observed in the spectrum of this quantity. In some cases, it also provides approximations for the matrices to which the spectral features are attributed. The last ``Figures'' column references all figures relevant to this quantity.}
\label{summary}
\end{table}

\subsection{Significance}
The outliers caused by the class means are clearly fundamental in predicting generalization. This is most evident through the following insights which will be presented in the following sections:
\begin{itemize}
    \item In the context of multinomial logistic regression, the ratio of outliers to bulk predicts misclassification.
    \item In the context of deepnets, feature class means \textit{gradually separate} from the bulk with growing depth and also \textit{gradually become orthogonal}. The ratio of outliers to bulk therefore predicts \textit{layer-wise linear separability}, while the standard deviation of the outliers represents the \textit{layer-wise orthogonality} of the classes.
\end{itemize}

\section{Problem setting}
Consider the balanced $C$-class classification problem whereby, given $n$ training examples in each of the $C$ different classes and their corresponding labels, the goal is to predict the labels on future data. Denote by $\x_{i,c}$ the $i$'th training example in the $c$'th class and by $\y_c$ its corresponding one-hot vector. A network is trained to classify an input $\x_{i,c}$ by passing it through a cascade of nonlinear transformations, ending with a linear classifier that outputs a set of predictions, $f(\x_{i,c};\thetaa) \in \R^C$. The parameters of the network, denoted by $\thetaa \in \R^p$, are trained using stochastic gradient descent (SGD) by minimizing the empirical cross-entropy loss $\ell$ averaged over the training data,
\begin{equation*}
    \Lagr(\thetaa) = \Ave_{i,c} \ell( f(\x_{i,c};\thetaa), \y_c ).
\end{equation*}
The train Hessian is defined to be the second derivative of the loss with respect to the parameters of the model, averaged over the training data, i.e.,
\begin{equation*}
    \Hess(\theta) = \Ave_{i,c} \left\{ \pdv[2]{\ell( f(\x_{i,c}; \theta), \y_c )}{\thetaa} \right\}.
\end{equation*}

\section{Spectral attribution via knockouts}
Throughout this work we will be pointing to spectral features visible in the eigenvalue distribution of various second moment and covariance matrices. We will attribute these features to various causes. We use two particular attribution procedures, based on different notions of knockout.
\begin{mdframed}
\begin{defn}\textup{{\textbf{(Subtraction knockout).}}}
    The process of subtracting a matrix $\B$ from another matrix $\A$ with the aim of eliminating certain spectral features. The resulting matrix will be denoted by $\A \ominus \B$.
\end{defn}
\end{mdframed}

\begin{mdframed}
\begin{defn}\textup{{\textbf{(Projection knockout).}}}
    The process of projecting the column space of a matrix $\B$ from another matrix $\A$, with the aim of eliminating certain spectral features. Mathematically, this is equivalent to computing $(\I - \B \B^{\dagger}) \A (\I - \B \B^{\dagger})$, where $\B^{\dagger}$ is the Moore–Penrose pseudoinverse of the matrix $\B$. The resulting matrix will be denoted by $\A \nparallel \B$. Assuming $\A$ and $\B$ are not square matrices, we define the projection knockout to be
    \begin{equation*}
        \A \nparallel \B
        = (\I - \U \U^\T) \A (\I - \V \V^\T),
    \end{equation*}
    where $\U$ and $\V$ contain all the left and right singular vectors of $\B$, respectively.
\end{defn}
\end{mdframed}

\begin{mdframed}
\begin{defn}\textup{{\textbf{(Spectral attribution via linear algebraic knockouts).}}}
    The process of attributing spectral features in the spectrum of matrix $\A$ to the spectrum of another matrix $\B$ by observing that these spectral features visually disappear after $\B$ is knocked out.
\end{defn}
\end{mdframed}


\begin{figure}[t]
    \centering
    \begin{subfigure}[t]{0.48\textwidth}
        \centering
        \includegraphics[width=1\textwidth]{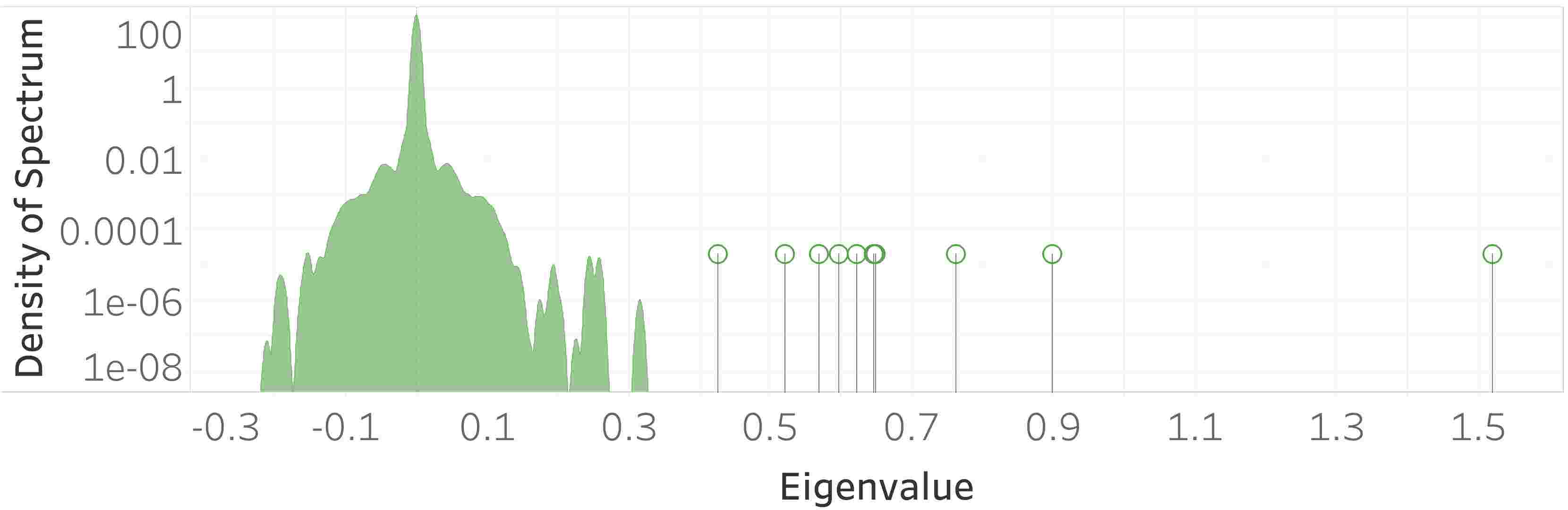}
        \caption{MNIST, train}
    \end{subfigure}
    \begin{subfigure}[t]{0.48\textwidth}
        \centering
        \includegraphics[width=1\textwidth]{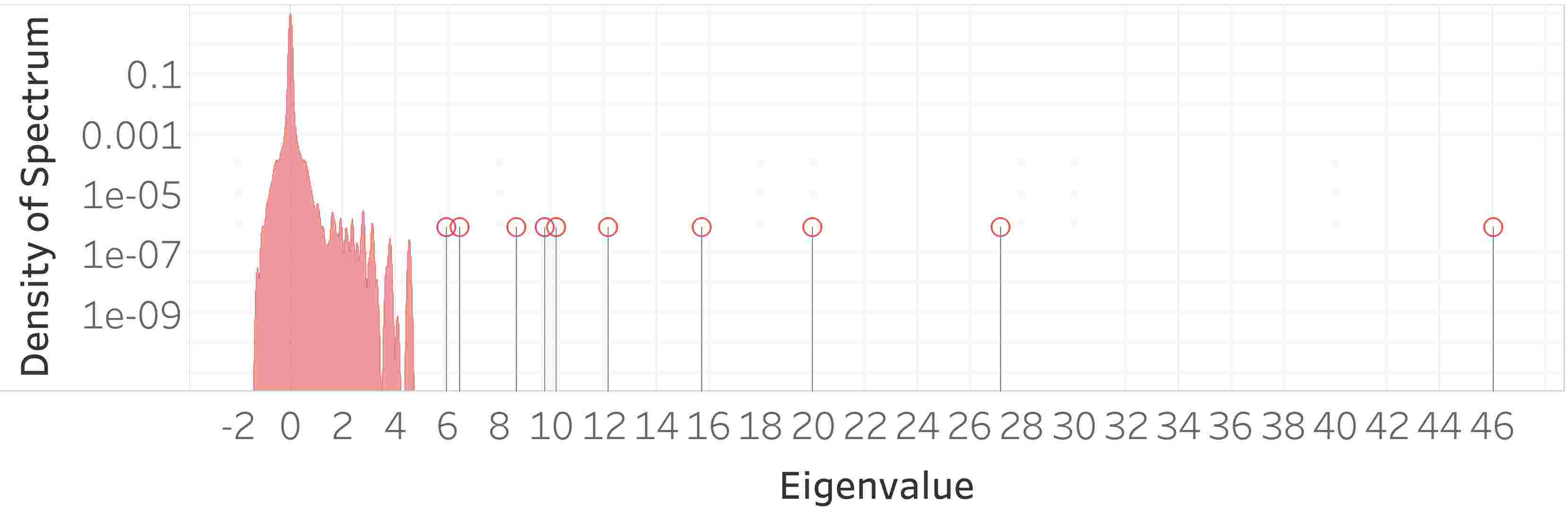}
        \caption{MNIST, test}
    \end{subfigure}
    \begin{subfigure}[t]{0.48\textwidth}
        \centering
        \includegraphics[width=1\textwidth]{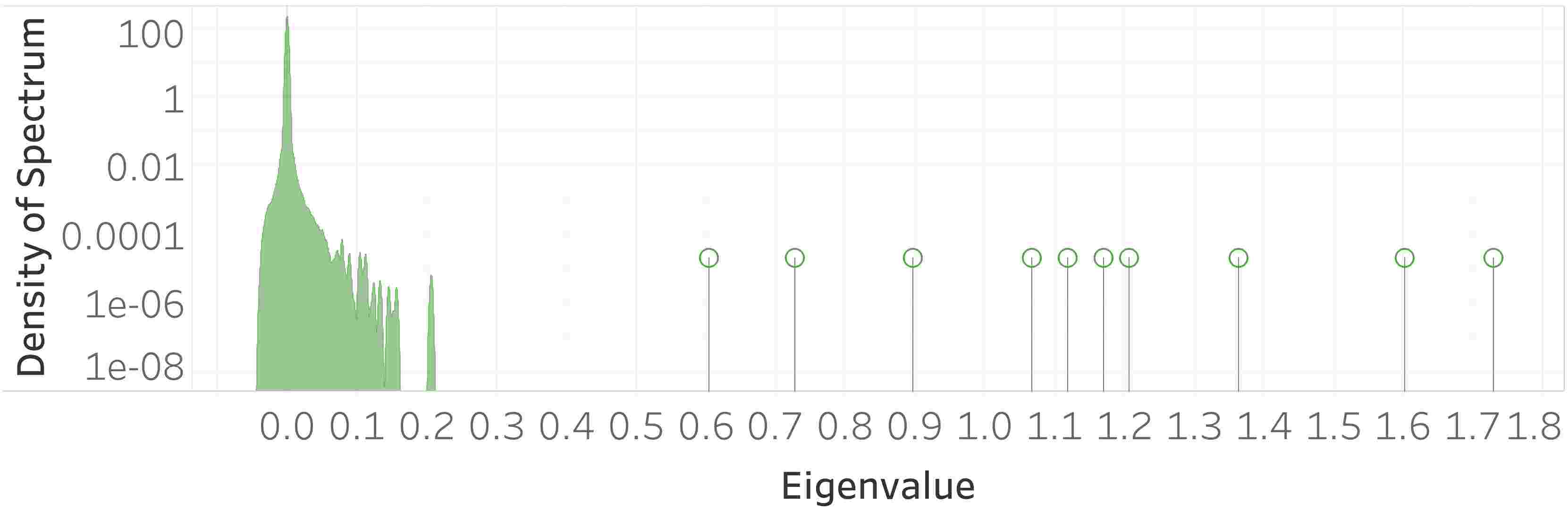}
        \caption{Fashion, train}
    \end{subfigure}
    \begin{subfigure}[t]{0.48\textwidth}
        \centering
        \includegraphics[width=1\textwidth]{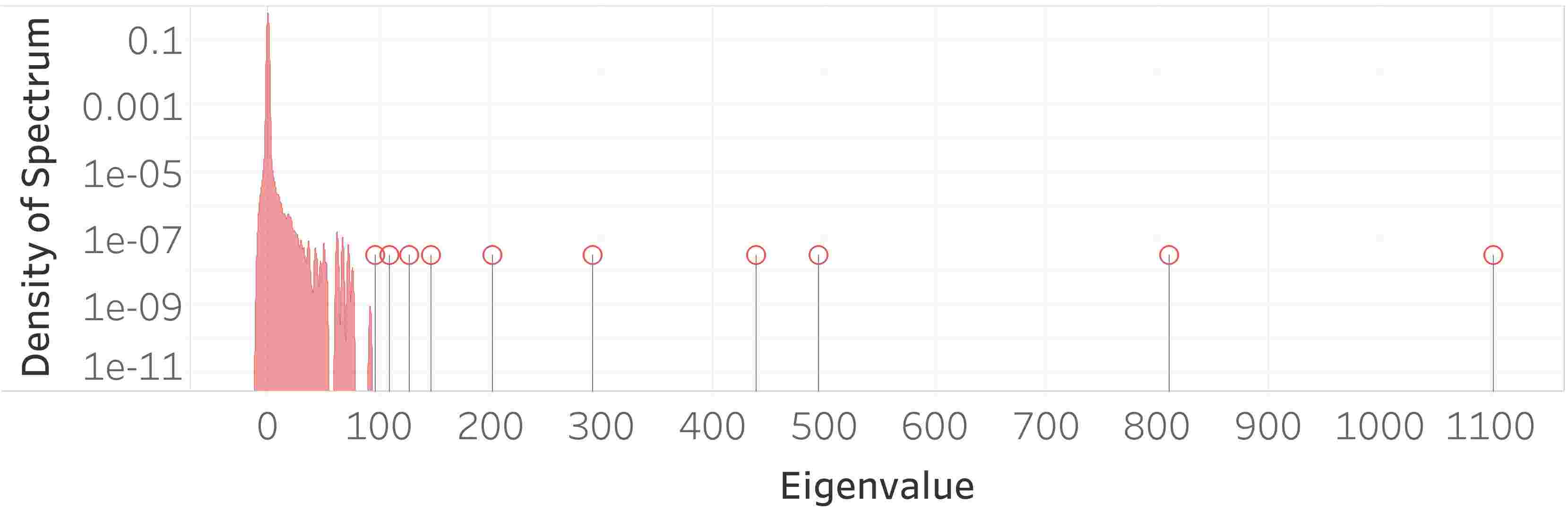}
        \caption{Fashion, test}
    \end{subfigure}
    \begin{subfigure}[t]{0.48\textwidth}
        \centering
        \includegraphics[width=1\textwidth]{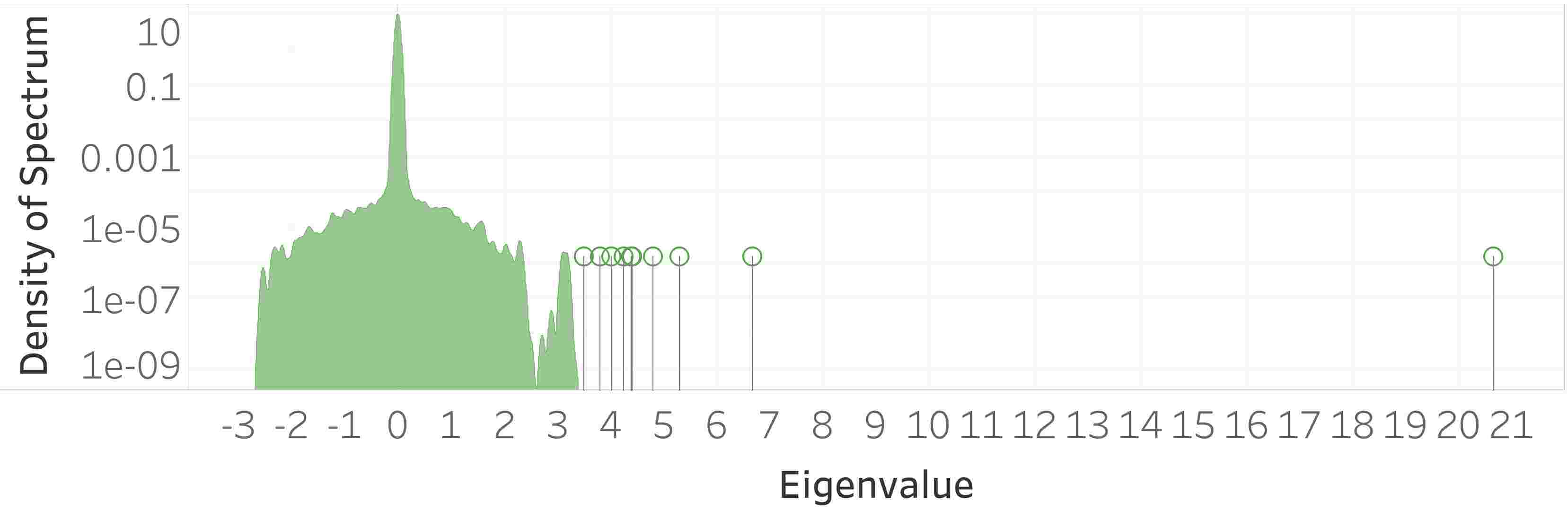}
        \caption{CIFAR10, train}
    \end{subfigure}
    \begin{subfigure}[t]{0.48\textwidth}
        \centering
        \includegraphics[width=1\textwidth]{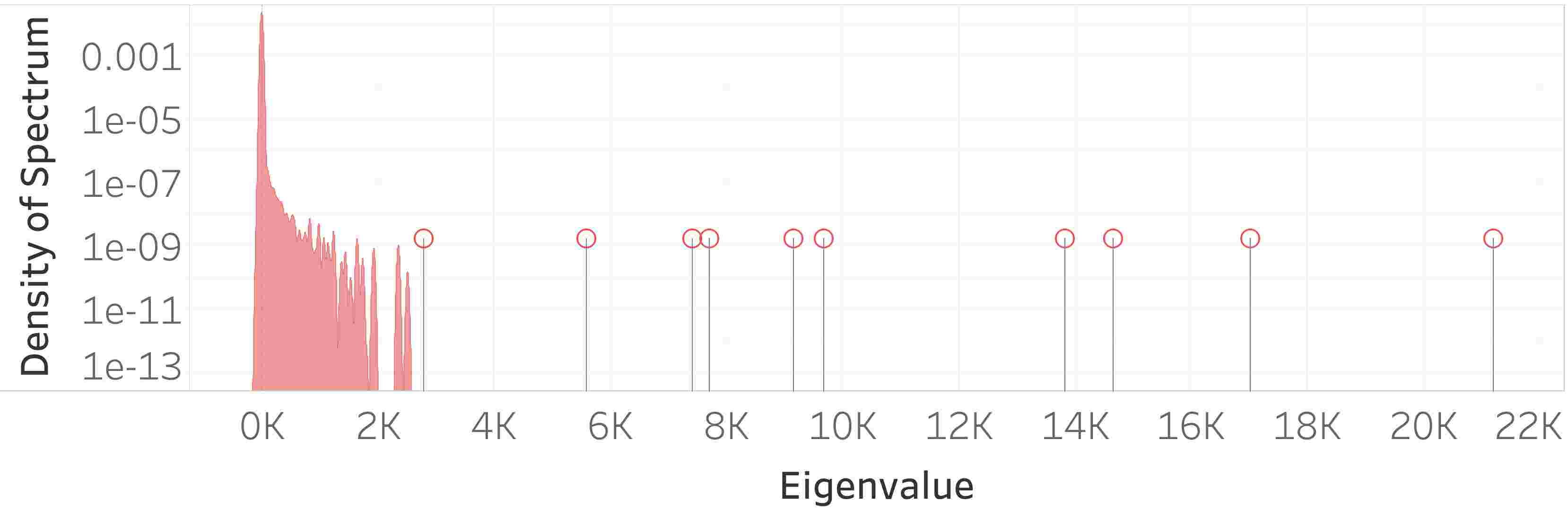}
        \caption{CIFAR10, test}
    \end{subfigure}
    \caption{\textbf{Spectrum of the Hessian for VGG11 trained on various datasets.} Each row of panels documents a `well-known' or `standard' dataset in deep learning. The panels in the left column correspond to the train Hessian, while those in the right column to the test Hessian. Notice the presence of a bulk and $C$(=10) outliers. The y-axis is on a logarithmic scale.}
    \label{VGG11_spectrum_train_test_with_SSI}
\end{figure}

\section{Hessian spectrum} \label{sec:Hessian}
\citet{sagun2016eigenvalues,sagun2017empirical} measured the spectrum of the Hessian, observing a bulk-and-outliers structure with approximately $C$ outliers. They experimented with: (i) two-hidden-layer networks, with $30$ hidden units each, trained on synthetic data sampled from a Gaussian Mixture Model data; and (ii) one-hidden-layer networks, with $70$ hidden units, trained on MNIST. In this section, we confirm their reports, this time at the full scale of modern state-of-the-art networks trained on real natural images.

We release software implementing state-of-the-art tools in numerical linear algebra, which allows one to approximate efficiently the spectrum of the Hessian of modern deepnets such as VGG and ResNet. We describe its functionality in Appendix \ref{numerical_linear_algebra}. Similar tools were concurrently proposed in the literature by \citet{ghorbani2019investigation,pfahler2019evolution,Granziol2019,chatzimichailidis2019gradvis}, each utilized for a different purpose. However, as far as we are aware, no other work released their software.

\subsection{Spectrum of Hessian has structure}
In Figure \ref{VGG11_spectrum_train_test_with_SSI} we plot the spectra of the train and test Hessian of VGG11, an architecture with 28 million parameters, trained on various datasets. The top-$C$ eigenspace was estimated precisely using $\textsc{LowRankDeflation}$ (built upon the power method) and the rest of the spectrum was approximated using \textsc{LanczosApproxSpec}. Both are described in Appendix \ref{numerical_linear_algebra}, where we also show the same plots except without first applying $\textsc{LowRankDeflation}$.

We observe a clear bulk-and-outliers structure with, arguably, $C$ outliers. The bulk is centered around zero and there is a big concentration of eigenvalues at zero due to the large number of parameters in the model ($28$ million) compared to the small amount of training ($50$ thousand) or testing ($10$ thousand) examples. As pointed out by \citet{sagun2016eigenvalues,sagun2017empirical} and further discussed by \citet{alain2019negative}, negative eigenvalues exist in the spectrum of the train Hessian. This is despite the fact that the model was trained for hundreds of epochs, the learning rate was annealed twice and its initial value was optimized over a set of $100$ values. Note there is a clear difference in magnitude between the train and test Hessian, despite the fact that both were normalized by the number of contributing terms.

\section{Decomposing Hessian into two components: \\ \texorpdfstring{$\Hess = \G + \E$}{Hess=G+E}} \label{sec:Hess_decomp}
Following the ideas presented by \citet{sagun2016eigenvalues}, we use the generalized Gauss-Newton decomposition of the Hessian and write it as a summation of two components:
\begin{align} \label{eq:Hess_G_E}
    \Hess = 
    & \underbrace{\Ave_{i,c} \left\{ \pdv{f(\x_{i,c};\thetaa)}{\thetaa}^\T \pdv[2]{\ell ( \z, \y_c)}{\z} \Bigg|_{\z_{i,c}} \pdv{f(\x_{i,c};\thetaa)}{\thetaa} \right\}}_{\G} \\
    + & \underbrace{\Ave_{i,c} \left\{ \sum_{c'=1}^C \pdv{\ell ( \z, \y_c)}{z_{c'}} \Bigg|_{\z_{i,c}} \pdv[2]{f_{c'}(\x_{i,c}; \thetaa)}{\thetaa} \right\}}_{\E},
\end{align}
where $\z_{i,c}=f(\x_{i,c}; \thetaa)$. In mathematical statistics $\G$ is called the Fisher Information Matrix (FIM). Moreover, it is related to the natural gradient algorithm \citep{amari1998natural}, as explained by \citet{pascanu2013revisiting}.

\begin{figure}
    \centering
    \begin{subfigure}[t]{0.925\textwidth}
        \centering
        \includegraphics[width=0.875\textwidth]{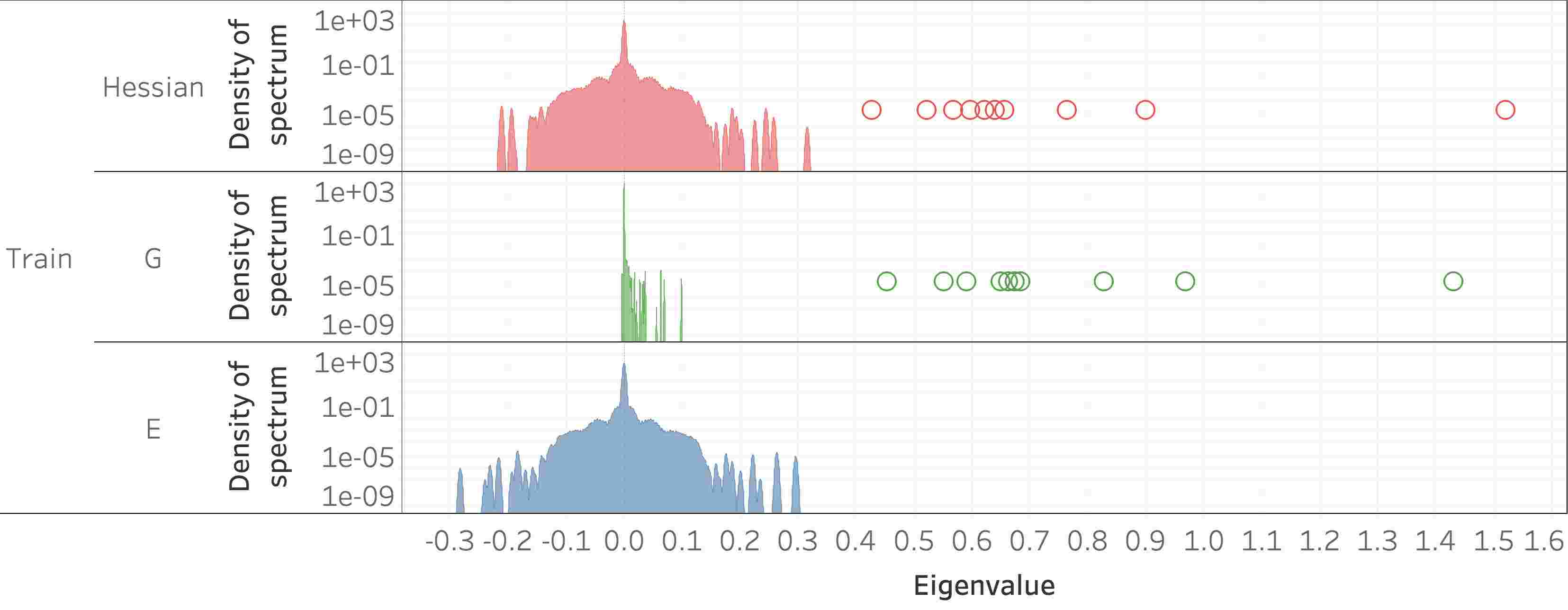}
        \includegraphics[width=0.875\textwidth]{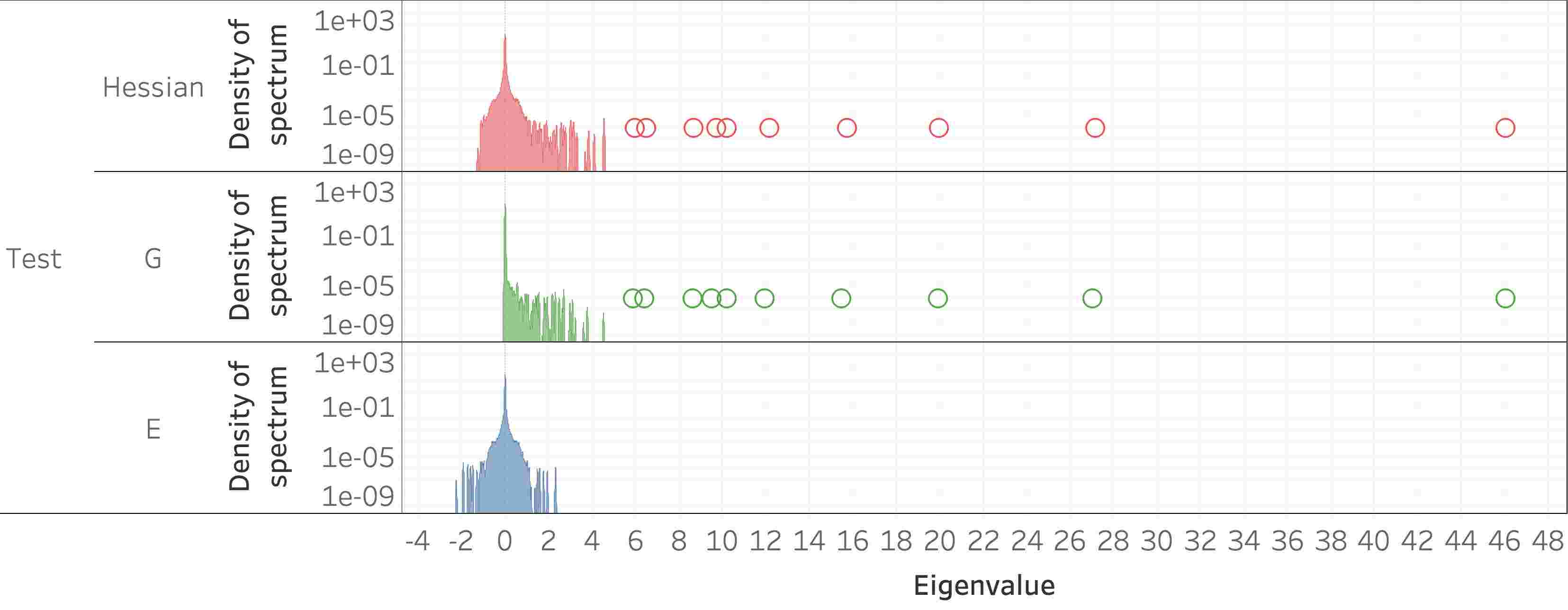}
        \caption{\textbf{Outliers attributed to $\G$.} The first block of three panels plots train spectra, while the second block plots test spectra. Within each block, each panel depicts the spectrum of a different matrix. The top panel: spectrum of Hessian. Middle panel: spectrum of $\G$ ($\E$ knocked out). Bottom panel: spectrum of $\E$ ($\G$ knocked out). Note how $\G$ is clearly responsible for the outliers.}
        \label{attribution_outliers}
    \end{subfigure}
    \begin{subfigure}[t]{0.925\textwidth}
        \centering
        \vspace{0.025\textwidth}
        \begin{overpic}[width=0.9\textwidth]{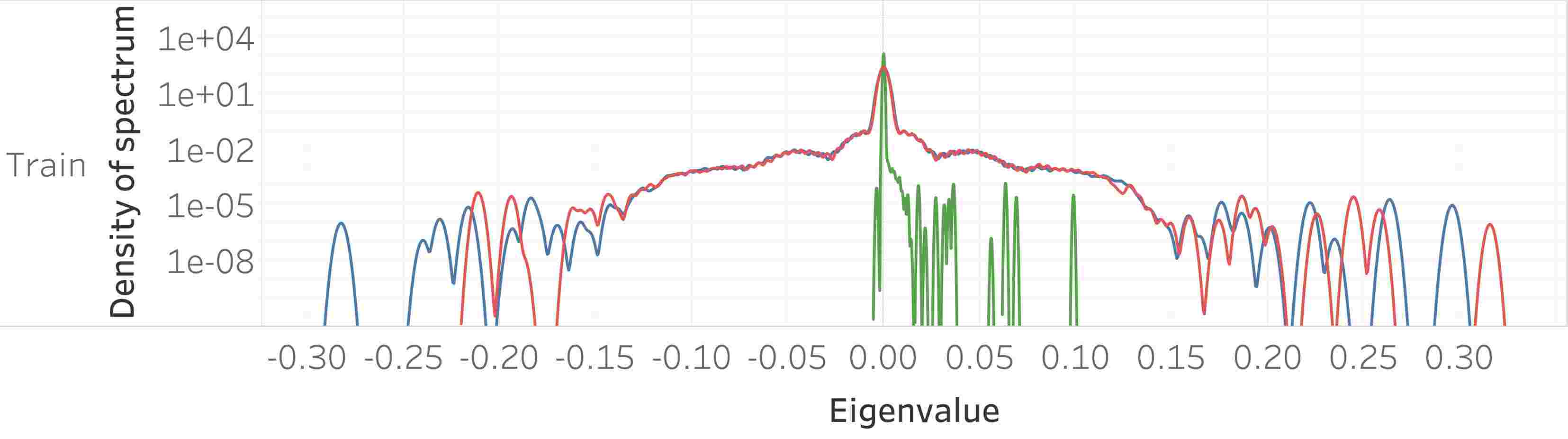}
        \put(150,85){{\parbox{1\linewidth}{%
        {\scriptsize
        \begin{align*}
            & \fcolorbox{black}{tableau_red}{\rule{0pt}{3pt}\rule{3pt}{0pt}} \quad
            \color{tableau_red} \Hess \\
            & \fcolorbox{black}{tableau_green}{\rule{0pt}{3pt}\rule{3pt}{0pt}} \quad
            \color{tableau_green} \G \\
            & \fcolorbox{black}{tableau_blue}{\rule{0pt}{3pt}\rule{3pt}{0pt}} \quad
            \color{tableau_blue} \E
        \end{align*}
        }%
        }}}
        \end{overpic}
        \includegraphics[width=0.9\textwidth]{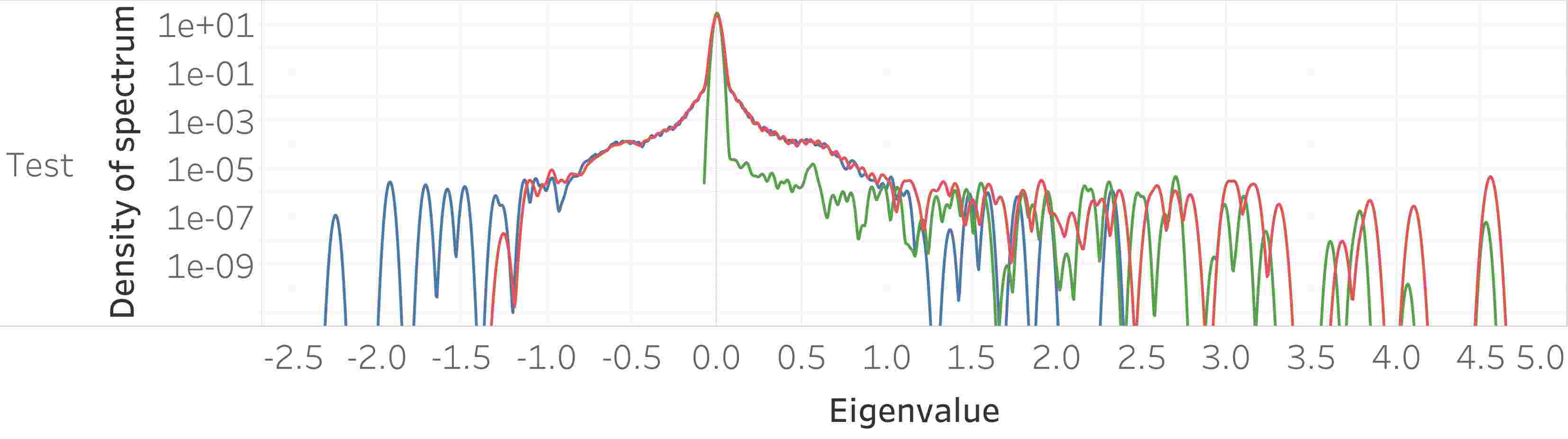}
        \caption{\textbf{Bulk attributed to $\E$.} Zoom in on the bulks of the Hessian and its two components. The top panel plots train spectra; the bottom, test spectra.
        \label{attribution_bulk}
        }
    \end{subfigure}
    \caption{\textbf{Spectrum of the Hessian with its constituent components.} The network is VGG11 and it was trained on MNIST sub-sampled to $5000$ examples per class. The y-axis of all plots is on a logarithmic scale.}
    \label{attribution}
\end{figure}

\subsection{Outliers attributable to \texorpdfstring{$\G$}{G}, bulk attributable to \texorpdfstring{$\E$}{E}}
Figure \ref{attribution} plots: (i) the spectrum of the Hessian, (ii) the spectrum of the Hessian after $\E$ is knocked out and only $\G$ is left; and (iii) the spectrum of the Hessian after $\G$ is knocked out and only $\E$ is left. Each spectrum was approximated using \textsc{LanczosApproxSpec} and the $\textsc{LowRankDeflation}$ procedure was applied on the Hessian and the $\G$ component to approximate the top-$C$ subspace. Notice how the spectra of all three matrices resemble variations of Figure \ref{fig:pattern}.

Notice how knocking out $\E$ shrinks the bulk significantly, indicating that the bulk originates largely from $\E$. Note also that the upper tails of the test Hessian and test $\G$ obey eigenvalue interlacing, as in Cauchy's interlacing theorem \citep{horn2012matrix}. 

Notice how knocking out $\G$ eliminates the outliers in the spectrum; the outliers in the Hessian are attributable to $\G$. \citet{papyan2019measurements} showed that the outliers in $\G$ are attributable to the presence of $C$ high energy gradient class means, which explains our previous observation of $C$ outliers in the spectrum of the Hessian.

\section{Cross-class structure in \texorpdfstring{$\G$}{G}} \label{decomp_G}
\citet{papyan2019measurements} proposed to decompose $\G$ based on the cross-class structure. We follow their proposition but slightly modify their decomposition. Recall the definition of an extended gradient\footnote{In \citet{papyan2019measurements} $\g_{i,c,c'}$ is defined slightly differently.}
\begin{equation} \label{eq:g_iccp}
    \g_{i,c,c'} = \pdv{\ell( f(\x_{i,c}; \thetaa), \y_{c'} )}{\thetaa}.
\end{equation}
Note that for $c=c'$, $\g_{i,c,c}$ is simply the usual gradient of the $i$'th training example in the $c$'th class. Alternatively, for $c \neq c'$, $\g_{i,c,c'}$ is the would-be gradient of the $i$'th training example in the $c$'th class, as if it belonged to cross-class $c'$ instead. This definition is useful since, as we prove in Appendix \ref{G_decomp}, $\G$ is a weighted second moment matrix of these gradients, i.e.,
\begin{equation} \label{eq:G_second_moment}
    \G = \sum_{i,c,c'} w_{i,c,c'} \g_{i,c,c'} \g_{i,c,c'}^\T,
\end{equation}
where the weights $w_{i,c,c'}$ are defined as follows:
\begin{equation*}
    w_{i,c,c'} = \frac{p_{i,c,c'}}{nC}.
\end{equation*}
Above, $p_{i,c,c'}$ is the $c'$-th entry of $p(\x_{i,c};\thetaa) \in \R^C$, which are the Softmax probabilities of $\x_{i,c}$. Define the \textit{gradient cross-class means}:
\begin{align*}
    \g_{c,c'} = \sum_i \pi_{i,c,c'} \g_{i,c,c'},
\end{align*}
the \textit{gradient within-cross-class covariance}:
\begin{equation} \label{eq:def_GW}
    \GW = \sum_{i,c,c'} w_{i,c,c'} (\g_{i,c,c'} - \g_{c,c'}) (\g_{i,c,c'} - \g_{c,c'})^\T,
\end{equation}
and its class/cross-class specific versions:
\begin{align*}
    \GWccp & = \sum_i \pi_{i,c,c'} (\g_{i,c,c'} - \g_{c,c'}) (\g_{i,c,c'} - \g_{c,c'})^\T,
\end{align*}
where the weights are given by
\begin{align*}
    \pi_{i,c,c'} & = \frac{w_{i,c,c'}}{w_{c,c'}} \\
    w_{c,c'} & = \sum_i w_{i,c,c'}.
\end{align*}
These equations group together gradients for a fixed pair of $c,c'$. Define the \textit{gradient class means}:
\begin{equation*}
    \g_c = \sum_{c' \neq c} \pi_{c,c'} \g_{c,c'}.
\end{equation*}
Notice that the above average does not take into account $\g_{c,c}$. The reason is that these are gradient means, which are approximately equal to zero at convergence of SGD. Define also the \textit{between-class gradient second moment}:
\begin{equation*}
    \GC = \sum_c \w_c \g_c \g_c^\T,
\end{equation*}
the \textit{within-cross-class gradient covariance}:
\begin{equation*}
    \GCP = \sum_{c, c' \neq c} w_{c,c'} (\g_{c,c'} - \g_c) (\g_{c,c'} - \g_c)^\T,
\end{equation*}
and its class-specific version:
\begin{equation*}
    \GCPc = \sum_{c' \neq c} \pi_{c,c'} (\g_{c,c'} - \g_{c'}) (\g_{c,c'} - \g_{c'})^\T,
\end{equation*}
where the weights are given by
\begin{align*}
    \pi_{c,c'} & = \frac{w_{c,c'}}{w_c} \\
    w_c & = \sum_{c' \neq c} w_{c,c'}.
\end{align*}
These equations represent an even coarser grouping, where the gradients with a fixed $c$ (and possibly varying $c'$) are grouped together. Although not used above, we will also define $\pi_c = \frac{w_c}{\sum_c w_c}$. Leveraging these definitions, we prove in Appendix \ref{G_decomp} that $\G$ can be decomposed as follows:
\begin{align} \label{eq:G_decomp}
    \G & =
    \GC
    + \GCP
    + \GW
    + \underbrace{\sum_c w_{c,c} \g_{c,c} \g_{c,c}^\T}_{\G_{c=c'}}.
\end{align}

\subsection{\texorpdfstring{$C$}{C} outliers attributable to \texorpdfstring{$\GC$}{GC},  \texorpdfstring{$C^2$}{C2} outliers attributable to \texorpdfstring{$\GCP$}{GCP} and bulk attributable to \texorpdfstring{$\GW$}{GW}}
\citet{papyan2019measurements} showed empirically that the outliers in the spectrum of $\G$ are attributable to the covariance of the gradient class means, i.e., $\GC$ in our decomposition. However, that earlier work did not provide empirical evidence for the existence of other components in the decomposition in Equation \eqref{eq:G_decomp}, since these are quite subtle and not always visibly pronounced in the spectra of $\G$ or its knockouts. For the present work, we developed a tool to depict the spectrum of $\log(\G)$ by leveraging the numerical linear algebra machinery developed in Appendix \ref{numerical_linear_algebra}. It turns out, the spectrum of $\log(\G)$ rather than $\G$ exposes all the components in the decomposition.

\begin{figure}[t]
    \centering
    \begin{overpic}[width=0.75\textwidth]{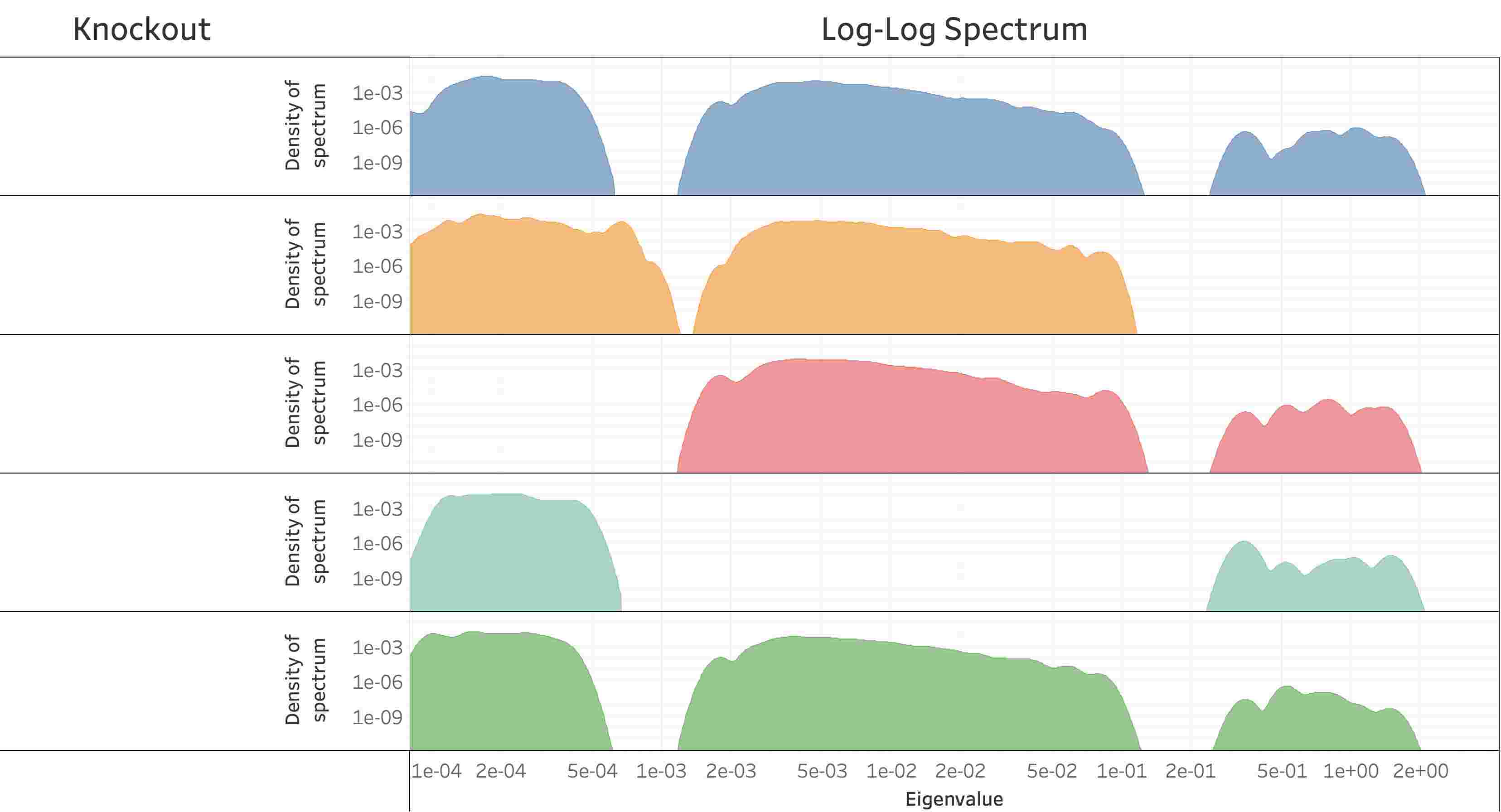}
    \put(-200,93){{\parbox{1\linewidth}{%
    {
    \begin{align*}
        & \G \ominus \mathbf{0} \\[0.585cm]
        & \G \ominus \GC \\[0.585cm]
        & \G \ominus \GCP \\[0.585cm]
        & \G \ominus \GW \\[0.585cm]
        & \G \ominus \G_{c=c'}
    \end{align*}
    }%
    }}}
    \end{overpic}
    \caption{\textbf{Attribution via knockouts of spectral features in spectrum of train $\G$.} Each panel in the right column plots the approximate log spectrum of the matrix indicated in the left column. The approximation is computed using the \textsc{LanczosApproxSpec} procedure. The network is VGG11 and it was trained on CIFAR10 subsampled to 136 examples per class. The y-axis is on a logarithmic scale.}
    \label{G_attribution}
\end{figure}

\begin{figure}
    \centering
    \includegraphics[width=0.9\textwidth]{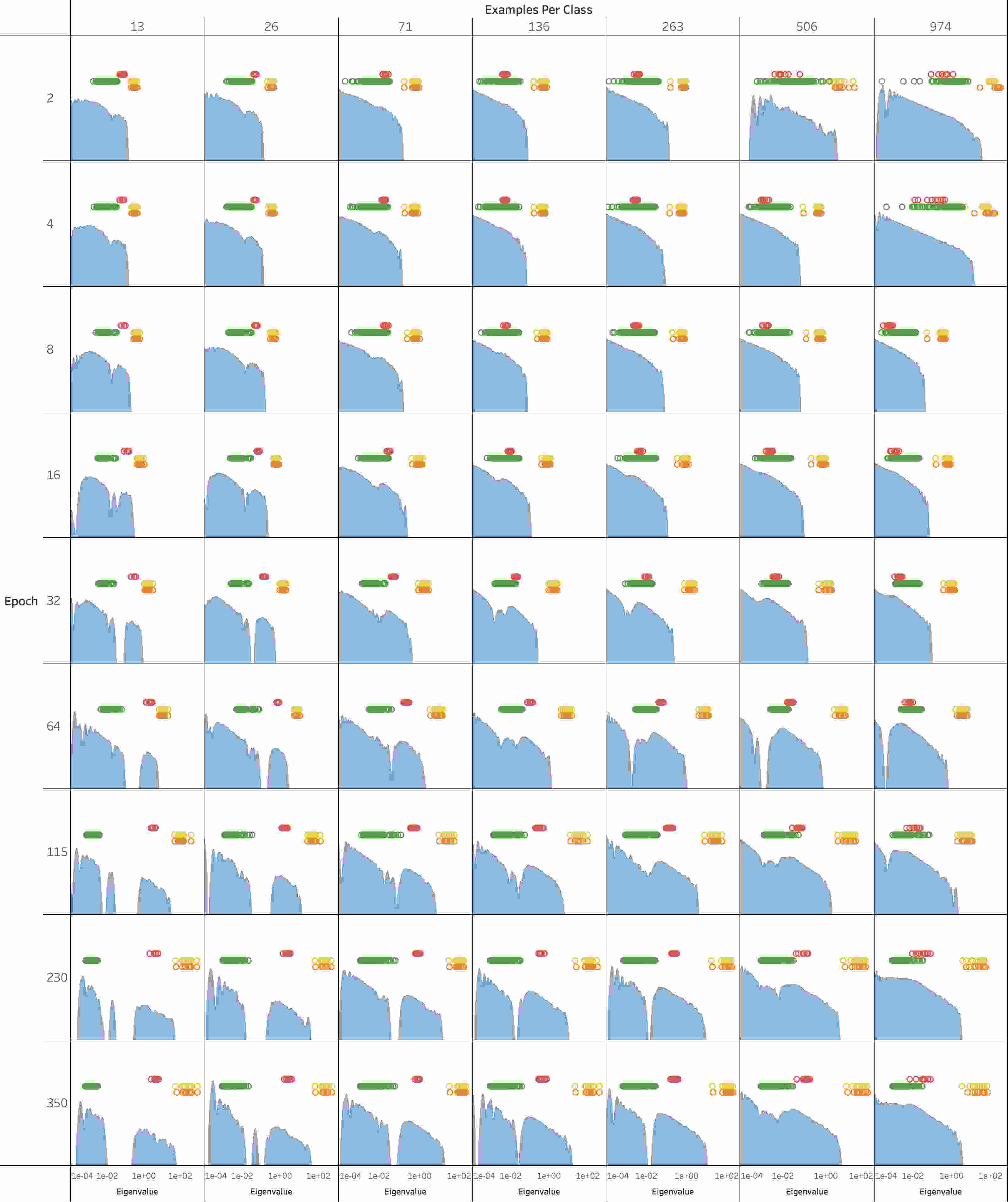}
    \caption{\textbf{Dynamics with training and sample size of cross-class structure in $\G$.} Each column of panels plots, for one specific sample size $N$, the spectrum of train $\G$ throughout the epochs of SGD, so that each row corresponds to a different epoch. Each panel also plots the eigenvalues of $\GC+\GCP+\G_{c=c'}$. The top-$C$ eigenvalues of this matrix, which are attributable to $\GC$, are colored in yellow. The next $C^2-2C$ eigenvalues, attributable to $\GCP$, are colored in green. The final $C$ eigenvalues, attributable to $\G_{c=c'}$, are colored in teal but are missing from the plots because their magnitude is less than $10^{-4}$. Each panel also plots in red the average eigenvalue of the matrices $\{ \GWc \}_{c=1}^C$, given by $\{ \frac{1}{N} \Tr{\GWc} \}_{c=1}^C$, where $\GWc$ is the within-cross-class covariance restricted to class $c$. The y-axis of all panels is on a logarithmic scale.}
    \label{log_G_SGD_ss}
\end{figure}

The first panel of Figure \ref{G_attribution} depicts the spectrum of $\log(\G)$. Notice its surprisingly simple structure with three separated bulks. The one in the middle is the main bulk that would ordinarily be seen in the spectrum of $\G$. The left and right bulks have very low density of eigenvalues compared to it--they can only be seen because we are looking at the spectrum of $\log(\G)$. The same Figure also shows the spectrum of $\G$ once different matrices are knocked out. Once $\GC$ is knocked out, the $C$ outliers on the right disappear, corroborating previous findings by \citet{papyan2019measurements} that claim these outliers are attributable to $\GC$. More importantly, once $\GCP$ is knocked out, the left mini-bulk disappears; it is attributable to the between-cross-class covariance. Once $\GW$ is knocked out, the main bulk disappears, implying it is attributable to the within-cross-class covariance. Subtracting $\G_{c=c'}$ has no clear effect on the spectrum, which can be explained by noticing that $\G_{c=c'}$ is a second moment of gradient means, which are approximately equal to zero at convergence of SGD.


\subsection{Dynamics of bulk and two groups of outliers with training and sample size}
Figure \ref{log_G_SGD_ss} plots spectra of $\log(\G)$ across epochs and training sample size. Each panel plots, in orange, the top-$C$ eigenvalues of $\G$, estimated using \textsc{SubspaceIteration} and, in blue, the log spectrum of the rank-$C$ deflated $\G$, estimated using \textsc{LanczosApproxSpec}. For numerical stability, we approximate the spectrum of $\log(\G + 10^{-5} \I)$.

Notice the alignment between the $C$ outliers of $\G$, colored in orange, and the eigenvalues of $\GC$, colored in yellow. Notice also the alignment between the $C^2$ outliers of $\G$ and the green dots corresponding to $\GCP$, and also the main bulk of $\G$ and the red dots corresponding to $\GW$. This correspondence aligns with the attribution of the different spectral features to the different components in the cross-class structure, observed already in the previous subsection.

Fixing sample size and increasing the number of epochs causes spectral bulks of $\GCP$ and $\GW$ to separate. In contrast, fixing the epoch and increasing sample size causes the spectral bulks to merge. Moreover, varying the epoch number or the training sample size does not change significantly the distance between the $\GC$ and $\GW$.

\section{Multilayer perceptron}
In this section we study the relation between the patterns in the features, backpropagated errors, gradients, FIM, and weights in the context of a multilayer perceptron (MLP), i.e., a cascade of fully connected layers.

\subsection{Forward pass}
In the forward pass, for each layer $1 \leq l \leq L$, we multiply the features of the previous layer $\h_{i,c}^{l-1}$ by a weight\footnote{In practice, we use batch normalization layers prior to ReLU. However, for simplicity of exposition we ignore them.} matrix $\W^l$ to produce the pre-activations of the next layer $\z_{i,c}^l$, i.e.,
\begin{equation*}
    \z_{i,c}^l = \W^l \h_{i,c}^{l-1}.
\end{equation*}
Note that for $l=1$, $\h_{i,c}^{l-1} = \h_{i,c}^0$ is equal to $\x_{i,c}$. These are then passed through a non-linearity $\sigma$, which in our case is the rectified linear unit (ReLU), to produce the features of the next layer,
\begin{equation*}
    \h_{i,c}^l = \sigma (\z_{i,c}^l).
\end{equation*}

\subsection{Backward pass}
The backward pass computes the gradient of the loss with respect to the parameters of the model, which in this case are the weight matrices,
\begin{equation*}
    \pdv{\ell( f(\x_{i,c};\thetaa), \y_c )}{\W^l}.
\end{equation*}
We will consider a more general gradient where the label class $c$ and the cross-class $c'$ are allowed to differ, i.e.,
\begin{equation*}
    \pdv{\ell( f(\x_{i,c};\thetaa), \y_{c'} )}{\W^l}.
\end{equation*}
Using the chain rule of calculus, one can show that:
\begin{align} \label{eq:gradient}
    \pdv{\ell( f(\x_{i,c};\thetaa), \y_{c'} )}{\W^l} 
    & = \pdv{\ell( f(\x_{i,c};\thetaa), \y_{c'} )}{\z_{i,c}^l} \pdv{\z_{i,c}^l}{\W^l} \\
    & = \deltaa_{i,c,c'}^l {\h_{i,c}^{l-1}}^\T,
\end{align}
where $\deltaa_{i,c,c'}^l$ denote the \textit{backpropagated errors}, which we define as follows:
\begin{equation*}
    \deltaa_{i,c,c'}^l = \pdv{\ell( f(\x_{i,c};\thetaa), \y_{c'} )}{\z_{i,c}^l}.
\end{equation*}
Let $\vect(\cdot)$ denote the operator forming a vector from a matrix by stacking its columns as subvector blocks within a single vector. Using this operator, the above equation can be vectorized as follows:
\begin{equation} \label{eq:kronecker_error}
    \pdv{\ell( f(\x_{i,c};\thetaa), \y_{c'} )}{\vect(\W^l)}
    = {\h_{i,c}^{l-1}} \otimes \deltaa_{i,c,c'}^l,
\end{equation}
where $\otimes$ denotes the Kronecker product. Recall our definition of $\g_{i,c,c'}$ in Equation \eqref{eq:g_iccp}, and define its analogous layer-specific quantity:
\begin{equation*}
    \g_{i,c,c'}^l = \pdv{\ell( f(\x_{i,c}; \thetaa), \y_{c'} )}{\vect(\W^l)};
\end{equation*}
using Equation \eqref{eq:kronecker_error}, this is equal to
\begin{equation} \label{eq:kronecker_g}
    \g_{i,c,c'}^l = {\h_{i,c}^{l-1}} \otimes \deltaa_{i,c,c'}^l.
\end{equation}
Note that $\g_{i,c,c'}^l$ is the $l$'th subvector of $\g_{i,c,c'}$. Define the \textit{feature} $\h_{i,c}$ to be the concatenation of all the feature subvectors $\{ \h_{i,c}^{l-1} \}_{l=1}^L$ into a single vector so that $\h_{i,c}^{l-1}$ is the $l$'th subvector of $\h_{i,c}$. Similarly, define the \textit{backpropagated error} $\deltaa_{i,c,c'}$ to be the concatenation of all the backpropagated errors $\{ \deltaa_{i,c,c'}^l \}_{l=1}^L$ into a single vector so that $\deltaa_{i,c,c'}^l$ is the $l$'th subvector of $\deltaa_{i,c,c'}$. Using these definitions, we can write
\begin{equation} \label{eq:khatrirao_g}
    \g_{i,c,c'} = {\h_{i,c}} \odot \deltaa_{i,c,c'},
\end{equation}
where $\odot$ denotes the Khatri-Rao\footnote{The Khatri-Rao product is a section of the full Kronecker product.} product, which computes in the $l$'th layer block of $\g_{i,c,c'}$ the Kronecker product of the corresponding $l$'th layer blocks from $\deltaa_{i,c,c'}$ and $\h_{i,c}$.

\subsection{Kronecker structure in \texorpdfstring{$\G$}{G}}
Plugging Equation \eqref{eq:khatrirao_g} into Equation \eqref{eq:G_second_moment}, we obtain that
\begin{align}
    \G
    & = \sum_{i,c,c'} w_{i,c,c'} \g_{i,c,c'} \g_{i,c,c'}^\T \\
    & = \sum_{i,c,c'} w_{i,c,c'} ({\h_{i,c}} \odot \deltaa_{i,c,c'}) ( \h_{i,c} \odot \deltaa_{i,c,c'})^\T \\
    & = \sum_{i,c,c'} w_{i,c,c'} (\h_{i,c} \h_{i,c}^\T) \odot (\deltaa_{i,c,c'} \deltaa_{i,c,c'}^\T).
     \label{eq:G_kronecker}
\end{align}
Similarly, the subset of $\G$ corresponding to the $l$'th layer is given by
\begin{equation*}
    \G^l = \sum_{i,c,c'} w_{i,c,c'} (\h_{i,c}^{l-1} {\h_{i,c}^{l-1}}^\T) \otimes (\deltaa_{i,c,c'}^l {\deltaa_{i,c,c'}^l}^\T).
\end{equation*}
The above equation relates $\G$ to the backpropagated errors $\deltaa_{i,c,c'}$ and the features $\h_{i,c}$. As such, in the next subsection we study the second moment of the features, and in the following subsection the weighted second moment of the backpropagated errors.

\subsection{Class block structure in features} \label{sec:features}
Consider the second moment of the $l$'th layer features,
\begin{equation*}
    \H^l = \Ave_{i,c} \h_{i,c}^l {\h_{i,c}^l}^\T.
\end{equation*}
Define the \textit{feature class means}:
\begin{equation*}
    \h_c^l = \Ave_i \h_{i,c}^l,
\end{equation*}
and the \textit{feature global mean}:
\begin{equation*}
    \h_G^l = \Ave_c \h_c^l.
\end{equation*}
Define further the \textit{between-class feature second moment}:
\begin{equation*}
    \HC^l = \Ave_c \h_c^l {\h_c^l}^\T,
\end{equation*}
and the \textit{within-class feature covariance},
\begin{equation*}
    \HW^l = \Ave_{i,c} (\h_{i,c}^l - \h_c^l) (\h_{i,c}^l - \h_c^l)^\T.
\end{equation*}
Using a mean-variance decomposition, one can show that
\begin{equation*}
    \H^l = \HC^l + \HW^l.
\end{equation*}
Define also a class-specific second moment matrix,
\begin{equation*}
    \H_c = \Ave_i \h_{i,c}^l {\h_{i,c}^l}^\T,
\end{equation*}
which is related to the global one through the following equation,
\begin{equation*}
    \H^l = \Ave_c \H_c^l.
\end{equation*}

\begin{figure}[t]
    \centering
    \begin{overpic}[width=0.975\textwidth,right]{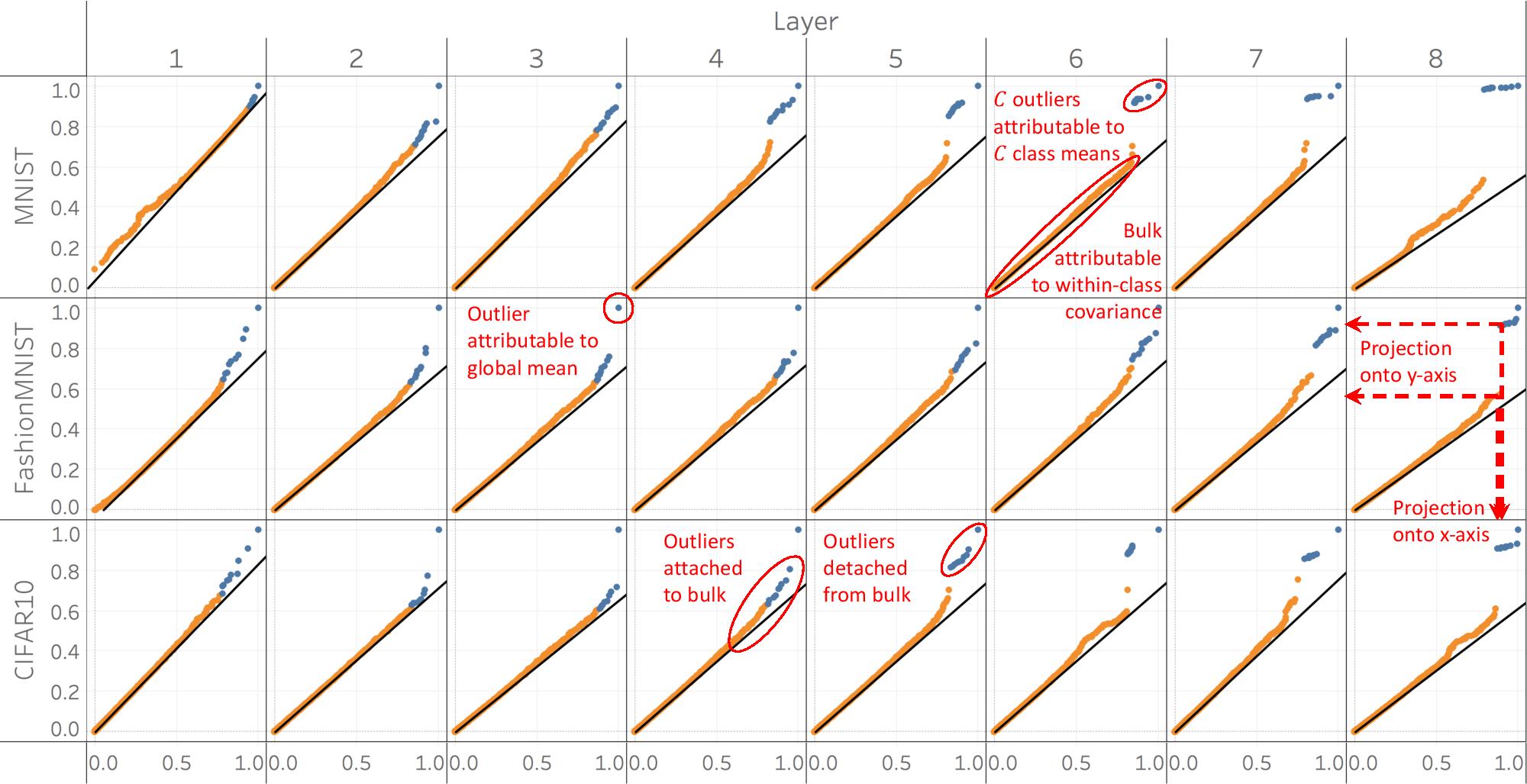}
        \scriptsize{
        \put(-2,95){\begin{sideways}$\log \lambda(\H^l)$\end{sideways}}
        \put(218,-10){$\log \lambda(\H^l \nparallel \HC^l)$}
        }
    \end{overpic}
    \smallskip
    \caption{\textbf{Attribution via knockouts of spectral features in the spectrum of deepnet features.} Each row of panels corresponds to a different dataset and each column to a different layer. Each panel shows a scatter plot of the top $750$ log-eigenvalues of $\H^l$ versus those of $\H^l \nparallel \HC^l$. The black curve within each panel is the identity line. The top-$C$ outliers are marked in blue and the rest of the eigenvalues are marked in orange. The sixth column of the first row shows the $C$ outliers, attributable to the between-class covariance, $\HC^l$, and the bulk, attributable to the within-class covariance, $\HW^l$. The third column of the second row shows the biggest outlier, which is attributable to the global mean, $\h_G$. The fourth and fifth columns of the third row show the layer in which the $C$ outliers separate from the bulk. The network is an eight-layer MLP with 2048 neurons in each hidden layer. For visibility purposes, the values on the x-axis and y-axis are normalized to the range $[0,1]$. See Section \ref{sec:knockout_features} for further details.}
    \label{fig:features_knockouts}
\end{figure}

\subsection{\texorpdfstring{$C$}{C} outliers attributable to \texorpdfstring{$\HC$}{HC} and bulk attributable to \texorpdfstring{$\HW$}{HW}} \label{sec:knockout_features}
In Figure \ref{fig:features_knockouts} we attribute via knockouts the top $C$ outliers in the spectrum of $\H^l$ to the spectrum of $\HC^l$. One can similarly attribute via knockouts the largest of these $C$ outliers to the second moment of the global mean. Specifically, Figure \ref{fig:features_knockouts} shows scatter plots of $\lambda_i(\H^l)$ versus $\lambda_i(\H^l \nparallel \HC^l)$. The last column of the second row illustrates the projection of a blue and an orange point on the x- and y-axes. Notice how their projections on the y-axis are dramatically farther apart than their projections on the x-axis. The projections on the y-axis are far apart because the spectrum of $\H^l$ has outliers, corresponding to the blue points, which are separated from a bulk, corresponding to the orange points. The same two points are quite close when projected on the x-axis. This is because the outliers are close to the bulk after the knockout procedure. Since the knockout eliminated the outliers, these would-be outliers are attributed to the matrix we knocked out, i.e., $\HC^l$. Roughly speaking, the orange points are situated along the identity line, which means the bulk of the spectrum is unaffected by knockouts. The blue points, on the other had, corresponding to the top-$C$ eigenvalues, deviate substantially from the identity line. In fact, in the last few columns they deviate so strongly that they separate markedly from the orange points; see the fourth and fifth panels of the third row. In short, there is a bulk-and-outliers structure in the spectrum of $\H^l$ and the outliers emerge from the bulk with increasing depth.

\subsection{Gradual separation and whitening of feature class-means} \label{sec:features_grad_sep}
Figure \ref{features_FCNet_bn_i} plots the same spectra of $\H^l$, together with the $C$ eigenvalues of $\HC^l$. Compare and contrast the first and last columns of the first row. In the first column of the first row, the top-$C$ eigenvalues associated with the eigenvalues of $\HC^l$ are smaller than the top-$C$ eigenvalues of $\H^l$. Those eigenvalues of $\HC^l$ are too small to induce outliers in the spectrum of $\HC^l$. In the last column of the first row, the top-$C$ eigenvalues of $\H^l$ and the top-$C$ eigenvalues of $\HC^l$ are close to each other--which is already expected; recall that we attributed these outliers to $\HC^l$ in the previous subsection. These eigenvalues are dramatically larger than later eigenvalues. Again, there is a bulk and $C$ eigenvalues sticking out of the bulk.

Observe the increasing separation, with increasing depth, of the $C$ eigenvalues of $\HC^l$ from the bulk. Underlying this, the class means have more ``energy'' as measured by $\tr{\HC^l}$. The class means grow energetic with depth and this causes separation. Statistically, this means the features at later layers are more discriminative.

Observe the eigenvalues of $\HC^l$ in different columns of the second row of panels, together with the whiskers highlighting the ratio between the second-largest and smallest eigenvalue. Notice how the ratio decreases as function of depth, i.e., the eigenvalues of $\HC^l$ become increasingly closer to each other. This implies the deepnet finds features whose class means are increasingly closer to orthogonal (once these class means are centered by a global mean, which eliminates the first outlier).

\begin{figure}[h!]
        \centering
        \includegraphics[width=1\textwidth]{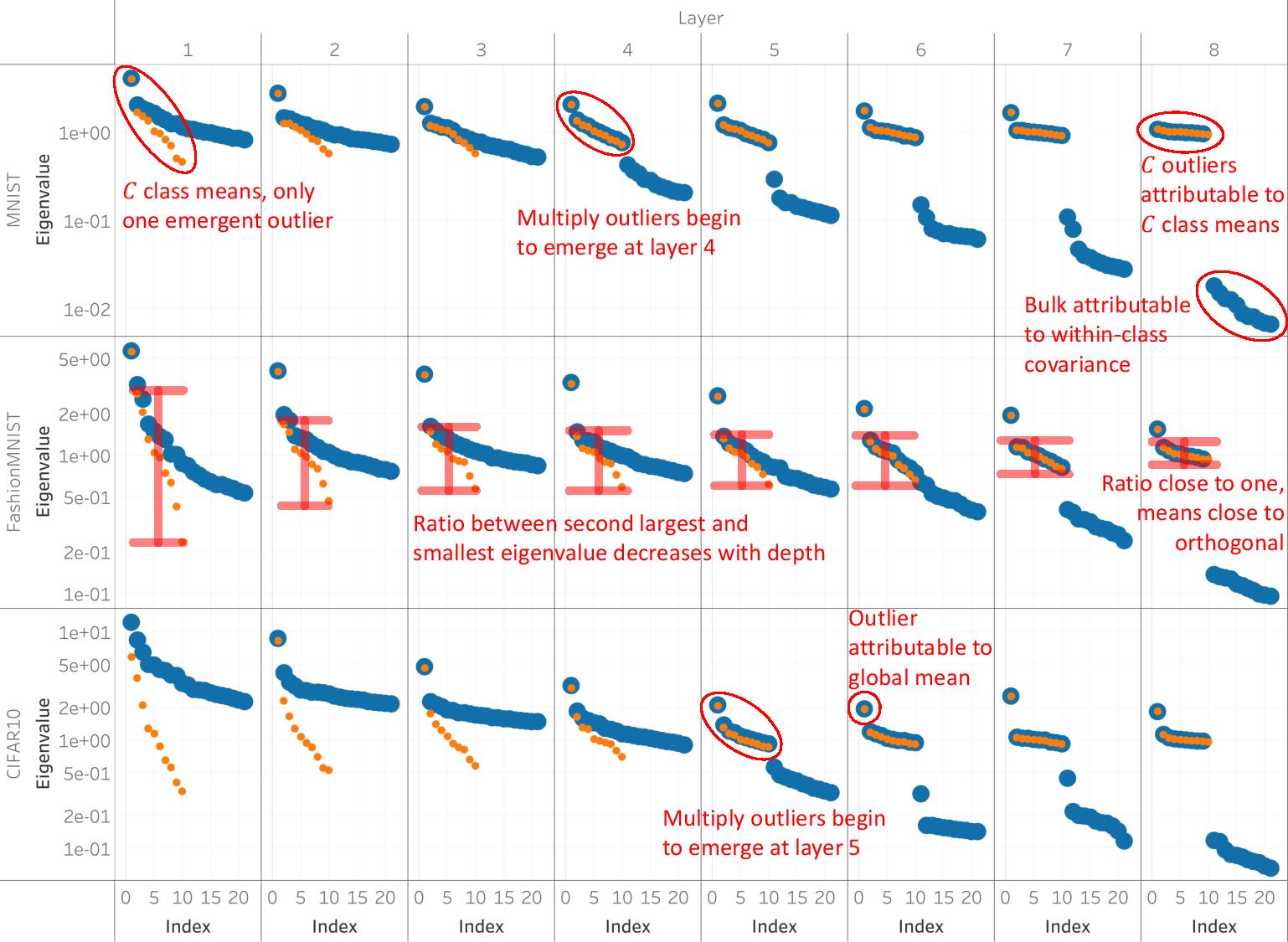}
    \caption{\textbf{Scree plots of top eigenvalues in the spectrum of deepnet features.} An alternative presentation of the data in Figure \ref{fig:features_knockouts}. Each row of panels corresponds to a different dataset and each column to a different layer. Each panel depicts spectra of $\H^l$ in blue and $\HC^l$ in orange. Eigenvalues in each panel are normalized by the median of $\left\{ \log(\lambda_i(\HC^l)) \right\}_{i=1}^C$. The panel in the last column of the first row highlights the ``bulk and $C$ outliers'' pattern. The panels in the second row show whiskers whose edges are located at the second largest and at the smallest eigenvalue of $\HC^l$. Notice how the ratio between the whisker edges decreases as function of depth, until a point where all eigenvalues are of similar magnitude. The third row highlights the top outlier, which is attributable to the global mean, $\h_G$. The network is an eight-layer MLP with 2048 neurons in each hidden layer. For more details, see Section \ref{sec:features_grad_sep}.}
    \label{features_FCNet_bn_i}
\end{figure}

\subsection{Cross-class block structure in backpropagated errors} \label{sec:errors}
Consider the weighted second moment of backpropagated errors:
\begin{equation*}
    \Deltaa^l = \frac{1}{NC} \sum_{i,c,c'} w_{i,c,c'} \deltaa_{i,c,c'}^l {\deltaa_{i,c,c'}^l}^\T.
\end{equation*}
Define the \textit{backpropagated errors cross-class means}:
\begin{equation*}
    \deltaa_{c,c'}^l = \sum_i \pi_{i,c,c'} \deltaa_{i,c,c'}^l,
\end{equation*}
and their corresponding \textit{within-cross-class covariance}:
\begin{equation*}
    \DeltaW^l = \sum_{i,c,c'} w_{i,c,c'} (\deltaa_{i,c,c'}^l - \deltaa_{c,c'}^l) (\deltaa_{i,c,c'}^l - \deltaa_{c,c'}^l)^\T.
\end{equation*}
These equations group together backpropagated errors for a fixed pair of $c,c'$. Define also the \textit{backpropagated errors class means}:
\begin{equation*}
    \deltaa_c^l = \sum_{c' \neq c} \pi_{c,c'} \deltaa_{c,c'}^l,
\end{equation*}
the \textit{between-class backpropagated error second moment}:
\begin{equation*}
    \DeltaC^l = \sum_c w_c \deltaa_c^l {\deltaa_c^l}^\T,
\end{equation*}
and the \textit{between-cross-class backpropagated error covariance}:
\begin{equation*}
    \DeltaCP^l = \sum_{\substack{c' \neq c \\ c}} w_{c,c'} (\deltaa_{c,c'}^l - \deltaa_{c'}^l) (\deltaa_{c,c'}^l - \deltaa_{c'}^l)^\T.
\end{equation*}
In these equations, the backpropagated errors with a fixed $c$ (and possibly varying $c'$) are clustered together. Leveraging these definitions, we obtain the following decomposition of $\Deltaa$:
\begin{align*}
    \Deltaa^l & = \DeltaC^l + \DeltaCP^l + \DeltaW^l + \underbrace{\sum_c w_{c,c} \deltaa_{c,c}^l {\deltaa_{c,c}^l}^\T}_{\Deltaa_{c = c'}}.
\end{align*}
The above decomposition is similar to the decomposition of the gradients in Section \ref{decomp_G} and its proof is therefore omitted. We can also define a class-specific weighted second moment given by,
\begin{equation*}
    \Deltaa_c^l = \frac{1}{N} \sum_{i,c'} w_{i,c,c'} \deltaa_{i,c,c'}^l {\deltaa_{i,c,c'}^l}^\T.
\end{equation*}

\begin{figure}
    \centering
    \begin{subfigure}[t]{1\textwidth}
        \centering
        \begin{minipage}{0.9\textwidth}
        \begin{overpic}[width=0.975\textwidth,right]{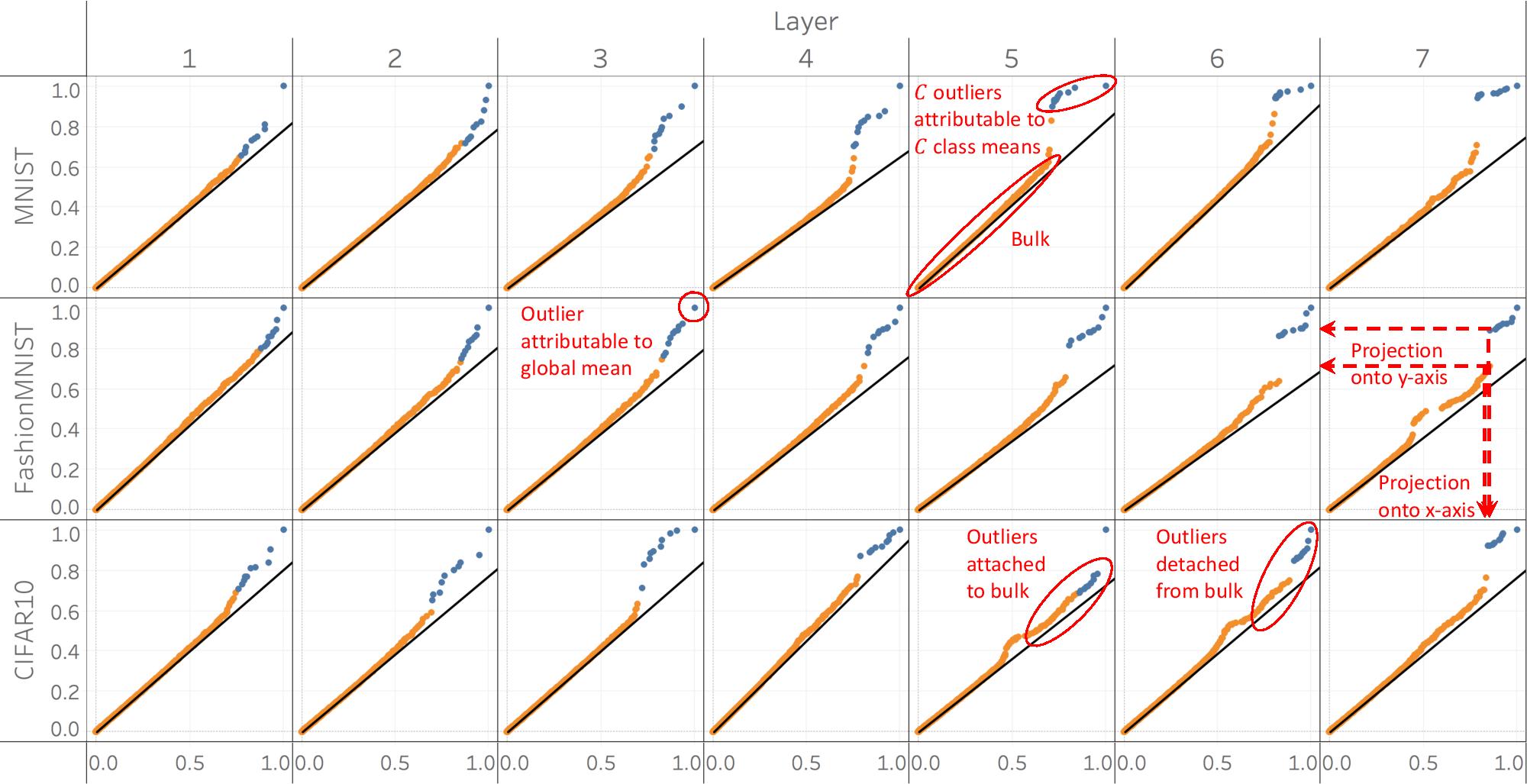}
        \scriptsize{
        \put(-2,95){\begin{sideways}$\log \lambda(\Deltaa^l)$\end{sideways}}
        \put(195,-10){$\log \lambda(\Deltaa^l \nparallel \DeltaC^l)$}
        }
        \end{overpic}
        \end{minipage}
        \bigskip
        \caption{\textbf{Top-$C$ outliers attributed to $C$ class means.} Each panel shows a scatter plot of top the $750$ log-eigenvalues of $\Deltaa^l$ versus those of $\Deltaa^l \nparallel \Deltaa_C^l$. The top-$C$ outliers are marked in blue and the rest of the eigenvalues are marked in orange. The fifth column of the first row shows the $C$ outliers, attributable to the between-class covariance, $\DeltaC^l$, and a bulk. The third column of the second row shows the biggest outlier, which is attributable to the global mean, $\deltaa_G$. The fifth and sixth columns of the third row show the layer in which the $C$ outliers separate from the bulk.}
        \label{backprop_errors_knockouts_C}
    \end{subfigure}
    \begin{subfigure}[t]{1\textwidth}
        \centering
        \begin{minipage}{0.9\textwidth}\begin{overpic}[width=0.975\textwidth,right]{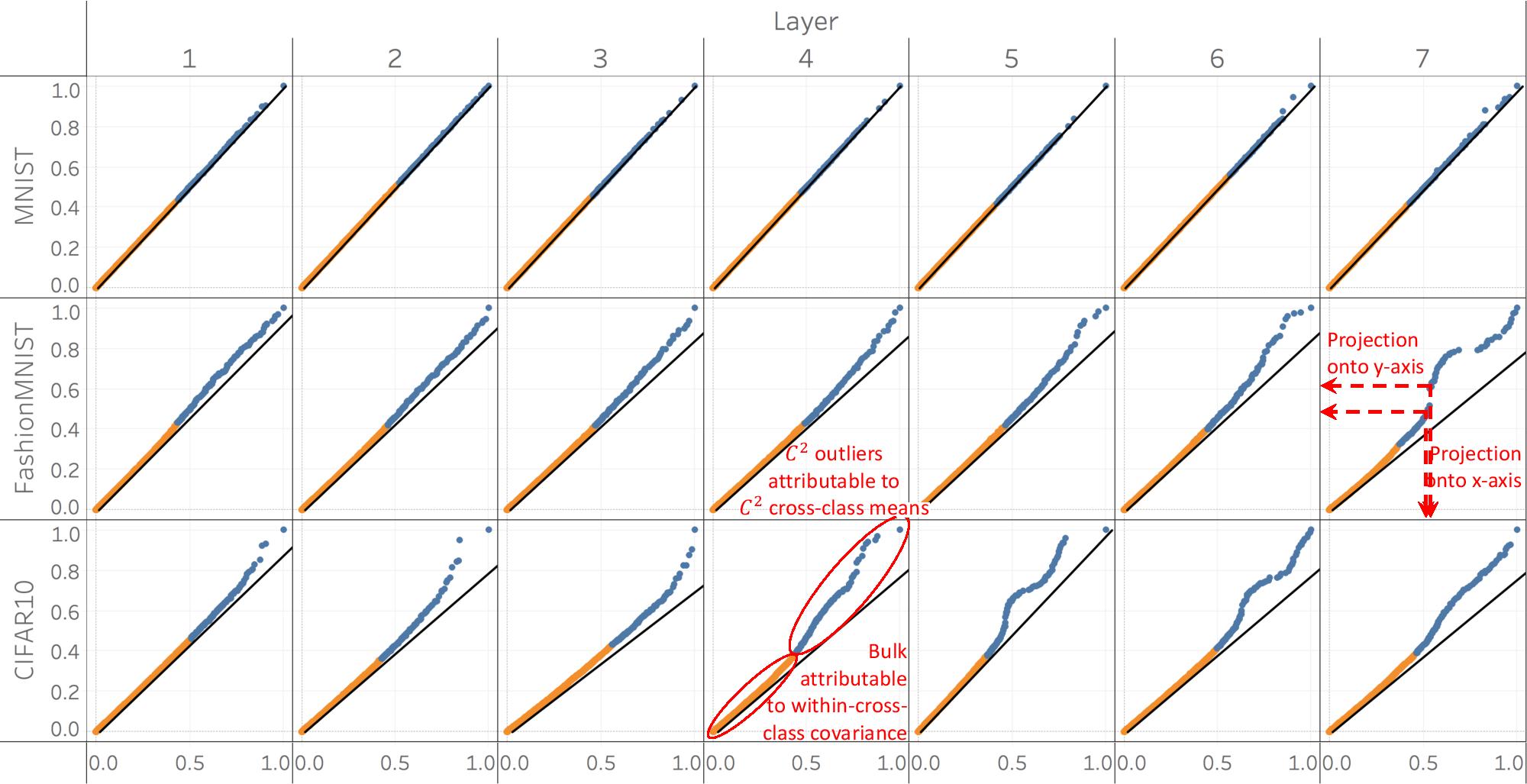}
        \scriptsize{
        \put(-5,80){\begin{sideways}$\log \lambda(\Deltaa^l \nparallel \DeltaC^l)$\end{sideways}}
        \put(170,-10){$\log \lambda \left( \Deltaa^l \nparallel \left( \DeltaC^l + \DeltaCP^l \right) \right)$}
        }
        \end{overpic}
        \end{minipage}
        \bigskip
        \caption{\textbf{Top-$C^2$ outliers attributed to $C^2$ cross-class means.} Each panel shows a scatter plot of the top $750$ log-eigenvalues of $\Deltaa^l \nparallel \DeltaC^l$ versus those of $\Deltaa^l \nparallel (\DeltaC^l+\DeltaCP^l)$. The top-$C^2$ outliers are marked in blue and the rest are marked in orange. The fourth column of the third row shows the $C^2$ outliers, attributable to the between-cross-class covariance, $\DeltaCP^l$, and a bulk, attributable to the within-cross-class covariance $\DeltaW^l$.}
        \label{fig:backprop_errors_knockouts_C^2}
    \end{subfigure}
    \caption{\textbf{Attribution via knockouts of spectral features in the spectrum of backpropagated errors.} Eight-layer MLP with 2048 neurons in each hidden layer. Within each subfigure, each row of panels corresponds to a different dataset and each column to a different layer. The black curve within each panel is the identity line. For visibility purposes, the values on the x-axis and y-axis are normalized to the range $[0,1]$. See Section \ref{sec:backprop_errors_knockouts} for further details.}
    \label{fig:backprop_errors_knockouts}
\end{figure}

\subsection{\texorpdfstring{$C$}{C} outliers attributable to \texorpdfstring{$\DeltaC$}{DeltaC}, \texorpdfstring{$C^2$}{C2} outliers attributable to \texorpdfstring{$\DeltaCP$}{DeltaCP} and bulk attributable to \texorpdfstring{$\DeltaW$}{DeltaW}} \label{sec:backprop_errors_knockouts}
In Figure \ref{backprop_errors_knockouts_C} we attribute via knockouts the top $C$ outliers in $\Deltaa^l$ to the spectrum of $\DeltaC^l$. Specifically, Figure depicts scatter plots of $\lambda_i(\Deltaa^l)$ versus $\lambda_i(\Deltaa^l \nparallel \DeltaC^l)$. The last column of the second row illustrates the projection of a blue and an orange point on the principal axes. Their projections on the y-axis are dramatically farther apart than their projections on the x-axis. Their projections on the y-axis are far apart because the spectrum of $\Deltaa^l$ has outliers, corresponding to the blue points, which are separated from a bulk, corresponding to the orange points. The same two points are quite close when projected on the x-axis. This is because the outliers are close to the bulk after the knockout procedure. Since the knockout eliminated the outliers, these would-be outliers are attributed to the matrix we knocked out, i.e., $\DeltaC^l$. The orange points are situated along the identity line, which means the bulk of the spectrum is unaffected by knockouts. The blue points, corresponding to the top-$C$ eigenvalues, deviate substantially from the identity line. In some panels they deviate so strongly that they separate markedly from the orange points; see fifth and sixth panels of the third row. In short, there is a bulk-and-outliers structure in the spectrum of $\Deltaa^l$ and the outliers emerge from the bulk with increasing depth.

In Figure \ref{fig:backprop_errors_knockouts_C^2} we attribute via knockouts the top $C^2$ outliers in the spectrum of $\Deltaa^l$ to the spectrum of $\DeltaC+\DeltaCP$. Specifically, Figure \ref{fig:backprop_errors_knockouts_C^2} depicts scatter plots of $\lambda_i(\Deltaa^l \nparallel \DeltaC^l)$ versus $\lambda_i(\Deltaa^l \nparallel (\DeltaC^l+\DeltaCP))$. The seventh column of the second row illustrates the projection of two blue points--not a blue and an orange point, as was done previously--onto the principal axes. Notice the gap dividing the set of all blue points into two groups: one that is close to the orange points, and another which is well-separated from the orange points. Notice how the projections of these two blue points onto the y-axis are dramatically father apart than their projections on the x-axis. Their projection onto the y-axis are far apart because the spectrum of $\Deltaa^l \nparallel \DeltaC^l$ has outliers, corresponding to one of the groups of blue points. These are separated from a bulk, comprised of the other group of blue points and also the orange points. The same two points are quite close when projected on the x-axis. This is because the outliers are close to the bulk after the knockout procedure. Since knockouting out $\DeltaCP$ eliminated the outliers in $\Deltaa^l \nparallel (\DeltaC^l+\DeltaCP)$, these would-be outliers are attributable to the matrix we knocked out, i.e., $\DeltaCP^l$. Some of the blue points and all of the orange points are situated along the identity line, which means the bulk of the spectrum is unaffected by the knockout procedure.

\begin{figure}[h!]
        \centering
        \includegraphics[width=1\textwidth]{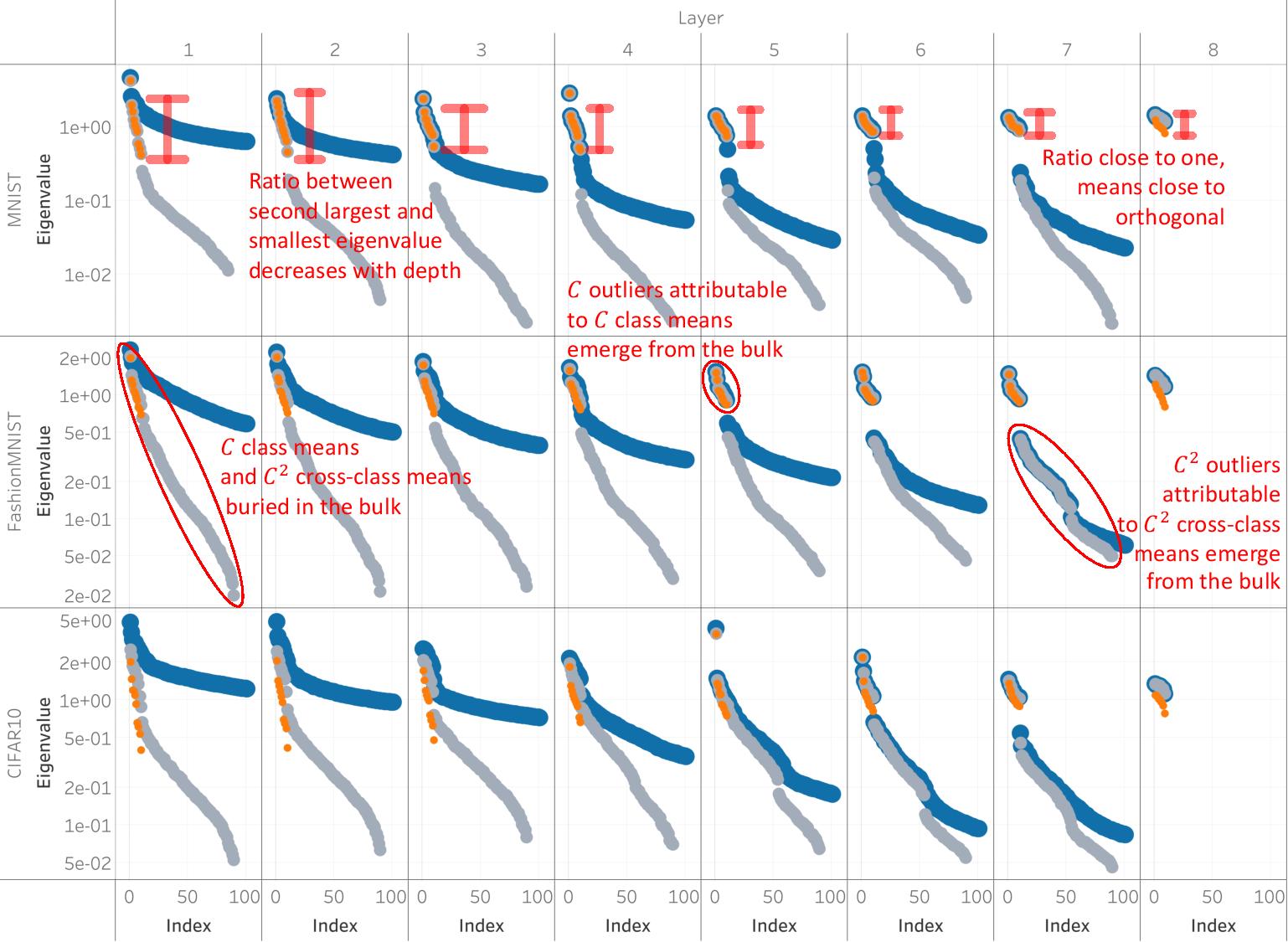}
    \caption{\textbf{Scree plot of top eigenvalues in spectra of deepnet backpropagated errors.} An alternative presentation of the data in Figure \ref{fig:backprop_errors_knockouts}. Each row of panels corresponds to a different dataset and each column to a different layer. Each panel depicts spectra of $\Deltaa^l$ in blue, $\DeltaC^l$ in orange and $\DeltaC^l+\DeltaCP^l$ in gray. Eigenvalues in each panel are normalized by the median of $\{ \log(\lambda_i( \DeltaC^l)) \}_{i=1}^C$. The first, fifth  and eighth columns of the second row track: (i) the emergence of the $C$ outliers from the bulk, (ii) the emergence of the $C^2$ outliers from the bulk. The panels in the first row show whiskers whose edges are located at the second largest and at the smallest eigenvalue of $\DeltaC^l$. Notice how the ratio between the whisker edges decreases as function of depth, until a point where all eigenvalues are of similar magnitude. The network is an eight-layer MLP with 2048 neurons in each hidden layer. For more details, see Section \ref{sec:jacobians_grad_sep}.}
    \label{jacobians_FCNet_bn_i}
\end{figure}

\subsection{Gradual separation and whitening of backpropagated error class-means} \label{sec:jacobians_grad_sep}
Figure \ref{jacobians_FCNet_bn_i} plots the same spectra of $\Deltaa^l$, together with the eigenvalues of $\DeltaC^l$ and $\DeltaC^l+\DeltaCP^l$. Compare and contrast the first, fifth and seventh column of the second row. In the first column of the second row, the top-$C$ eigenvalues associated with the eigenvalues of $\DeltaC^l$, and the next $C^2$ outliers associated with the eigenvalues of $\DeltaCP^l$, are smaller than the top eigenvalues of $\Deltaa^l$. Those eigenvalues of $\DeltaC^l$ and $\DeltaCP^l$ are too small to induce outliers in the spectrum of $\Deltaa^l$. In the fifth column of the second row, the top-$C$ eigenvalues of $\Deltaa^l$ and the top-$C$ eigenvalues of $\DeltaC^l$ are close to each other, which is already expected; recall that we attributed the top-$C$ outliers in the spectrum of $\Deltaa^l$ to $\DeltaC^l$ in the previous subsection. Moreover, the $C^2$ eigenvalues associated with $\DeltaCP^l$ are smaller than the top $C^2$ eigenvalues of $\Deltaa^l$. Those eigenvalues are too small to induce outliers in the spectrum of $\Deltaa^l$. In the seventh column of the second row, some of the top-$C^2$ eigenvalues of $\Deltaa^l$ and the $C^2$ eigenvalues of $\DeltaCP^l$ are close to each other. This is again already expected; recall that we attributed the top-$C$ outliers in $\Deltaa^l$ to $\DeltaC^l$ and the top-$C^2$ outliers in $\Deltaa^l$ to $\DeltaCP^l$.

Observe the increasing separation, with increasing depth, of the $C$ eigenvalues of $\DeltaC^l$ from the bulk. Underlying this, the class means have more ``energy'' as measured by $\tr{\DeltaC^l}$. The class means grow energetic with depth and this causes separation. Heuristically, this figure implies that the backpropagated errors contain strong class information in deeper layers, but this information is gradually lost during backpropagation to successively shallower layers.

Observe in the panels of the first row the eigenvalues of $\DeltaC^l$ throughout the layers. The whiskers highlight the ratio between the second-largest and the smallest eigenvalue. Note how this ratio decreases with depth; by the last layer, all eigenvalues are of fairly similar magnitude. This implies the error class means are ultimately very close to being orthogonal, but their orthogonality is gradually deteriorating as they are backpropagated throughout the layers.

\begin{figure}
    \centering
    \includegraphics[width=1\textwidth]{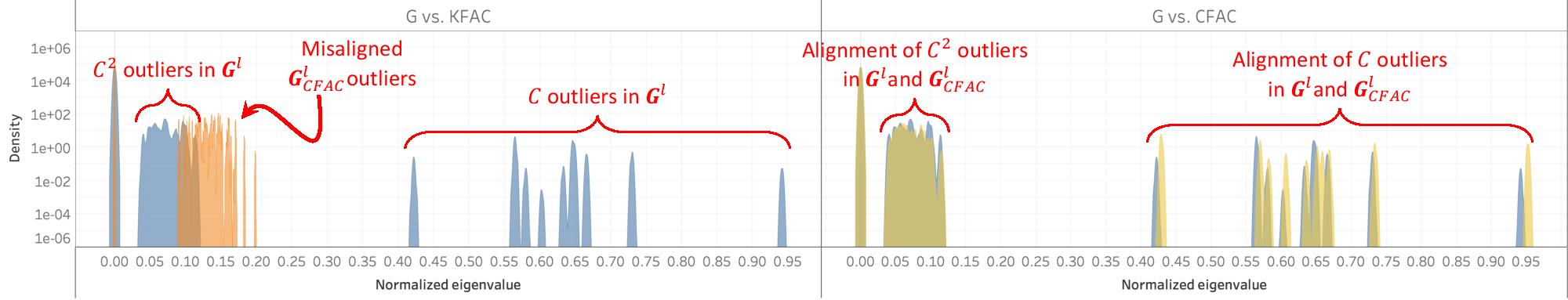}
    \caption{\textbf{Spectra of $\G^l$, $\KFAC^l$ and $\CFAC^l$ for layer $l=8$.} Each panel plots the spectrum of $\G^l$ in blue. The left panel highlights features in the spectrum of $\G$, which include $C$ outliers, as well as a mini-bulk composed of $C^2$ outliers. The main bulk in the spectrum of the last layer features is concentrated and appears as a narrow bump around the origin. In addition, the left panel plots the spectrum of $\KFAC^l$ in orange, while the right one plots the spectrum of $\CFAC^l$ in yellow. Notice that outliers in the spectrum of $\KFAC^l$ \textbf{do not} align well with outliers in the spectrum of $\G^l$. On the other hand, notice that the $C$ and $C^2$ outliers in the spectrum of $\CFAC^l$ \textbf{do align} well with corresponding outliers in the spectrum of $\G^l$. The network is an eight-layer MLP with 2048 neurons in each hidden layer trained on MNIST. The y-axis is on a logarithmic scale.}
    \label{KFAC_vs_CFAC_single_layer}
\end{figure}

\begin{figure}
    \centering
    \begin{subfigure}[t]{0.8\textwidth}
        \centering
        \includegraphics[width=1\textwidth]{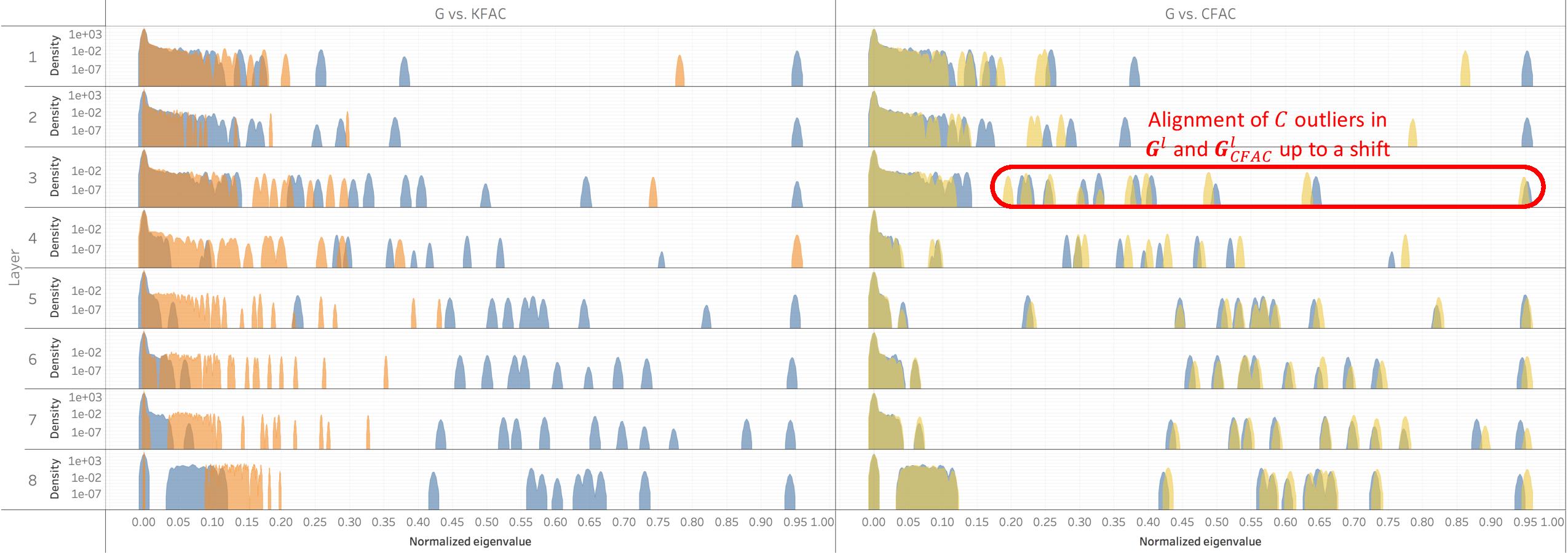}
        \caption{MNIST}
        \smallskip
    \end{subfigure}
    \begin{subfigure}[t]{0.8\textwidth}
        \centering
        \includegraphics[width=1\textwidth]{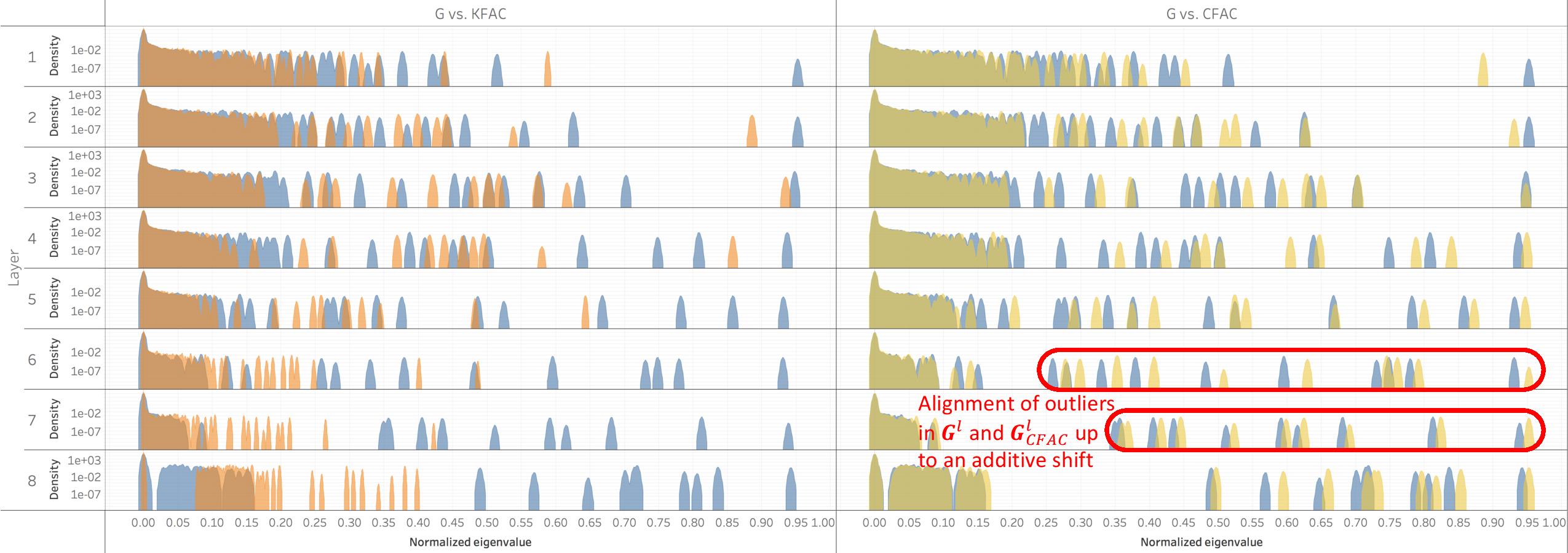}
        \caption{FashionMNIST}
        \smallskip
    \end{subfigure}
    \begin{subfigure}[t]{0.8\textwidth}
        \centering
        \includegraphics[width=1\textwidth]{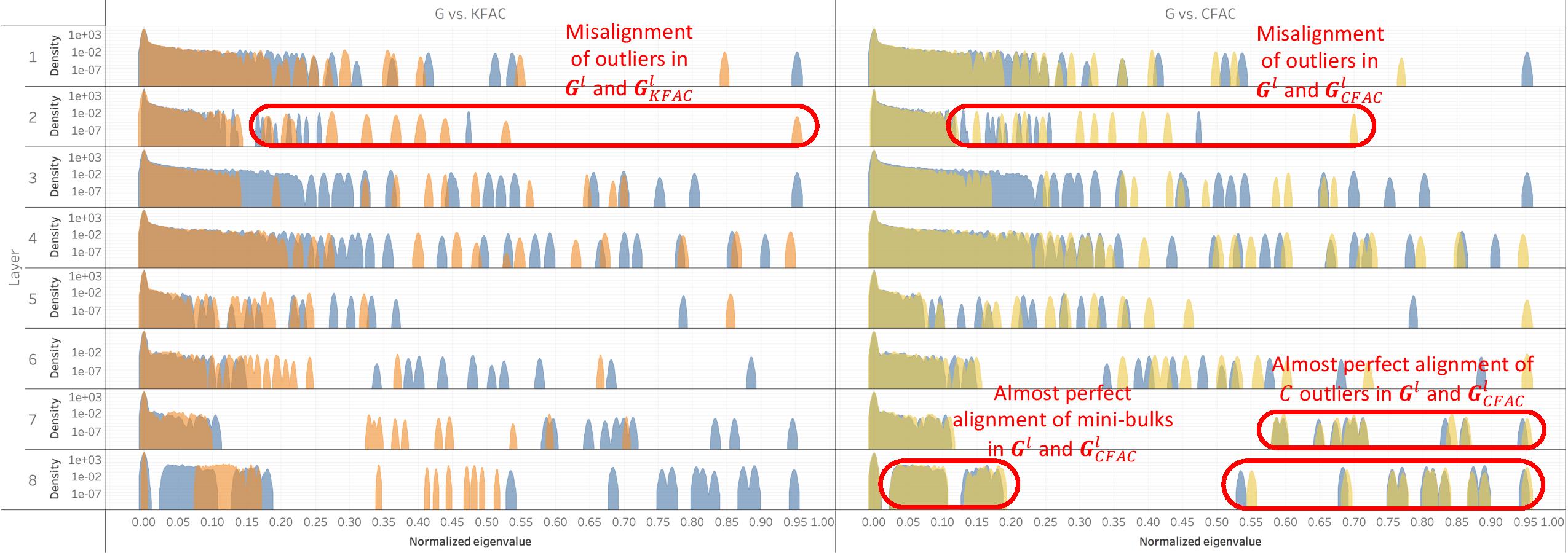}
        \caption{CIFAR10}
    \end{subfigure}
    \caption{\textbf{Spectra of $\G^l$, $\KFAC^l$ and $\CFAC^l$.} Each subfigure corresponds to one specific dataset. Each row of panels corresponds to one specific layer. In each panel, the spectrum of $\G^l$ is plotted in blue. The first column of panels plots the spectrum of $\KFAC^l$ in orange, while the second column of panels plots the spectrum of $\CFAC$ in yellow. Notice how the spectra of $\G^l$ and $\KFAC^l$ do not align well, as evident by the lack of agreement of blue and orange outliers, mini-bulks and main bulks. Conversely, notice how the spectra of $\G^l$ and $\CFAC^l$ align much better, especially at deeper layers. At the seventh and eighth layers of CIFAR10, the outliers and mini-bulks in the spectra of $\G^l$ and $\CFAC^l$ align almost perfectly. In other cases, such as the sixth and seventh layers of FashionMNIST, the outliers in the spectra of $\G^l$ and $\CFAC^l$ align up to a small shift. In some shallower layers, such as the third layer of MNIST, the outliers in the spectrum of $\G^l$ align with those in the spectrum of $\CFAC^l$. In other cases, such as the second layer of CIFAR10, the misalignment of the outliers in the spectra of $\G^l$ and $\CFAC^l$ is as bad as the misalignment between the outliers of $\G^l$ and $\KFAC^l$. The network is an eight-layer MLP with 2048 neurons in each hidden layer. The y-axis is on a logarithmic scale.}
    \label{KFAC_vs_CFAC}
\end{figure}

\subsection{Class-distinct factorized approximate curvature} \label{sec:CFAC}
Traditional quadratic optimization can be readily solved using a single Newton's method step, by moving in the direction of the gradient preconditioned by $\Hess^{-1}$. For nonlinear least squares problems, one can utilize the Gauss–Newton algorithm, which replaces $\Hess^{-1}$ by $\G^{-1}$ and performs several steps, instead of just one. In the context of deepnets, Gauss-Newton is problematic, since the matrix of $\G$ is too large for it to be stored in memory. Researchers would like to find other matrices that might precondition the gradient \citep{grosse2015scaling,ye2018hessian,wang2019utilizing,xu2019newton,kylasa2019gpu,xu2020second}. For example, \citet{martens2015optimizing} proposed the KFAC matrix as a proxy for $\G$.
\begin{defn}\textup{{\textbf{(KFAC; \citet{martens2015optimizing}).}}}
The KFAC matrix is defined as follows:
\begin{align*}
    \KFAC
    & = \Ave_{i,c} \left\{ \h_{i,c} \h_{i,c}^\T  \right\} \odot \left( \sum_{i,c,c'} w_{i,c,c'} \deltaa_{i,c,c'} \deltaa_{i,c,c'}^\T  \right) \\
    & = \H \odot \Deltaa,
\end{align*}
where $\odot$ denotes Khatri-Rao product. The block diagonal KFAC matrix is given by:
\begin{align*}
    \KFAC^l
    & = \Ave_{i,c} \left\{ \h_{i,c}^l {\h_{i,c}^l}^\T \right\} \otimes \left( \sum_{i,c,c'} w_{i,c,c'} \deltaa_{i,c,c'}^l {\deltaa_{i,c,c'}^l}^\T  \right) \\
    & = \H^l \otimes \Deltaa^l,
\end{align*}
where $\otimes$ is the Kronecker product\footnote{Hence the name of their method--Kronecker-Factored Approximation of Curvature (KFAC).}.
\end{defn}
\noindent The idea of using Kronecker-based preconditioners can be traced to earlier works by \citet{van1993approximation}, and the use of a block-diagonal approximation can be traced back to \citet{collobert2004gentle}. KFAC has proven to be very useful in practice, and since its conception several generalizations and applications were proposed for it \citep{grosse2016kronecker,ba2016distributed,wu2017second,wu2017scalable,wang2019eigendamage}.

Tools developed in this present work can be used to study the spectra of these matrices, and thereby to also understand the extent (or not) of agreement between the spectra of $\G$ and $\KFAC$. When we apply our attribution tools (see left panels of Figures \ref{KFAC_vs_CFAC_single_layer} and \ref{KFAC_vs_CFAC}), we notice very distinct failures of approximation. $\KFAC$ has its mini-bulks and outliers drastically misaligned from those of $\G$. We propose an alternative approximation for $\G$.
\begin{defn}\textup{{(\textbf{Class-distinct Factorized Approximate Curvature)T.}}}
    The CFAC matrix is given by:
    \begin{equation*}
        \CFAC = \Ave_c \H_c \odot \Deltaa_c,
    \end{equation*}
    and the block diagonal CFAC matrix is given by:
    \begin{equation*}
        \CFAC^l = \Ave_c \H_c^l \otimes \Deltaa_c^l.
    \end{equation*}
\end{defn}

\subsection{CFAC is a better proxy for \texorpdfstring{$\G$}{G} than KFAC}
In Figure \ref{KFAC_vs_CFAC_single_layer} we compare the spectra of $\G^l$, $\KFAC^l$ and $\CFAC^l$ for the last layer of an MLP trained on MNIST. Inspection reveals that the spectra of $\CFAC^l$ and $\G^l$ align much better than do those of $\KFAC^l$ and $\G^l$. In Figure \ref{KFAC_vs_CFAC}, we perform a more comprehensive comparison, which considers all the layers of the MLP across three canonical datasets. At deeper layers, the spectra of $\CFAC^l$ and $\G^l$ again align much better than $\KFAC$ and $\G^l$. At shallower layers, improvement in alignment depends on the dataset.

\subsection{Relating the class/cross-class structure in the features, backpropagated errors and \texorpdfstring{$\G$}{G}}
Recall that $\G$, $\H$ and $\Deltaa$ all have $C$-dimensional eigenspaces attributable to the class-specific mean structure. In KFAC, each of the $C$ features class means would be Khatri-Rao-multiplied by each of the $C$ error class means. In other words, the feature mean of one class would be multiplied by the error mean of another class. In CFAC, on the other hand, the feature mean of a class would \textbf{only} be multiplied by the error mean of that class.

The previous subsection suggests that the class means of gradients, which cause the outliers in $\G$, can be approximated as the Khatri-Rao product of the $C$ feature class means and the $C$ error class means, i.e.,
\begin{equation*}
    \g_c \approx \deltaa_c \odot \h_c.
\end{equation*}
Moreover, the $C^2$ cross-class means causing the other visible outliers can be approximated as the Khatri-Rao product of the $C$ feature class means and the $C^2$ error cross-class means, i.e.,
\begin{equation*}
    \g_{c,c'} \approx \deltaa_{c,c'} \odot \h_c.
\end{equation*}

\begin{figure}[t]
    \centering
    \begin{overpic}[width=0.975\textwidth,right]{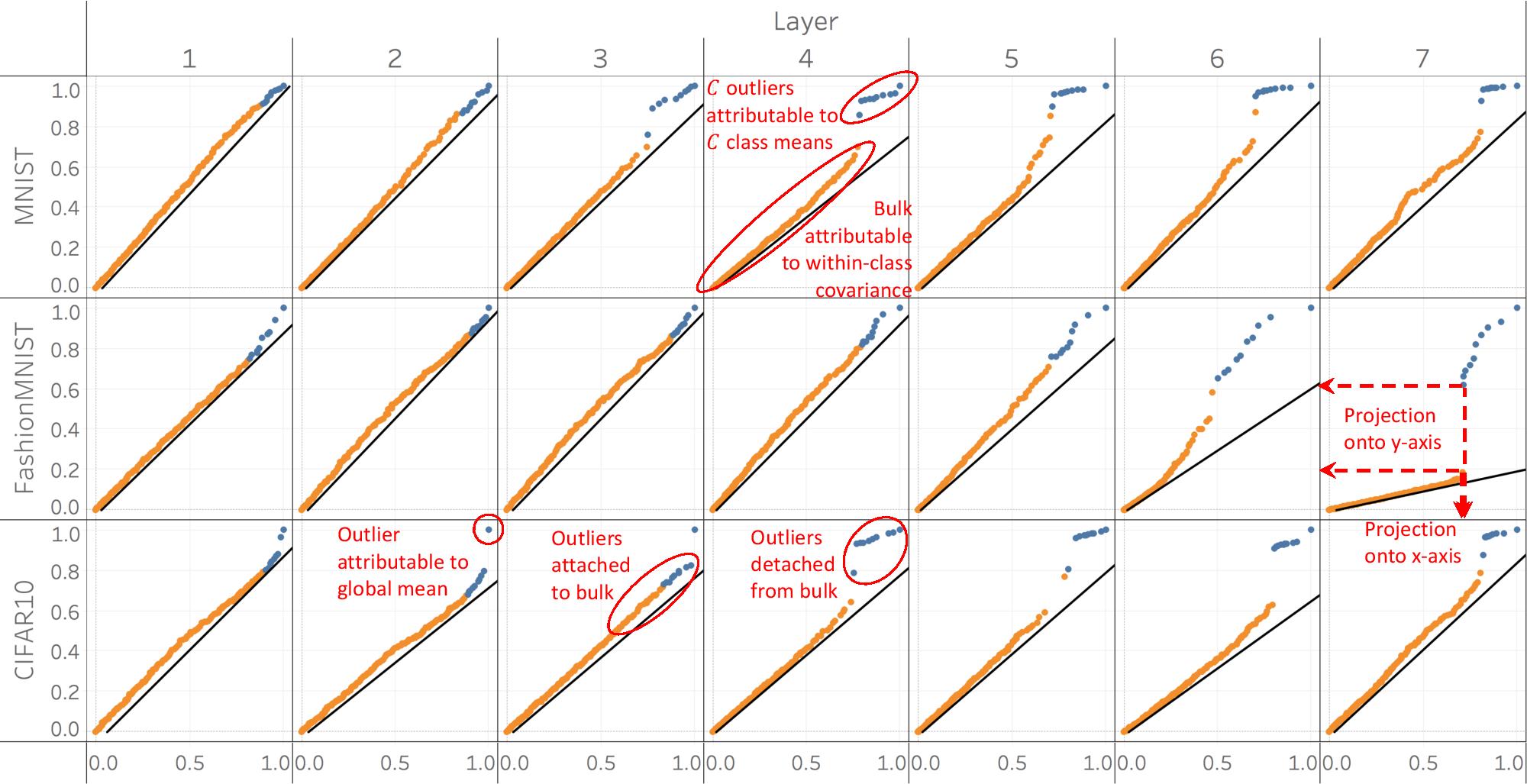}
        \scriptsize{
        \put(-2,95){\begin{sideways}$\log \lambda(\W^l)$\end{sideways}}
        \put(219,-10){$\log \lambda(\W^l \nparallel \WC^l)$}
        }
    \end{overpic}
    \smallskip
    \caption{\textbf{Attribution via knockouts of spectral features of deepnet weights.} Each row of panels corresponds to a specific dataset and each column to a specific layer. Each panel shows a scatter plot of the top $150$ log singular values of $\W^l$ versus those of $\W^l \nparallel \WC^l$. The black curve within each panel is the identity line. The top-$C$ outliers are marked in blue; other singular values are marked in orange. The fourth column of the first row shows the $C$ outliers, attributable to the between-class covariance, $\WC^l$, and the bulk, attributable to the within-class covariance, $\WW^l$. The second column of the third row shows the biggest outlier, which is attributable to the global mean, $\deltaa_G \h_G^\T$. The third and fourth columns of the third row show the layer in which the $C$ outliers separate from the bulk. The network is an eight-layer MLP with 2048 neurons in each hidden layer. For plotting purposes, values on the x-axis and y-axis are normalized to the range $[0,1]$. See Section \ref{sec:weights} for further details.}
    \label{fig:weights_knockouts}
\end{figure}

\subsection{Relating the class structure in features, backpropagated errors and weights} \label{sec:weights}
In this subsection we assume a weight decay regularization is used to train the network. Recalling Equation \eqref{eq:gradient}, the gradient of the loss with respect to the $l$'th layer weight matrix $\W^l$ is given by 
\begin{equation*}
    \pdv{\Lagr(\theta)}{\W^l} = \Ave_{i,c} \pdv{\ell( f(\x_{i,c};\thetaa), \y_{c} )}{\W^l} 
    = \Ave_{i,c} \deltaa_{i,c,c}^l {\h_{i,c}^{l-1}}^\T + \eta \W^l.
\end{equation*}
Setting the above gradient to zero, we obtain
\begin{equation*}
    \W^l = - \frac{1}{\eta} \Ave_{i,c} \deltaa_{i,c,c}^l {\h_{i,c}^{l-1}}^\T.
\end{equation*}
Define the \textit{between-class weight matrix}
\begin{equation*}
    \WC^l = - \frac{1}{\eta} \Ave_c \deltaa_{c,c}^l {\h_c^{l-1}}^\T,
\end{equation*}
where
\begin{equation*}
    \deltaa_{c,c}^l = \Ave_i \deltaa_{i,c}^l.
\end{equation*}

In Figure \ref{fig:weights_knockouts} we attribute via knockouts the top $C$ outliers in the spectrum of $\W^l$ to the spectrum of $\WC^l$. Similarly we could attribute via knockouts the largest of these $C$ outliers to the second moment of the global mean. Specifically, Figure \ref{fig:features_knockouts} shows scatter plots of singular values of $s_i(\W^l)$ versus those of $s_i(\W^l \nparallel \WC^l)$ (here we use the definition of projection knockouts for non-square matrices). The last column of the second row illustrates the projection of a blue and an orange point on the x- and y-axes. The projections on the y-axis are well-separated, which implies the spectrum of $\W^l$ has $C$ outliers, corresponding to the blue points, which are separated from a bulk, corresponding to the orange points. The projections onto the x-axis are not well-separated, which implies that the outliers if any, are largely eliminated after the knockout procedure. Since the knockout eliminated the outliers, these would-be outliers are attributed to the matrix we knocked out, i.e., $\WC^l$. Notice how the orange points are situated along the identity line, which means the bulk of the spectrum is unaffected by the knockout procedure, whereas the blue points, corresponding to the top-$C$ singular values, deviate substantially from the identity line. In fact, in the last few columns they deviate so much that they separate from the orange points, as shown by the third and fourth panels of the third row. This implies a `bulk-and-outliers' pattern exists in the spectrum of $\W^l$, and the outliers emerge from the bulk as we move towards deeper layers.

\section{Multinomial logistic regression} \label{sec:logistic}
The emergence of the pattern in Figure \ref{fig:pattern} is not restricted to spectra originating from deepnets, as it already appears in simpler examples.
In this section, we study the spectrum of the FIM of a multinomial logistic regression classifier in order to obtain intuitive understanding of (i) the different components in the spectrum of $\G$; and (ii) the relation between these components and the misclassification error. 

\subsection{FIM of multinomial logistic regression}
Assuming our network is a multinomial logistic regression classifier, we have
\begin{equation*}
    f(\x_{i,c}; \thetaa) = \W \x_{i,c}, \quad \thetaa=\vect(\W),
\end{equation*}
where we denoted by $\vect(\cdot)$ the operator forming a vector from a matrix by stacking its columns as subvector blocks within a single vector. Denote by $\otimes$ the Kronecker product. Using the property of Kronecker products, $\vect(\A\B) = (\B^\T \otimes \I_C) \vect(\A)$, we can rewrite the classifier as follows:
\begin{equation*}
    f(\x_{i,c}; \thetaa)
    = \W \x_{i,c}
    = (\x_{i,c}^\T \otimes \I_C) \vect(\W)
    = (\x_{i,c}^\T \otimes \I_C) \thetaa.
\end{equation*}
The above implies that
\begin{equation*}
    \pdv{f(\x_{i,c};\thetaa)}{\thetaa}^\T = \x_{i,c} \otimes \I_C,
\end{equation*}
and also
\begin{equation*}
    \pdv[2]{\ell ( \z, \y_c)}{\z} \Bigg|_{\z_{i,c}}
    = \diag (\p_{i,c}) - \p_{i,c} \p_{i,c}^\T,
\end{equation*}
where $\p_{i,c} \in \R^C$ denotes the predicted probability vector of the $i$'th example belonging to the $c$'th class. The above identity was proven by \citet{bohning1992multinomial}. Plugging these into the definition of $\G$ in Equation \eqref{eq:Hess_G_E}, we obtain
\begin{equation*}
    \G = \Ave_{i,c} \left\{ (\x_{i,c} \otimes \I_C) (\diag (\p_{i,c}) - \p_{i,c} \p_{i,c}^\T) (\x_{i,c} \otimes \I_C)^\T \right\}.
\end{equation*}

\subsection{Spectrum of FIM under generative model setting}
\begin{mdframed}
\begin{defn} \textup{{\textbf{(Canonical classification model).}}}
    An instance of the \textit{canonical classification model} $\CCM(D,C)$ consists of the following components:
    \begin{itemize}
        \item $C$ canonical vector class means, $\e_c \in \R^D$ i.e.,
        \begin{equation*}
            \e_c =
            [0, \dots, 0, \underbrace{1}_{\text{index c}}, 0, \dots, 0].
        \end{equation*}
        \item $NC$ independently normally distributed $D$-dimensional sample deviations:
        \begin{equation*}
            \z_{i,c} \sim \mathcal{N} (0, \I),
            \quad 1 \leq c \leq C,
            \quad 1 \leq i \leq N.
        \end{equation*}
        \item A scalar representing the signal-to-noise ratio $t$.
    \end{itemize}
    The above give rise to a set $\{ \x_{i,c} \}_{i,c}$ obeying the class block structure, whose elements satisfy:
    \begin{equation*}
        \x_{i,c} = t \ \e_c + \z_{i,c}.
    \end{equation*}
\end{defn}
\end{mdframed}

\noindent We can further simplify the $\CCM$ model by imposing the following assumption.

\begin{mdframed}
\begin{defn} \textup{{\textbf{(Symmetric probabilities).}}}
    The probabilities predicted by the multinomial logistic regression classifier are \textit{symmetric} if the following holds:
    \begin{equation*}
        p_{i,c,c'} =
        \left\{\begin{array}{ll}
        1-\alpha, & \text{for } c=c' \\
        \frac{\alpha}{C-1}, & \text{for } c \neq c'
        \end{array}\right\},
    \end{equation*}
    where $p_{i,c,c'}$ is the probability of the $i$'th example in the $c$'th class belonging to cross-class $c'$. In other words, the probabilities are independent of the sample index $i$ and the cross-class $c'$, assuming $c' \neq c$. The vector of probabilities will be denoted by $\p_{\cdot,c}$.
\end{defn}
\end{mdframed}

\noindent The following lemma, whose proof is deferred to Appendix \ref{multinomial_logistic_regression}, proves that in the canonical classification model the expected FIM admits a particularly simple form.

\begin{mdframed}
\begin{restatable}[]{lem}{lemmamlm} \label{lemma:logistic_FIM}
\textup{\textbf{(Expected FIM of multinomial logistic regression trained on $\CCM$).}}
The expected FIM of multinomial logistic regression, with symmetric probabilities, trained on $\CCM(D,C)$, is given by:
    \begin{equation*}
        \Exp \G = \frac{s}{C} \cdot \blkdiag(\U_1, \dots, \U_C, \zero_{DC-C^2}) + \I \otimes \Ub,
    \end{equation*}
    where
    \begin{align*}
        s = & t^2 \\
        \U_c = & \diag (\p_{\cdot,c}) - \p_{\cdot,c} \p_{\cdot,c}^\T, \\
        \Ub = & \Ave_c \U_c,
    \end{align*}
    and the operator $\blkdiag(\cdot)$ forms a block diagonal matrix with $\U_c$ in its $c$'th block and zeros elsewhere.
\end{restatable}
\end{mdframed}
Notice that the $C \times C$ matrices $\U_c$ are in general of rank $C$ and the block diagonal matrix is of rank $C^2$. Hence, the expected FIM is a perturbation of a rank-$C^2$ matrix. The following theorem, proven in Appendix \ref{multinomial_logistic_regression}, describes the spectrum of this matrix.

\begin{mdframed}
\begin{restatable}[]{thm}{thmmlm} \label{thm:spec_V_GCC}
    \textup{\textbf{(Spectrum of expected FIM of multinomial logistic regression trained on $\CCM$).}}
    The spectrum of the expected FIM of multinomial logistic regression, with symmetric probabilities, trained on $\CCM(D,C)$, is given by:
    \begin{align*}
        & \lambda_i \left( \Exp \G \right) \\
        = & \left\{\begin{array}{ll}
        \frac{\alpha}{C-1} \left( s(1-\alpha) + \left( 2 - \alpha \frac{C}{C-1} \right) \right), & \text{for } 1 \leq i \leq C \\
        \frac{\alpha}{C-1} \left( \frac{s}{C} + \left( 2 - \alpha \frac{C}{C-1} \right) \right), & \text{for } C < i \leq C(C-1) \\
        \frac{\alpha}{C-1} \left( 2 - \alpha \frac{C}{C-1} \right), & \text{for } C(C-1) < i \leq D(C-1) \\
        0, & \text{for } D(C-1) < i \leq DC
        \end{array}\right\}. \nonumber
    \end{align*}
\end{restatable}
\end{mdframed}
Notice how the spectrum has $C$ outliers, a mini-bulk consisting of approximately $C^2$ eigenvalues, and a main bulk. Notice also how increasing the signal-to-noise ratio $s$ results in: (i) the separation of the top-$C$ outliers from the mini-bulk (because $1-\alpha$ would increase); (ii) the separation of the top-$C$ outliers from the bulk; and (iii) the separation of the mini-bulk from the main bulk. As such, the ratio of outliers to mini-bulk, the ratio of outliers to bulk, and the ratio of mini-bulk to bulk are all predictive of misclassification.

\section{Overview of literature in light of this paper}
In light of the insights gathered throughout this work, we can now view in a different light several papers in the literature.

\subsection{Start looking at layer-wise features and backpropagated errors, stop focusing on the Hessian}
In the last year, a large amount of research effort has been devoted into investigating the spectrum of the Hessian. This paper shows that the Hessian is the most complicated quantity one could possibly investigate since: (a) it inherits its cross-class structure from the product of features and backpropagated errors; (b) its cross-class structure, caused by the FIM, is perturbed by another matrix $\E$; and (c) it blends the information across the layers into a single quantity. The research community should therefore move from studying the Hessian into studying the layer-wise features and backpropagated errors. 

\citet{zhang2019all} raised similar concerns, stating that the analysis of the optimization landscape would be better performed through the study of individual layers, rather than a single holistic quantity, such as the Hessian. \citet{jiang2018predicting} further motivated the investigation of layer-wise  features by showing how generalization error can be predicted from margins of layer-wise features.

\subsection{Start looking at class/cross-class means and covariances, stop looking at eigenvalues}
The most surprising contribution of this paper is that it shows the inadequacy of eigenanalysis in revealing the fundamental structure in deep learning. It is true the spectrum reflects this structure, as seen through the plethora of measurements made in the literature, but the spectrum does not explain this structure, the class and cross-class means do.

Throughout this paper we consider the setting in which the number of classes $C$ is smaller than the dimension of features in a certain layer. However, if $C$ is bigger than the dimension of the features, then the spectrum would not manifest any bulk-and-outliers structure. However, class means, cross-class means and within-cross-class covariances would still be perfectly valid quantities to measure and study.

Even if the number of classes is small, if the class means are not gigantic, we may see that the spectrum has no outliers. However, the class and cross-class means are still present in the data, they are just not visible through eigenanalysis.

\subsection{The effect of batch normalization on the bulk-and-outliers structure}
\citet{ghorbani2019investigation} claim that ``batch normalization \citep{ioffe2015batch} pushes outliers back into the bulk'' and ``hypothesize a mechanistic explanation for why batch normalization speeds up optimization: it does so via suppression of outlier eigenvalues which slow down optimization''.

This paper explains their empirical observation. Specifically, the top-$C$ outliers in the spectrum of the Hessian are caused by a Khatri-Rao product between the feature global mean, $\h_G$, and backpropagated errors class means $\deltaa_c$. The whitening step of batch normalization subtracts the global mean from the features, which significantly decreases their class-mean magnitude. As a result, the outliers in the Hessian are pushed back into the main bulk. \citet{jacot2019freeze} provided a similar explanation for why batch normalization accelerates training by proving that the top outlier in the spectrum of the NTK matrix \citep{jacot2018neural} is suppressed as a result of batch normalization.

However, the fact that these outliers are pushed back by batch normalization does not mean they are completely removed, nor does it lessen their fundamental significance. The outliers are caused by class means, which separate from the bulk as function of depth. They represent the separation of class information from noise and are of utmost importance in studying the classification performance of deepnets.

\subsection{Flatness conjecture}
\citet{hochreiter1997flat} conjectured that the flatness of the loss function around the minima found by SGD is correlated with good generalization. Recent empirical work by \cite{keskar2016large,jastrzkebski2017three,jastrzebski2018relation,jiang2019fantastic,lewkowycz2020large} gave credence to this statement by showing, through empirical measurements, how sharpness can predict generalization. \citet{dinh2017sharp} questioned this conjecture by showing that most notions of flatness are problematic for deep models and can not be tied to generalization. Their argument relied on the observation that one can reparametrize a model and increase arbitrarily the sharpness of its minima, without changing the function it implements. However, the measures of flatness considered by \citet{dinh2017sharp} were the trace (sum of eigenvalues) or spectral norm (maximal eigenvalue) of the Hessian; they never considered the separation of the outliers from the bulk in the spectrum of the Hessian which, in light of our paper, should correlate with generalization.

\subsection{Predicting generalization through information matrices}
\citet{thomas2019information} propose to predict the generalization gap of deepnets based on the Takeuchi information criterion (TIC), which is equal to the trace of the covariance of gradients $\C$ divided by the FIM $\G$, i.e., $\Tr{\C \G^{-1}}$. They further propose to approximate it with an easier to compute quantity, $\Tr{\C} / \Tr{\G}$. Their measurements show correlation between these quantities and the generalization gap.

In light of our paper, we can understand better their measurements. The number of classes in their experiments is small, and thus so it the number of outliers. The traces are therefore dominated by the bulk eigenvalues. This, in turn, implies that their formulas neglects completely the separation of class means from noise. Instead, their prediction of the generalization gap relies on the sheer bulk magnitude. In the context of multinomial logistic regression, which we analyzed in Section \ref{sec:logistic}, we proved that the bulk eigenvalues indeed scale with misclassification. 

\subsection{Catastrophic forgetting}
Catastrophic forgetting refers to the phenomenon that deepnets, when trained on a series of tasks, tend to suffer from loss of accuracy on tasks which were learned first. A recent work
by \citet{ramasesh2020anatomy} found, through a rigorous empirical investigation, that catastrophic forgetting tends to occur primarily in deeper representations. Our work complements their findings by showing that deeper representations are more discriminative than shallower ones, and that catastrophic forgetting is likely the result of class means suppression.

\subsection{KFAC ineffient for large batch training}
Recently \citet{ma2020inefficiency} observed that KFAC is inefficient for large batch training of deepnets. While the current paper does not shed light on this phenomenon, it does propose CFAC as a promising alternative to KFAC.

\section{Conclusions}
There exist many fundamental objects associated with deepnets. No one has any intuition about their properties, structure or behaviour. Researchers therefore started extracting from them descriptive statistics like eigenvalues. It is remarkable such measurements are even possible, given how high-dimensional some of these objects are. After measuring the eigenvalues, one starts seeing a great deal of structure, which is curiously explicit, consisting of $C$ outliers, $C^2$ secondary outliers, and a main-bulk. Researchers pointed to the existence of some of the structure or its traces. However, it has been an open question as to what is causing this structure to appear and how to exploit it. This paper shows there is indeed a specific highly organized structure. Indeed, those initial observations can be expected to persist everywhere. Indeed, these are not artifacts but deeply significant observations. Indeed, understanding them is very important for understanding deepnet performance. This is not a spandrel, it is a clue of deep significance.


\newpage

\appendix

\section{Cross-class structure in G} \label{G_decomp}
Recall the definition of $\G$,
\begin{equation*}
    \G = \Ave_{i,c} \left\{ \pdv{f(\x_{i,c})}{\thetaa}^\T \pdv[2]{\ell (\x_{i,c},\y_{i,c})}{\f} \pdv{f(\x_{i,c})}{\thetaa} \right\}.
\end{equation*}
Plugging the Hessian of multinomial logistic regression \citet{bohning1992multinomial}, we obtain
\begin{equation*}
    \G = \Ave_{i,c} \left\{ \pdv{f(\x_{i,c})}{\thetaa}^\T \left( \text{diag} (\p_{i,c}) - \p_{i,c} \p_{i,c}^\T \right) \pdv{f(\x_{i,c})}{\thetaa} \right\}.
\end{equation*}
The above can be shown to be equivalent to
\begin{align*}
    \G = \Ave_{i,c} \left\{ \sum_{c'} p_{i,c,c'}
    \left( \pdv{f_{c'}(\x_{i,c})}{\thetaa} - \sum_{c'} p_{i,c,c'} \pdv{f_{c'}(\x_{i,c})}{\thetaa} \right)^\T \right. \\
    \left. \left( \pdv{f_{c'}(\x_{i,c})}{\thetaa} - \sum_{c'} p_{i,c,c'} \pdv{f_{c'}(\x_{i,c})}{\thetaa} \right) \right\}.
\end{align*}
Define the $p$-dimensional vector,
\begin{equation*}
    \g_{i,c,c'} = \pdv{\ell( f(\x_{i,c}; \thetaa), \y_{c'} )}{\thetaa},
\end{equation*}
and note that
\begin{align*}
    \g_{i,c,c'}
    & = \pdv{\ell( f(\x_{i,c}; \thetaa), \y_{c'} )}{\thetaa} \\
    & = \pdv{f(\x_{i,c}; \thetaa)}{\thetaa}^\T \pdv{\ell( f(\x_{i,c}; \thetaa), \y_{c'} )}{\f} \\
    & = \pdv{f(\x_{i,c})}{\thetaa}^\T (\y_{c'} - \p_{i,c}) \\
    & = \pdv{f_{c'}(\x_{i,c})}{\thetaa} - \sum_{c'} p_{i,c,c'} \pdv{f_{c'}(\x_{i,c})}{\thetaa}.
\end{align*}
Plugging the above expression into the definition of $\G$, we get
\begin{equation*}
    \G = \Ave_{i,c} \left\{ \sum_{c'} p_{i,c,c'} \g_{i,c,c'} \g_{i,c,c'}^\T \right\},
\end{equation*}
or equally
\begin{equation*}
    \G = \sum_{i,c,c'} w_{i,c,c'} \g_{i,c,c'} \g_{i,c,c'}^\T.
\end{equation*}

\subsection{First decomposition}
Recall the following definitions:
\begin{align*}
    \g_{c,c'} & = \sum_i \pi_{i,c,c'} \g_{i,c,c'} \\
    \GWccp & = \sum_i \pi_{i,c,c'} (\g_{i,c,c'} - \g_{c,c'}) (\g_{i,c,c'} - \g_{c,c'})^\T \\
    \pi_{i,c,c'} & = \frac{w_{i,c,c'}}{w_{c,c'}} \\
    w_{c,c'} & = \sum_i w_{i,c,c'}.
\end{align*}
Note that
\begin{align*}
    \G = & \sum_{i,c,c'} w_{i,c,c'} \g_{i,c,c'} \g_{i,c,c'}^\T \\
    = & \sum_{c,c'} w_{c,c'} \frac{1}{w_{c,c'}} \sum_i w_{i,c,c'} \g_{i,c,c'} \g_{i,c,c'}^\T \\
    = & \sum_{c,c'} w_{c,c'} \sum_i \pi_{i,c,c'} \g_{i,c,c'} \g_{i,c,c'}^\T \\
    = & \sum_{c,c'} w_{c,c'} (\g_{c,c'} \g_{c,c'}^\T + \GWccp) \\
    = & \sum_{c,c'} w_{c,c'} \g_{c,c'} \g_{c,c'}^\T + \sum_{c,c'} w_{c,c'} \GWccp \\
    = & \sum_{\substack{c,c'\\ c \neq c'}} w_{c,c'} \g_{c,c'} \g_{c,c'}^\T + \sum_c w_{c,c} \g_{c,c} \g_{c,c}^\T
    + \sum_{c,c'} w_{c,c'} \GWccp.
\end{align*}
In what follows, we further decompose only the first summation in the above equation.

\subsection{Second decomposition}
Recall the following definitions:
\begin{align*}
    \g_c & = \sum_{c' \neq c} \pi_{c,c'} \g_{c,c'} \\
    \GCPc & = \sum_{c' \neq c} \pi_{c,c'} (\g_{c,c'} - \g_{c'}) (\g_{c,c'} - \g_{c'})^\T \\
    \pi_{c,c'} & = \frac{w_{c,c'}}{w_c} \\
    w_c & = \sum_{c' \neq c} w_{c,c'}.
\end{align*}
We have
\begin{align*}
    \sum_{\substack{c,c'\\c \neq c'}} w_{c,c'} \g_{c,c'} \g_{c,c'}^\T
    = & \sum_c \sum_{c' \neq c} w_{c,c'} \g_{c,c'} \g_{c,c'}^\T \\
    = & \sum_c w_c \frac{1}{w_c} \sum_{c' \neq c} w_{c,c'} \g_{c,c'} \g_{c,c'}^\T \\
    = & \sum_c w_c \sum_{c' \neq c} \pi_{c,c'} \g_{c,c'} \g_{c,c'}^\T \\
    = & \sum_c w_c (\g_c \g_c^\T + \GCPc) \\
    = & \sum_c w_c \g_c \g_c^\T + \sum_c w_c \GCPc.
\end{align*}

\subsection{Combination}
Combining all the expressions from the previous subsections, we get
\begin{equation} \label{eq:G_decomp}
    \G
    = \underbrace{\sum_c w_c \g_c \g_c^\T}_{\GC}
    + \underbrace{\sum_c w_c \GCPc}_{\GCP}
    + \underbrace{\sum_{c,c'} w_{c,c'} \GWccp}_{\GW}
    + \underbrace{\sum_c w_{c,c} \g_{c,c} \g_{c,c}^\T}_{\G_{c=c'}}.
\end{equation}

\section{Multinomial logistic regression} \label{multinomial_logistic_regression}

\begin{mdframed}
\lemmamlm*
\end{mdframed}
\begin{proof}
Recall the definition of $\G$:
\begin{align*}
    \G = \Ave_{i,c} \left\{ (\x_{i,c} \otimes \I_C) (\diag (\p_{i,c}) - \p_{i,c} \p_{i,c}^\T) (\x_{i,c} \otimes \I_C)^\T \right\}.
\end{align*}
Using our symmetric probabilities assumption, there exist $C \times C$ matrices $\{ \U_c \}_c$ for which
\begin{equation*}
    \G = \Ave_{i,c} \left\{ (\x_{i,c} \otimes \I_C) \U_c (\x_{i,c} \otimes \I_C)^\T \right\}.
\end{equation*}
Using the property of the Kronecker product, $(\A \otimes \B)(\C \otimes \D) = (\A\C) \otimes (\B\D)$, the expression in the above average can be simplified into:
\begin{align*}
    & \left( \x_{i,c} \otimes \I_C \right) \U_c \left( \x_{i,c} \otimes \I_C \right)^\T \\
    = & \left( \x_{i,c} \otimes \I_C \right) (1 \otimes \U_c) \left( \x_{i,c} \otimes \I_C \right)^\T \\
    = & \left( \x_{i,c} \otimes \U_c \right) \left( \x_{i,c} \otimes \I_C \right)^\T \\
    = & (\x_{i,c} \x_{i,c}^\T) \otimes \U_c.
\end{align*}
Plugging this expression into the previous equation, we get
\begin{equation*}
    \G = \Ave_{i,c} \left\{ (\x_{i,c} \x_{i,c}^\T) \otimes \U_c \right\}.
\end{equation*}
Plugging our assumption that $\x_{i,c} = t \e_c + \z_{i,c}$, we obtain
\begin{align*}
    \G
    = & t^2 \Ave_{i,c} \left\{ (\e_c \e_c^\T) \otimes \U_c \right\} \\
    + & t \Ave_{i,c} \left\{ (\e_c \z_{i,c}^\T) \otimes \U_c \right\} \\
    + & t \Ave_{i,c} \left\{ (\z_{i,c} \e_c^\T) \otimes \U_c \right\} \\
    + & \Ave_{i,c} \left\{ (\z_{i,c} \z_{i,c}^\T) \otimes \U_c \right\}.
\end{align*}
Taking the expectation of both sides over $\z_{i,c}$, we get
\begin{equation*}
    \Exp \G
    = t^2 \Ave_c \left\{ (\e_c \e_c^\T) \otimes \U_c \right\}
    + \Ave_c \left\{ \I \otimes \U_c \right\}.
\end{equation*}
The above can be written concisely as follows:
\begin{equation*}
    \Exp \G = \frac{s}{C} \blkdiag(\U_1, \dots, \U_C, \zero_{DC-C^2}) + \I \otimes \Ave_c \U_c.
\end{equation*}
\end{proof}

\begin{lem}
Consider the $C \times C$ matrix,
\begin{equation*}
    \A =
    \left[\begin{array}{ccccc}
    a & b & \dots & b \\
    b & a & \dots & b \\
    \vdots & \vdots & \ddots & \vdots \\
    b & b & \dots & a
    \end{array}\right].
\end{equation*}
Its eigenvalues are given by:
\begin{equation*}
    \lambda(\A) = 
    \left\{\begin{array}{ll}
    a+b(C-1), & \text{for } i=1 \\
    a-b, & \text{for } 1 < i \leq C
    \end{array}\right\}.
\end{equation*}
\end{lem}
\begin{proof}
Notice that
\begin{equation*}
    \A \frac{1}{\sqrt{C}} \one = (a + b(C-1)) \frac{1}{\sqrt{C}} \one.
\end{equation*}
As such, $\frac{1}{\sqrt{C}} \one$ is an eigenvector with an eigenvalue $a+b(C-1)$. Recall that the trace of a matrix is equal to the sum of its eigenvalues. Notice that the rest of the eigenvalues are all the same. As such, their value is equal to 
\begin{align*}
    \frac{\Tr{\A} - (a + b(C-1))}{C-1}
    = \frac{aC - (a + b(C-1))}{C-1} = a - b.
\end{align*}
\end{proof}

\begin{lem}
Consider the $C \times C$ matrix,
\begin{equation*}
    \B = \left[\begin{array}{c|cccc}
    a & b & b & \dots & b \\ \hline
    b & d & c & \dots & c \\
    b & c & d & & \vdots \\
    \vdots & \vdots & & \ddots & \vdots \\
    b & c & \dots & \dots & d
    \end{array}\right].
\end{equation*}
Its eigenvalues are given by:
\begin{equation*}
    \lambda \left( \B \right) = 
    \left\{\begin{array}{ll}
    (T + \Delta)/2, & \text{for } i=1 \\
    (T - \Delta)/2, & \text{for } i=2 \\
    d-e, & \text{for } 2 < i \leq C
    \end{array}\right\}.
\end{equation*}
where
\begin{align*}
    T = & a + d+e(C-2) \\
    D = & a \left( d + e \left(C-2\right) \right) - b^2(C-1) \\
    \Delta = & \sqrt{T^2 - 4D}.
\end{align*}
\end{lem}
\begin{proof}
Consider the $(C-1) \times (C-1)$ matrix,
\begin{equation*}
    \A =
    \left[\begin{array}{ccccc}
    d & c & \dots & c \\
    c & d & \dots & c \\
    \vdots & \vdots & \ddots & \vdots \\
    c & c & \dots & d
    \end{array}\right].
\end{equation*}
Its eigenvalues are given by:
\begin{equation*}
    \lambda(\A) = 
    \left\{\begin{array}{ll}
    d+c(C-1), & \text{for } i=1 \\
    d-c, & \text{for } 1 < i \leq C-1
    \end{array}\right\}.
\end{equation*}
Denote by $\v_3,\dots,\v_C$ the $C-2$ eigenvectors corresponding to the eigenvalue $d-c$. Define the orthonormal basis:
\begin{equation*}
    \V = \left[ \e_1, \frac{1}{\sqrt{C}}\onet, \v_3, \dots, \v_C \right],
\end{equation*}
where $\onet$ is a $C$-dimensional vector of ones, except in the first coordinate, where it is equal to zero. In other words,
\begin{equation*}
    \onet = \left[ 0, 1, \dots, 1 \right].
\end{equation*}
Notice that
\begin{align*}
    \V^T \B \V =
    \left[\begin{array}{cc|cccc}
    a & b \sqrt{C-1} & 0 & 0 & \dots & 0 \\
    b \sqrt{C-1} & d+c(C-2) & 0 & 0 & \dots & 0 \\ \hline
    0 & 0 & d-c & 0 & \dots & 0 \\
    0 & 0 & 0 & d-c & & 0 \\
    \vdots & & \vdots & & \ddots & \vdots \\
    0 & 0 & 0 & 0 & \dots & d-c
    \end{array}\right].
\end{align*}
The diagonal values in the above matrix are obtained through the following equations
\begin{align*}
    \e_1^\T \B \e_1 = & a \\
    \frac{1}{\sqrt{C}} \onet^\T \B \frac{1}{\sqrt{C}} \onet_1 = & d + c(C-2) \\
    \v_i \B \v_i = & d - c, \quad i \geq 3.
\end{align*}
The off-diagonal values are obtained by noticing that
\begin{align*}
    \e_1^\T \B \frac{1}{\sqrt{C}} \onet = & b \sqrt{C-1} \\
    \e_1^\T \B \v_i = & 0, \quad i \geq 3 \\
    \v_i^\T \B \v_j = & 0, \quad i,j \geq 3.
\end{align*}
The above implies that $d-c$ is an eigenvalue with multiplicity $C-2$. The other two eigenvalues are equal to the eigenvalues of the matrix
\begin{equation*}
    \C = \left[\begin{array}{cc}
    a & b \sqrt{C-1} \\
    b \sqrt{C-1} & d+e(C-2)
    \end{array}\right].
\end{equation*}
Recall that the trace of the matrix is equal to the sum of its eigenvalues and the determinant to the multiplication of the eigenvalues. As such,
\begin{align*}
    a + d + e(C-2) = \Tr{\C} = \lambda_1 + \lambda_2 \\
    a(d + e(C-2)) - b^2(C-1) = |\C| = \lambda_1 \lambda_2.
\end{align*}
Notice that
\begin{equation*}
    \lambda_2 = \frac{|\C|}{\lambda_1}.
\end{equation*}
Plugging the above into the equation of the trace, we get
\begin{equation*}
    \Tr{\C} = \lambda_1 + \frac{|\C|}{\lambda_1}.
\end{equation*}
Multiplying both sides by $\lambda_1$ and rearranging the terms, we obtain
\begin{equation*}
    0 = \lambda_1^2 -\Tr{\C} \lambda_1 + |\C|,
\end{equation*}
the solution of which is
\begin{equation*}
    \lambda_1 = \frac{\Tr{C} \pm \sqrt{\Tr{\C}^2 - 4|\C|}}{2}.
\end{equation*}
Defining
\begin{equation*}
    \Delta = \sqrt{\Tr{\C}^2 - 4|\C|},
\end{equation*}
we obtain
\begin{equation*}
    \lambda_1 = \frac{\Tr{C} \pm \Delta}{2}.
\end{equation*}
\end{proof}

\begin{mdframed}
\thmmlm*
\end{mdframed}
\begin{proof}
Consider without loss of generality the matrix $\U_1$.
Using the symmetric probabilities assumption, the $(c,c')$ entry of this matrix is given by:
\begin{align*}
    \U_1(c',c'')
    = \Ave_i \left\{ \delta_{\{c'=c''\}} p_{i,1,c'} - p_{i,1,c'} p_{i,1,c''} \right\}.
\end{align*}
As such, the whole matrix equals
\begin{align*}
    \U_1 = &
    \left[\begin{array}{c|cccc}
    1-\alpha & 0 & 0 & \dots & 0 \\ \hline
    0 & \frac{\alpha}{C-1} & 0 & \dots & 0 \\
    0 & 0 & \frac{\alpha}{C-1} & & \vdots \\
    \vdots & \vdots & & \ddots & \vdots \\
    0 & 0 & \dots & \dots & \frac{\alpha}{C-1}
    \end{array}\right] \\
    - &
    \left[\begin{array}{c|cccc}
    (1-\alpha)^2 & \frac{\alpha(1-\alpha)}{C-1} & \frac{\alpha(1-\alpha)}{C-1} & \dots & \frac{\alpha(1-\alpha)}{C-1} \\ \hline
    \frac{\alpha(1-\alpha)}{C-1} & (\frac{\alpha}{C-1})^2 & (\frac{\alpha}{C-1})^2 & \dots & (\frac{\alpha}{C-1})^2 \\
    \frac{\alpha(1-\alpha)}{C-1} & (\frac{\alpha}{C-1})^2 & (\frac{\alpha}{C-1})^2 & & \vdots \\
    \vdots & \vdots & & \ddots & \vdots \\
    \frac{\alpha(1-\alpha)}{C-1} & (\frac{\alpha}{C-1})^2 & \dots & \dots & (\frac{\alpha}{C-1})^2
    \end{array}\right].
\end{align*}
Simplifying the expressions, we get
\begin{equation*}
    \U_1 =
    \left[\begin{array}{c|cccc}
    \alpha(1-\alpha) & -\frac{\alpha(1-\alpha)}{C-1} & -\frac{\alpha(1-\alpha)}{C-1} & \dots & -\frac{\alpha(1-\alpha)}{C-1} \\ \hline
    -\frac{\alpha(1-\alpha)}{C-1} & \frac{\alpha}{C-1}\left( 1- \frac{\alpha}{C-1} \right) & -(\frac{\alpha}{C-1})^2 & \dots & -(\frac{\alpha}{C-1})^2 \\
    -\frac{\alpha(1-\alpha)}{C-1} & -(\frac{\alpha}{C-1})^2 & \frac{\alpha}{C-1}\left( 1- \frac{\alpha}{C-1} \right) & & \vdots \\
    \vdots & \vdots & & \ddots & \vdots \\
    -\frac{\alpha(1-\alpha)}{C-1} & -(\frac{\alpha}{C-1})^2 & \dots & \dots & \frac{\alpha}{C-1}\left( 1- \frac{\alpha}{C-1} \right)
    \end{array}\right].
\end{equation*}
Notice that
\begin{equation*}
    \Ub =
    \left[\begin{array}{ccccc}
    a & b & \dots & b \\
    b & a & \dots & b \\
    \vdots & \vdots & \ddots & \vdots \\
    b & b & \dots & a
    \end{array}\right],
\end{equation*}
where
\begin{align*}
    a
    = & \frac{1}{C} \alpha(1-\alpha) + \frac{C-1}{C} \frac{\alpha}{C-1}\left( 1- \frac{\alpha}{C-1} \right) \\
    = & \frac{\alpha}{C} \left( 2 - \alpha \frac{C}{C-1} \right),
\end{align*}
and
\begin{align*}
    b
    = & \frac{2}{C} \frac{-\alpha(1-\alpha)}{C-1}  + \frac{C-2}{C} \frac{-\alpha^2}{(C-1)^2} \\
    = & - \frac{\alpha}{C(C-1)} \left( 2 - \alpha \frac{C}{C-1} \right).
\end{align*}
Using the previous lemma, the eigenvalues of $\Ub$ are given by
\begin{equation*}
    \lambda(\Ub)
    = \left\{\begin{array}{ll}
    a+b(C-1), & \text{for } i=1 \\
    a-b, & \text{for } 1 < i \leq C
    \end{array}\right\},
\end{equation*}
or equally,
\begin{equation*}
    \lambda(\Ub)
    = \left\{\begin{array}{ll}
    0, & \text{for } i=1 \\
    \frac{\alpha}{C-1} \left( 2 - \alpha \frac{C}{C-1} \right), & \text{for } 1 < i \leq C
    \end{array}\right\}.
\end{equation*}
Recalling that
\begin{equation*}
    \Exp \G = \frac{s}{C} \blkdiag(\U_1, \dots, \U_C, \zero_{(D-C)C}) + \I \otimes \Ub,
\end{equation*}
we get that $\Exp \G$ has $(D-C)(C-1)$ eigenvalues equal to
\begin{equation} \label{eq:remaining_blkdiag}
    \frac{\alpha}{C-1} \left( 2 - \alpha \frac{C}{C-1} \right),
\end{equation}
and also $D-C$ eigenvalues equal to zero. Next, notice that
\begin{equation*}
    \frac{s}{C} \U_1 + \Ub \\
    = \left[\begin{array}{c|cccc}
    a & b & b & \dots & b \\ \hline
    b & d & e & \dots & e \\
    b & e & d & & \vdots \\
    \vdots & \vdots & & \ddots & \vdots \\
    b & e & \dots & \dots & d
    \end{array}\right],
\end{equation*}
where
\begin{align*}
    a = & \frac{s}{C} \alpha(1-\alpha) + \frac{\alpha}{C} \left( 2 - \alpha \frac{C}{C-1} \right) \\
    b = & - \frac{s}{C} \frac{\alpha(1-\alpha)}{C-1} - \frac{\alpha}{C(C-1)} \left( 2 - \alpha \frac{C}{C-1} \right) \\
    e = & - \frac{s}{C} \left( \frac{\alpha}{C-1} \right)^2 - \frac{\alpha}{C(C-1)} \left( 2 - \alpha \frac{C}{C-1} \right) \\
    d = & \frac{s}{C} \frac{\alpha(C - 1 - \alpha)}{(C-1)^2} + \frac{\alpha}{C} \left( 2 - \alpha \frac{C}{C-1} \right),
\end{align*}
or equally,
\begin{align*}
    a = & \frac{\alpha}{C} \left(s(1-\alpha) + \left( 2 - \alpha \frac{C}{C-1} \right)\right) \\
    b = & - \frac{\alpha}{C(C-1)} \left(s(1-\alpha) + \left(2 - \alpha \frac{C}{C-1} \right)\right) \\
    e = & - \frac{\alpha}{C(C-1)} \left( s \frac{\alpha}{C-1} + \left( 2 - \alpha \frac{C}{C-1} \right)\right) \\
    d = & \frac{\alpha}{C} \left( s \frac{C - 1 - \alpha}{(C-1)^2} + \left( 2 - \alpha \frac{C}{C-1} \right) \right).
\end{align*}
Using the previous lemma, its eigenvalues are given by:
\begin{equation} \label{eq:eig_summed}
    \lambda \left( \frac{s}{C} \U_1 + \Ub \right) = 
    \left\{\begin{array}{ll}
    (T + \Delta)/2, & \text{for } i=1 \\
    (T - \Delta)/2, & \text{for } i=2 \\
    d-e, & \text{for } 2 < i \leq C
    \end{array}\right\}.
\end{equation}
where
\begin{align*}
    T = & a + d+e(C-2) \\
    D = & a \left( d + e \left(C-2\right) \right) - b^2(C-1) \\
    \Delta = & \sqrt{T^2 - 4D}.
\end{align*}
Plugging the values of $d$ and $e$ and simplifying the expressions, we get
\begin{align} \label{eq:d_minus_e}
    d - e = &
    \frac{\alpha}{C} \left( s \frac{C - 1 - \alpha}{(C-1)^2} + \left( 2 - \alpha \frac{C}{C-1} \right) \right) \\
    + & \frac{\alpha}{C(C-1)} \left( s \frac{\alpha}{C-1} + \left( 2 - \alpha \frac{C}{C-1} \right)\right) \\
    = & \frac{\alpha}{C-1} \left( \frac{s}{C} + \left( 2 - \alpha \frac{C}{C-1} \right) \right).
\end{align}
Plugging the values into the trace, we obtain
\begin{align*}
    T = & \frac{\alpha}{C} \left(s(1-\alpha) + \left( 2 - \alpha \frac{C}{C-1} \right)\right) \\
    + & \frac{\alpha}{C} \left( s \frac{C - 1 - \alpha}{(C-1)^2} + \left( 2 - \alpha \frac{C}{C-1} \right) \right) \\
    - & \frac{\alpha}{C(C-1)} \left( s \frac{\alpha}{C-1} + \left( 2 - \alpha \frac{C}{C-1} \right)\right)(C-2) \\
    = & \frac{s \alpha(1-\alpha)}{C-1} + \frac{\alpha}{C-1} \left( 2 - \alpha \frac{C}{C-1} \right) \\
    = & \frac{\alpha}{C-1} \left( s (1-\alpha) + \left( 2 - \alpha \frac{C}{C-1} \right) \right).
\end{align*}
Notice that
\begin{align*}
    & d + e \left(C-2\right) \\
    = & \frac{\alpha}{C} \left( s \frac{C - 1 - \alpha}{(C-1)^2} + \left( 2 - \alpha \frac{C}{C-1} \right) \right) \\
    - & \frac{\alpha}{C(C-1)} \left( s \frac{\alpha}{C-1} + \left( 2 - \alpha \frac{C}{C-1} \right)\right) (C-2) \\
    = & \frac{\alpha}{C(C-1)} \left( s (1-\alpha) + \left( 2 - \alpha \frac{C}{C-1} \right) \right).
\end{align*}
As such,
\begin{align*}
    & a \left( d + e \left(C-2\right) \right) \\
    = & \frac{\alpha}{C} \left( s (1-\alpha) + \left( 2 - \alpha \frac{C}{C-1} \right)\right) \\
    \times & \frac{\alpha}{C(C-1)} \left( s (1-\alpha) + \left( 2 - \alpha \frac{C}{C-1} \right) \right) \\
    = & \frac{\alpha^2}{C^2 (C-1)} \left( s (1-\alpha) + \left( 2 - \alpha \frac{C}{C-1} \right) \right)^2.
\end{align*}
Notice also that
\begin{align*}
    & b^2(C-1) \\
    = & \left( \frac{\alpha}{C(C-1)} \left(s (1-\alpha) + \left(2 - \alpha \frac{C}{C-1} \right)\right) \right)^2(C-1) \\
    = & \frac{\alpha^2}{C^2(C-1)} \left(s (1-\alpha) + \left(2 - \alpha \frac{C}{C-1} \right)\right)^2.
\end{align*}
Combining the two previous expressions, we get the determinant is equal to zero,
\begin{align*}
    D = & a \left( d + e \left(C-2\right) \right) - b^2(C-1) \\
    = & \frac{\alpha^2}{C^2 (C-1)} \left(s (1-\alpha) + \left(2 - \alpha \frac{C}{C-1} \right)\right)^2 \\
    - & \frac{\alpha^2}{C^2(C-1)} \left(s (1-\alpha) + \left(2 - \alpha \frac{C}{C-1} \right)\right)^2 \\
    = & 0.
\end{align*}
Hence,
\begin{align*}
    \Delta = & \sqrt{T^2 - 4D} = T.
\end{align*}
One of the eigenvalues of $\frac{s}{C} \U_1 + \Ub$ is zero, since
\begin{equation} \label{eq:second_eigval}
    \frac{T - \Delta}{2} = \frac{T - T}{2} = 0,
\end{equation}
while the other is equal to
\begin{align} \label{eq:first_eigval}
    \frac{T + \Delta}{2} = \frac{T + T}{2} = T = \frac{\alpha}{C-1} \left( s (1-\alpha) + \left( 2 - \alpha \frac{C}{C-1} \right) \right).
\end{align}
Plugging Equations \eqref{eq:d_minus_e}, \eqref{eq:first_eigval} and \eqref{eq:second_eigval} into Equation \eqref{eq:eig_summed}, we get
\begin{equation*}
    \lambda \left( \frac{s}{C} \U_1 + \Ub \right) = 
    \left\{\begin{array}{ll}
    \frac{\alpha}{C-1} \left( s (1-\alpha) + \left( 2 - \alpha \frac{C}{C-1} \right) \right), & \text{for } i=1 \\
    \frac{\alpha}{C-1} \left( \frac{s}{C} + \left( 2 - \alpha \frac{C}{C-1} \right) \right), & \text{for } 1 < i < C \\
    0, & \text{for } i=C
    \end{array}\right\}.
\end{equation*}
Recall that the spectrum of $\frac{s}{C} \U_c + \Ub$ does not depend on $c$. As such, the spectrum of $\frac{s}{C} \blkdiag(\U_1, \dots, \U_C) + \I \otimes \Ub$ is obtained by simply multiplying the multiplicity of each eigenvalue in $\frac{s}{C} \U_1 + \Ub$ by $C$, i.e.,
\begin{align*}
    & \lambda \left( \frac{s}{C} \blkdiag(\U_1, \dots, \U_C) + \I \otimes \Ub \right) \\
    = & \left\{\begin{array}{ll}
    \frac{\alpha}{C-1} \left( s (1-\alpha) + \left( 2 - \alpha \frac{C}{C-1} \right) \right), & \text{for } 1 \leq i \leq C \\
    \frac{\alpha}{C-1} \left( \frac{s}{C} + \left( 2 - \alpha \frac{C}{C-1} \right) \right), & \text{for } C < i \leq C(C-1) \\
    0, & \text{for } C(C-1) < i \leq C^2
    \end{array}\right\}.
\end{align*}
Combining the above with Equation \eqref{eq:remaining_blkdiag}, we conclude
\begin{align*}
    & \lambda_i \left( \Exp \G \right)
    = \lambda_i \left( \frac{s}{C} \blkdiag(\U_1, \dots, \U_C, \zero_{(D-C)C}) + \I \otimes \Ub \right) \\
    = & \left\{\begin{array}{ll}
    \frac{\alpha}{C-1} \left( s(1-\alpha) + \left( 2 - \alpha \frac{C}{C-1} \right) \right), & \text{for } 1 \leq i \leq C \\
    \frac{\alpha}{C-1} \left( \frac{s}{C} + \left( 2 - \alpha \frac{C}{C-1} \right) \right), & \text{for } C < i \leq C(C-1) \\
    \frac{\alpha}{C-1} \left( 2 - \alpha \frac{C}{C-1} \right), & \text{for } C(C-1) < i \leq D(C-1) \\
    0, & \text{for } D(C-1) < i \leq DC
    \end{array}\right\}.
\end{align*}
\end{proof}

\section{Tools from numerical linear algebra}\label{numerical_linear_algebra}
Our approach for approximating the spectrum of deepnet Hessians (and similarly other quantities) builds on the survey of \citet{lin2016approximating}, which discussed several different methods for approximating the density of the spectrum of large linear operators; many of which were first developed by physicists and chemists in quantum mechanics starting from the 1970's \citep{ducastelle1970moments,wheeler1972modified,turek1988maximum,drabold1993maximum}. From the methods presented therein, we implemented and tested two: the Lanczos method and KPM. From our experience Lanczos was effective and useful, while KPM was temperamental and problematic. As such, we focus in this work on Lanczos.

\begin{figure}[h]
    \centering
    \begin{subfigure}[t]{0.25\textwidth}
        \centering
        \includegraphics[width=1\textwidth]{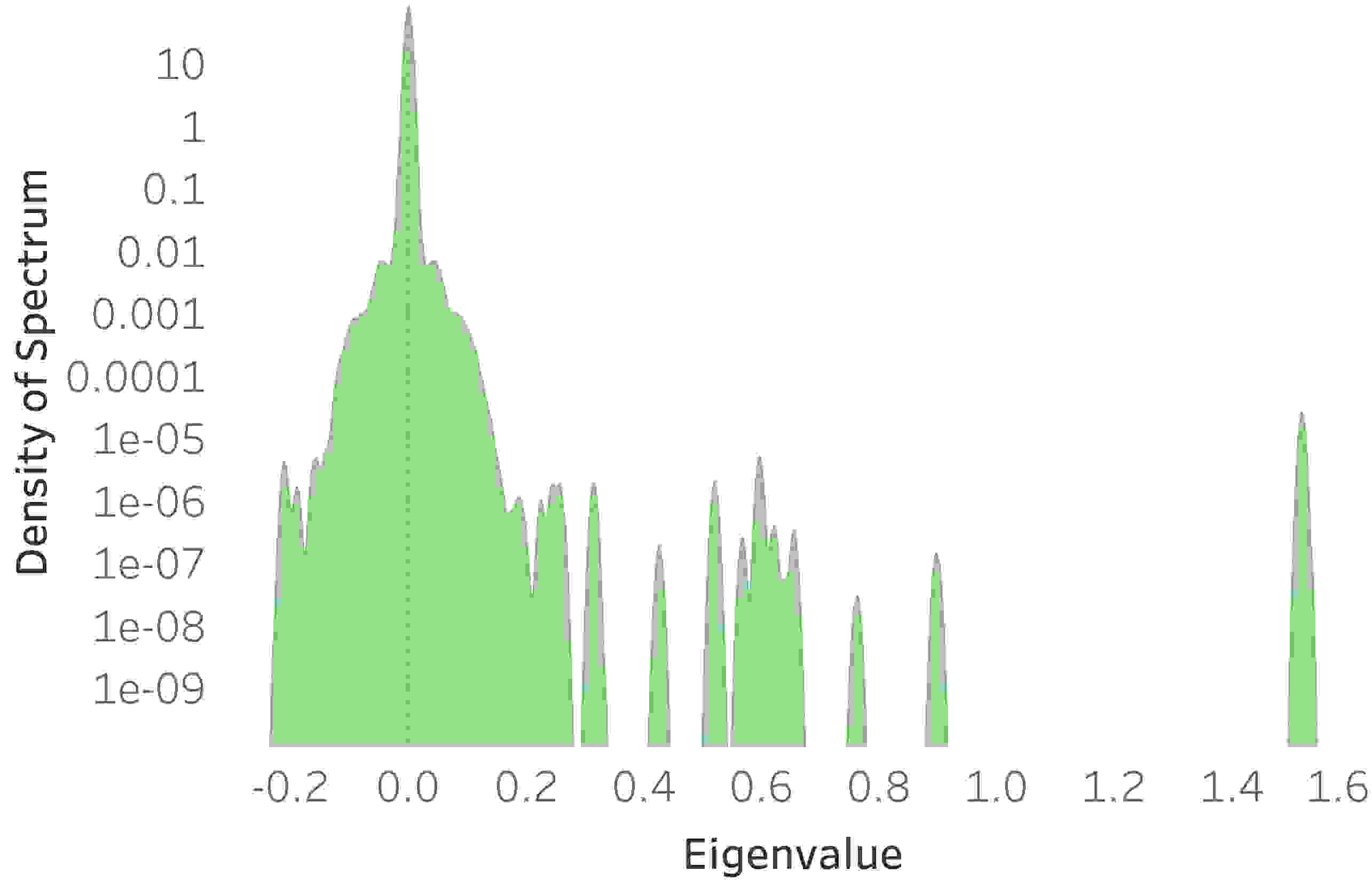}
        \caption{MNIST, train}
    \end{subfigure}
    \begin{subfigure}[t]{0.25\textwidth}
        \centering
        \includegraphics[width=1\textwidth]{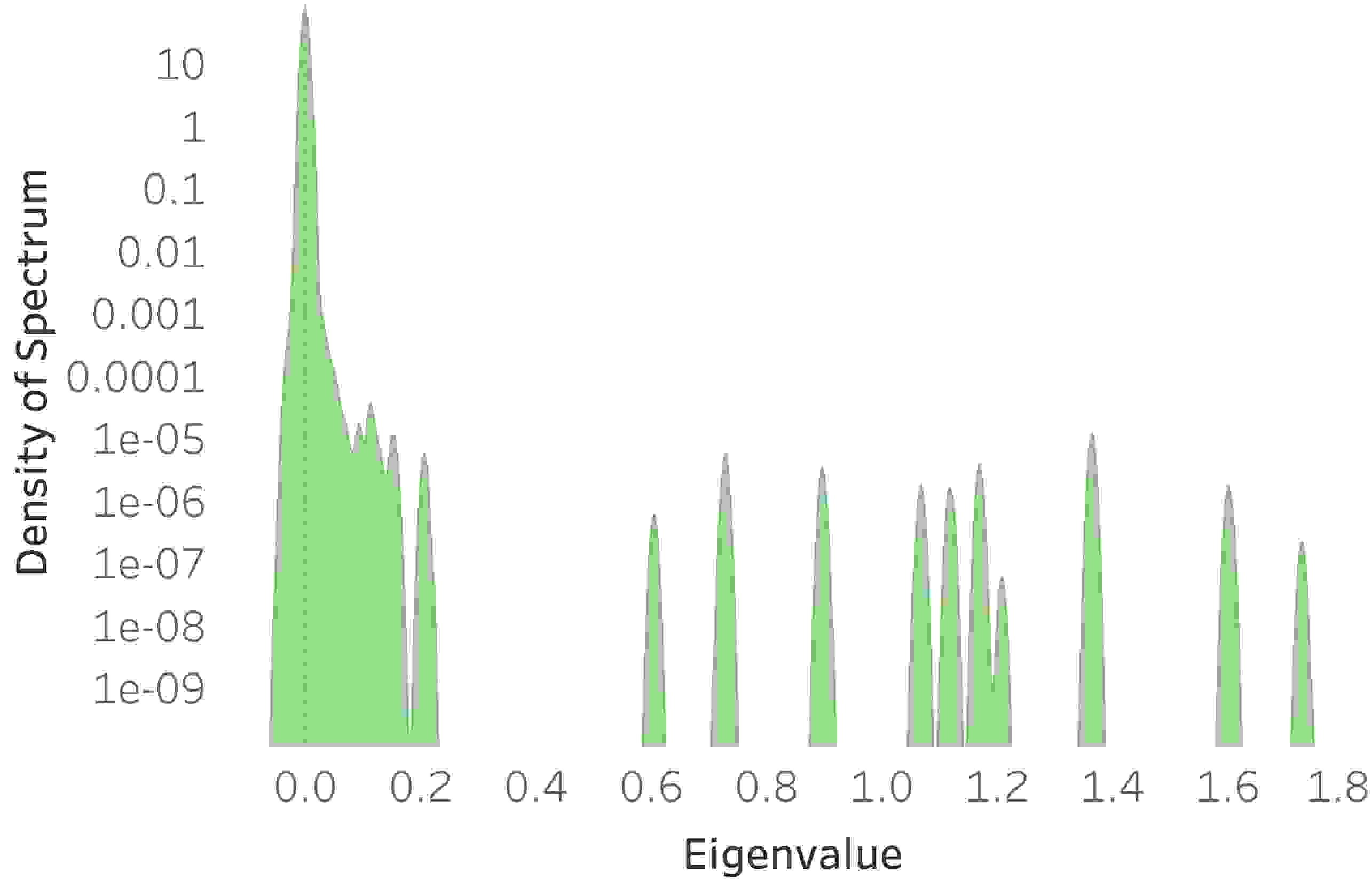}
        \caption{Fashion, train}
    \end{subfigure}
    \begin{subfigure}[t]{0.25\textwidth}
        \centering
        \includegraphics[width=1\textwidth]{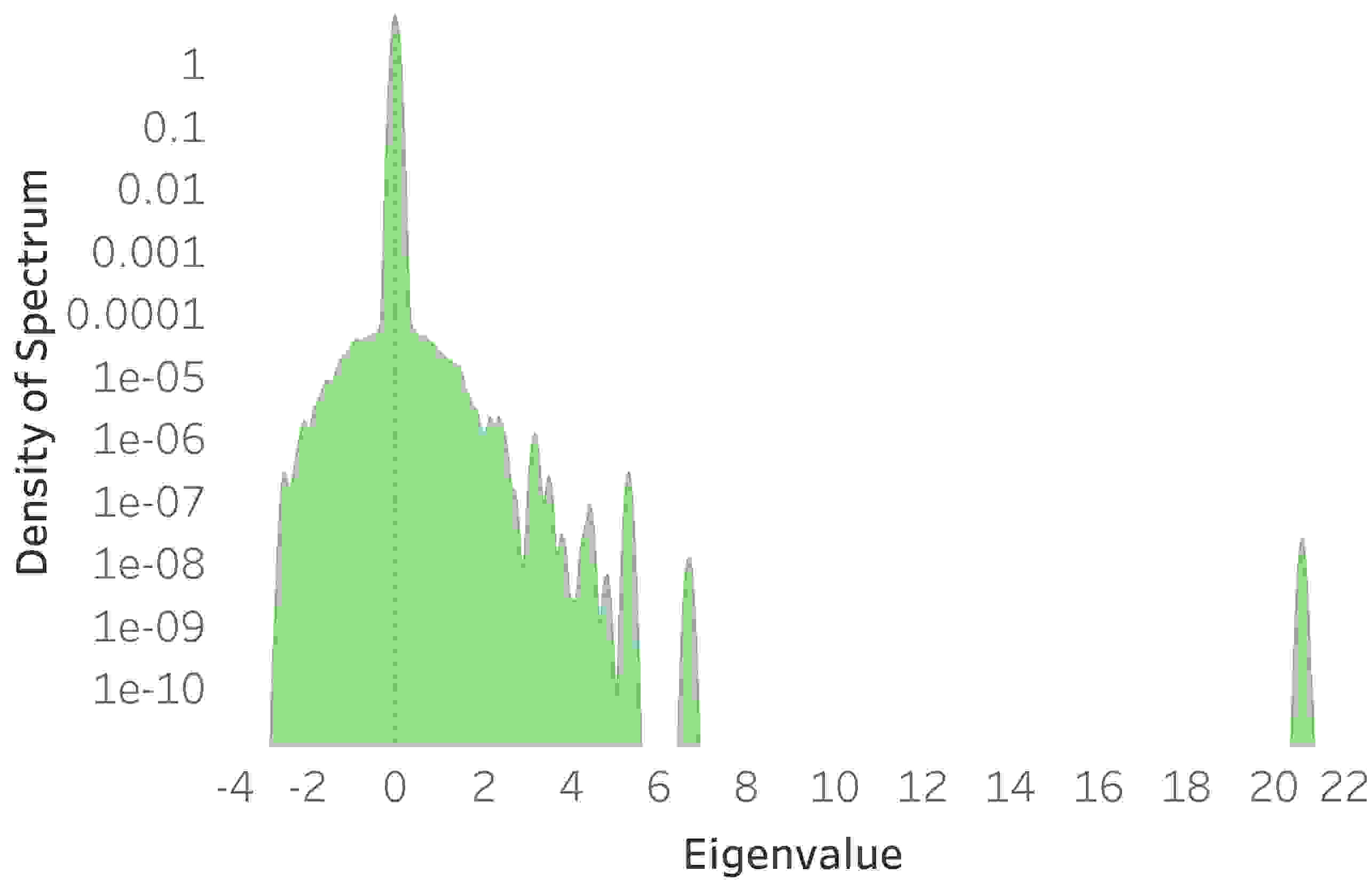}
        \caption{CIFAR10, train}
    \end{subfigure}
    \begin{subfigure}[t]{0.25\textwidth}
        \centering
        \includegraphics[width=1\textwidth]{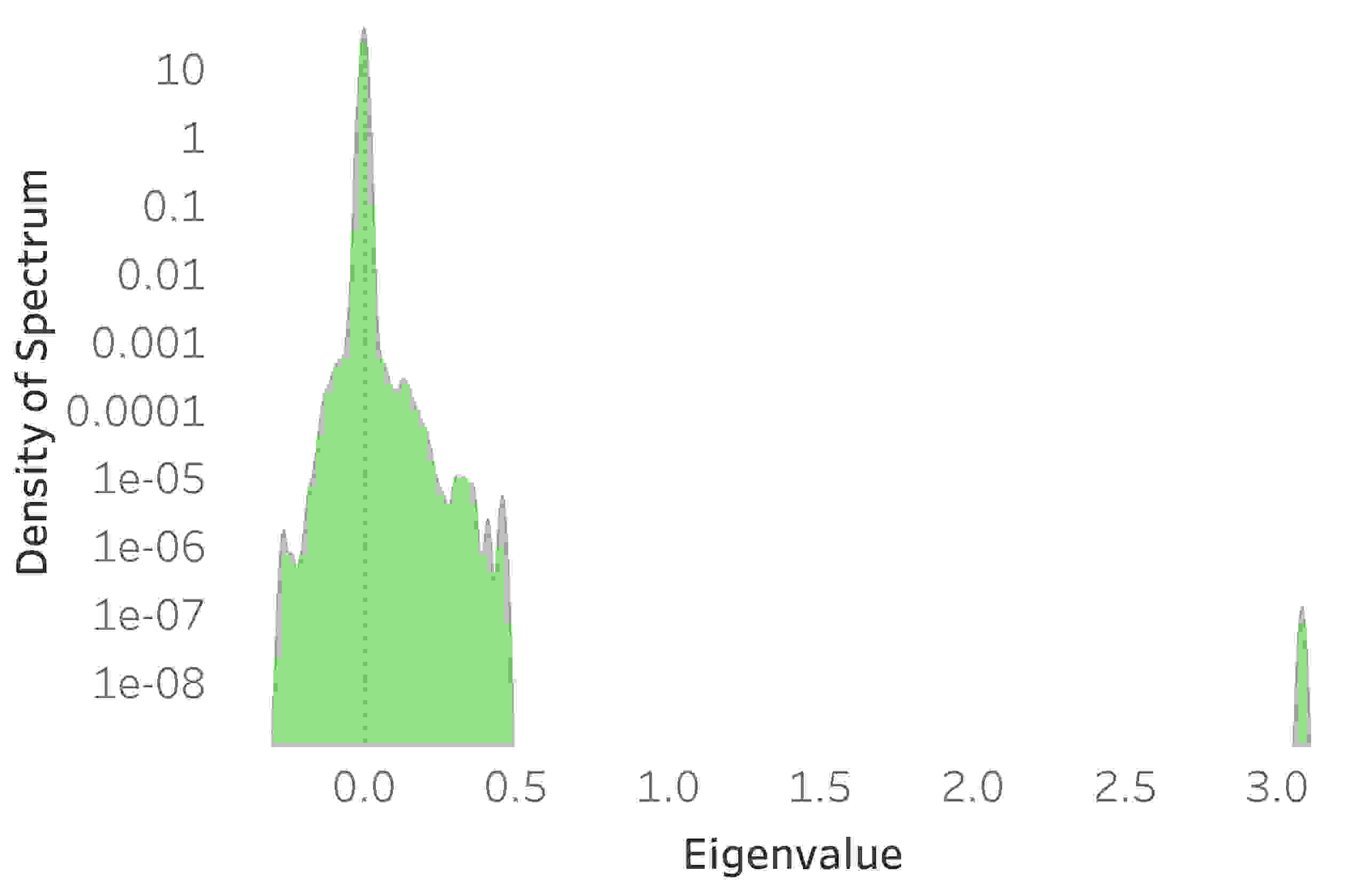}
        \caption{CIFAR100, train}
    \end{subfigure}
    \begin{subfigure}[t]{0.25\textwidth}
        \centering
        \includegraphics[width=1\textwidth]{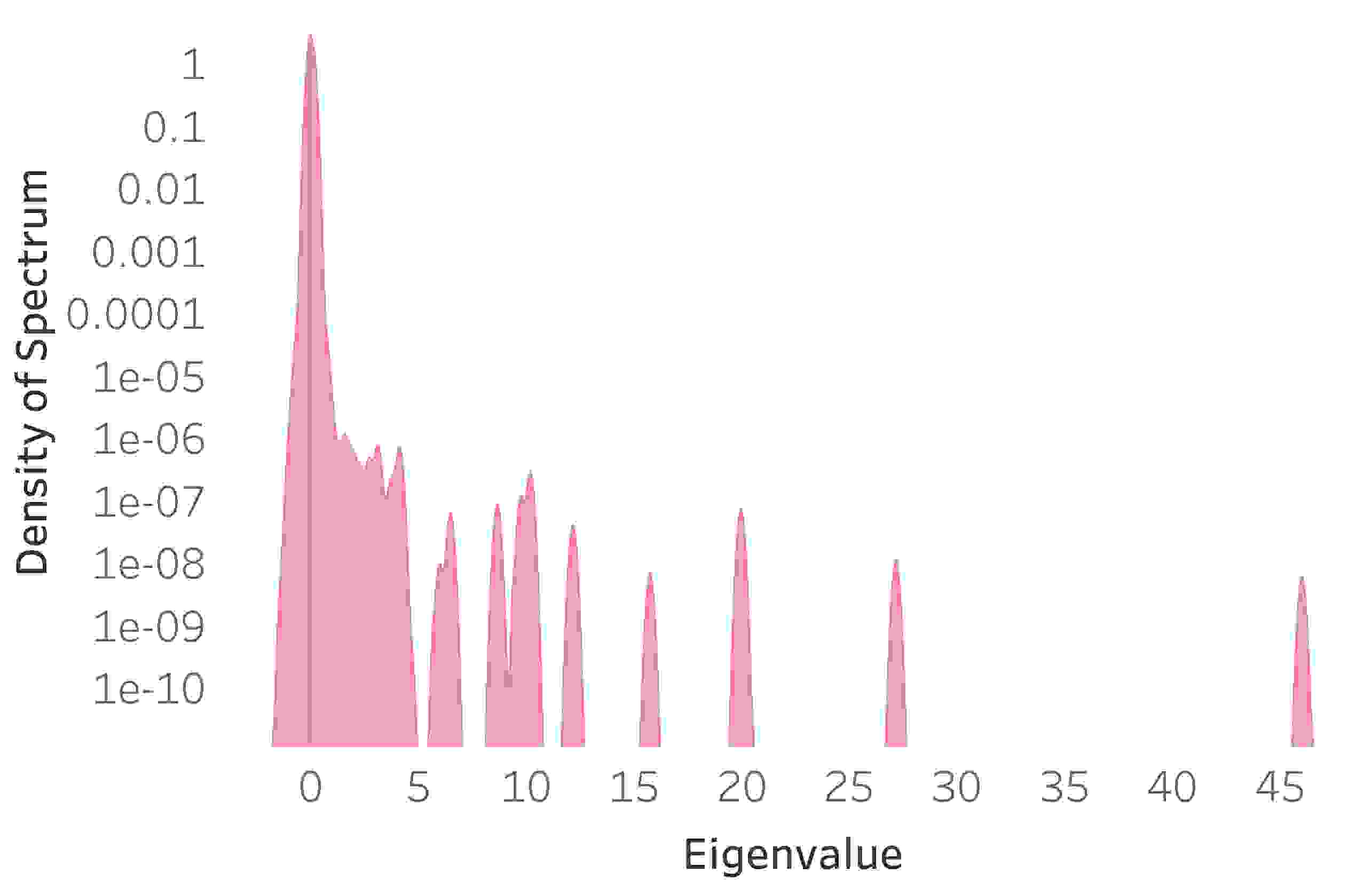}
        \caption{MNIST, test}
    \end{subfigure}
    \begin{subfigure}[t]{0.25\textwidth}
        \centering
        \includegraphics[width=1\textwidth]{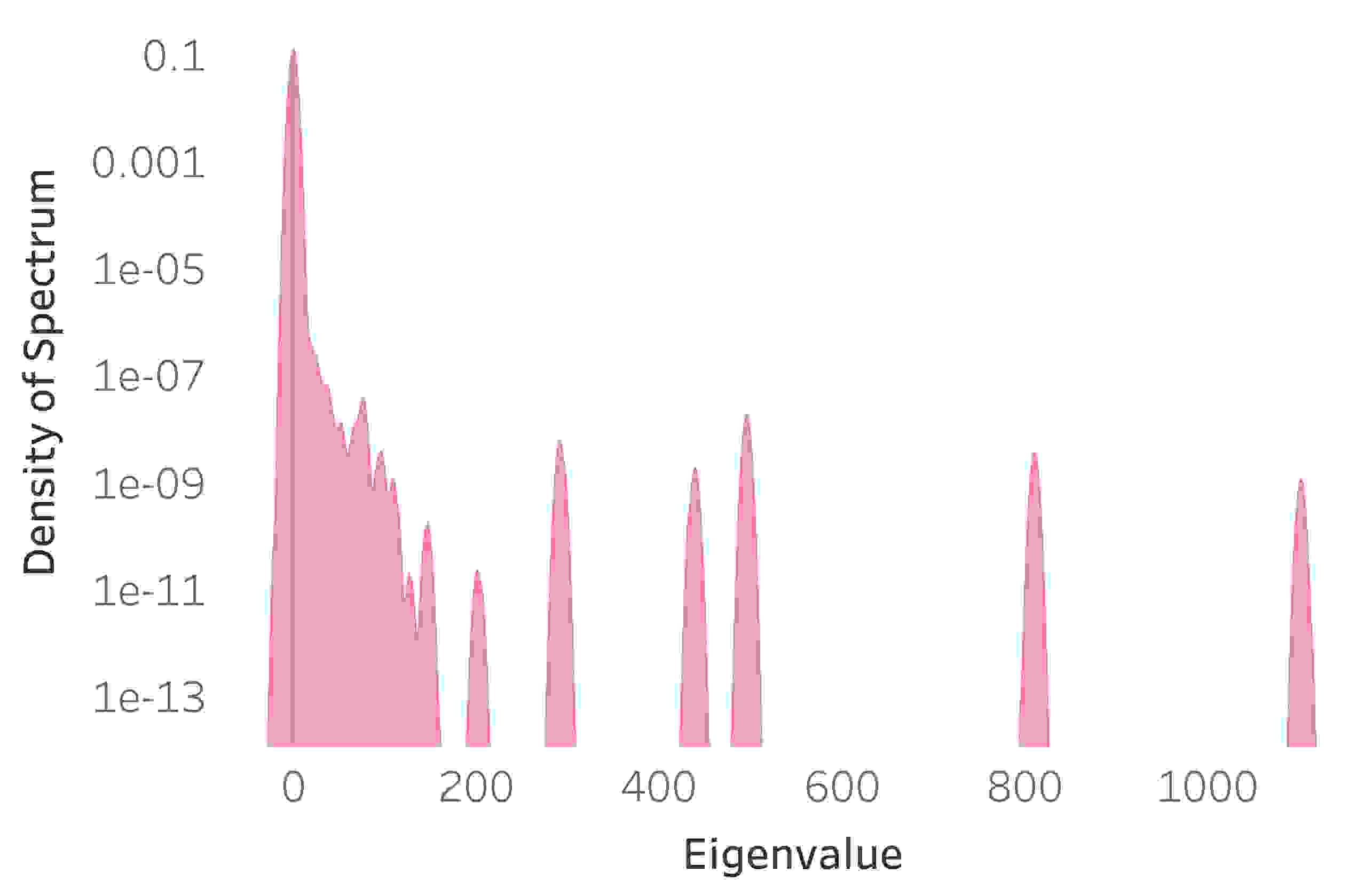}
        \caption{Fashion, test}
    \end{subfigure}
    \begin{subfigure}[t]{0.25\textwidth}
        \centering
        \includegraphics[width=1\textwidth]{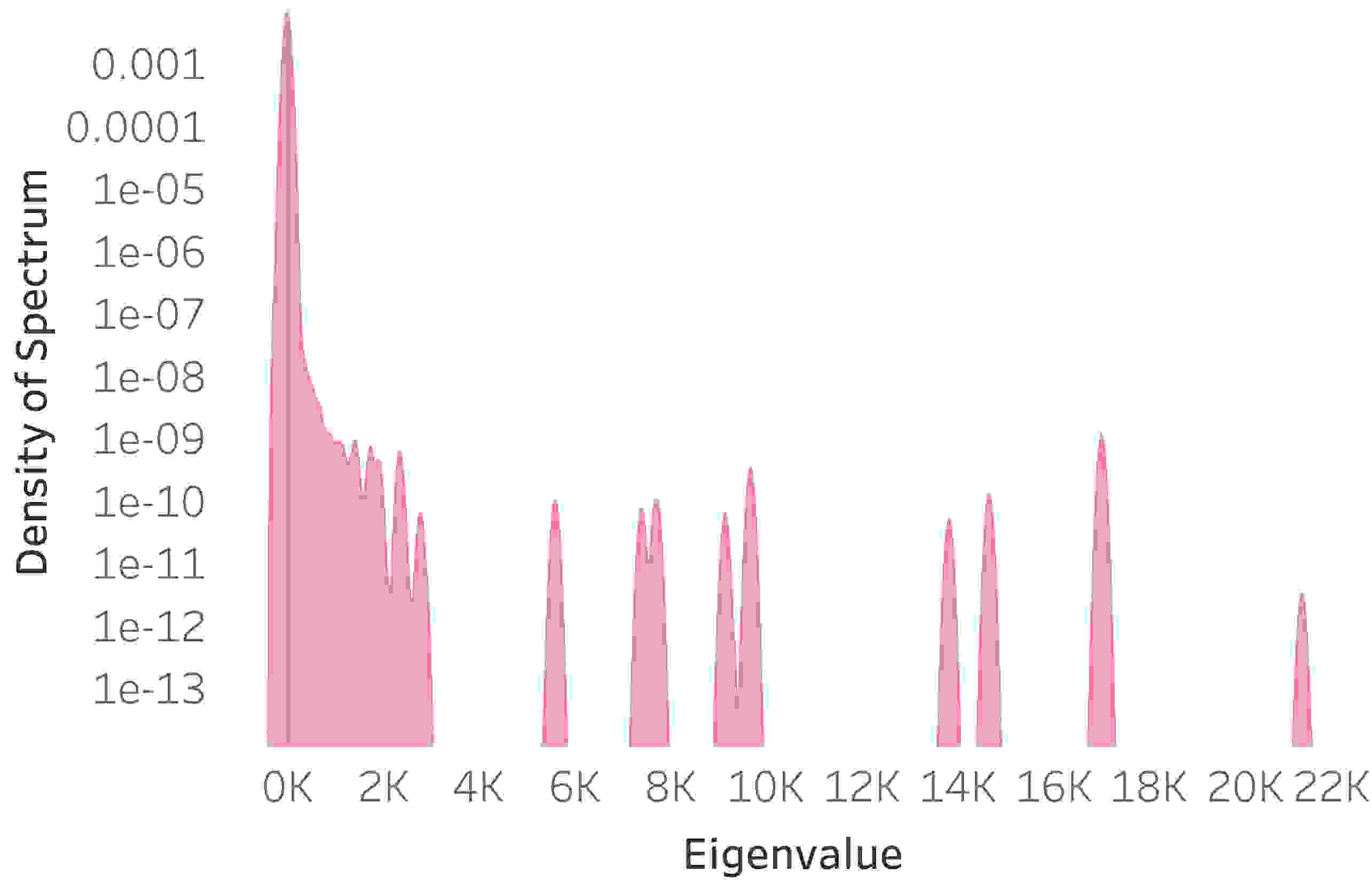}
        \caption{CIFAR10, test}
    \end{subfigure}
    \begin{subfigure}[t]{0.25\textwidth}
        \centering
        \includegraphics[width=1\textwidth]{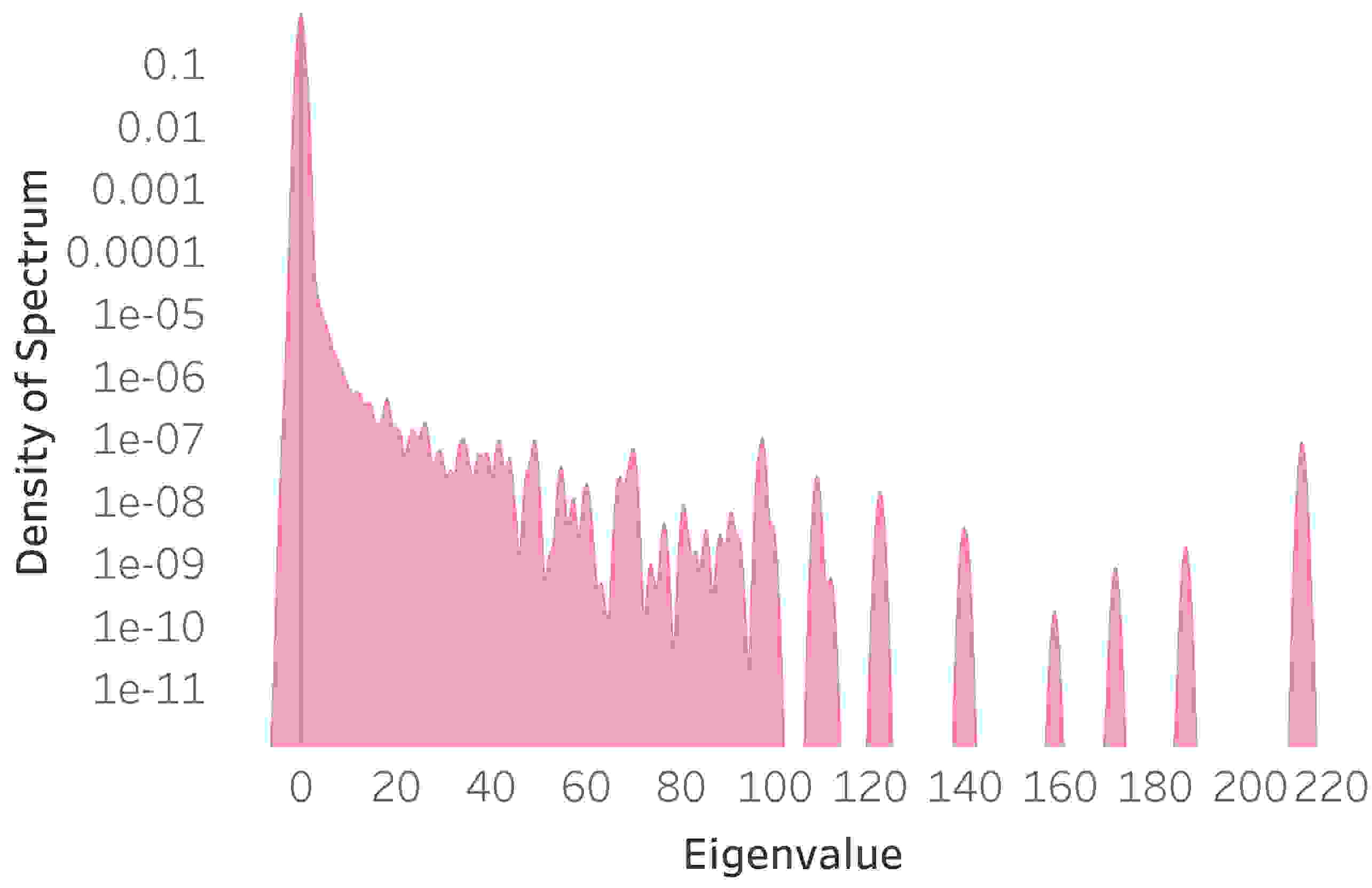}
        \caption{CIFAR100, test}
    \end{subfigure}
    \caption{\textbf{Spectrum of the Hessian for VGG11 trained on various datasets.} Each panel corresponds to a different famous dataset in deep learning. The spectrum was approximated using \textsc{LanczosApproxSpec}. Unlike the figure in the main manuscript, the top-$C$ eigenspace was not removed using $\textsc{LowRankDeflation}$.}
    \label{VGG11_spectrum_train_test_without_SSI}
\end{figure}

\subsection{\textbf{\textsc{SlowLanczos}}}
The Lanczos algorithm \citep{lanczos1950iteration} computes the spectrum of a symmetric matrix $\A \in \R^{p \times p}$ by first reducing it to a tridiagonal form $\Tb_p \in \R^{p \times p}$ and then computing the spectrum of that matrix instead. The motivation is that computing the spectrum of a tridiagonal matrix is very efficient, requiring only $O(p^2)$ operations. The algorithm works by progressively building an adapted orthonormal basis $\V_m \in \R^{p \times m}$ that satisfies at each iteration the relation $\V_m^\T \A \V_m = \Tb_m$, where $\Tb_m \in \R^{m \times m}$ is a tridiagonal matrix. For completeness, we summarize its main steps in Algorithm \ref{alg:SlowLanczos}.

\begin{algorithm}[h!]
\caption{$\textbf{\textsc{SlowLanczos}}(\A)$}
\label{alg:SlowLanczos}
\KwIn{Linear operator $\A \in \R^{p \times p}$ with spectrum in the range $[-1,1]$.}
\phantom{}
\KwResult{Eigenvalues and eigenvectors of the tridiagonal matrix $\Tb_p$.}
\phantom{}
\For{$m=1,\dots,p$}{
    \eIf{$m == 1$}{
        sample $\vv \sim \mathcal{N}(0,\I)$\;
        $\vv_1 = \frac{\vv_1}{\|\vv_1\|_2}$\;
        $\ww = \A \vv_1$\;
    }{
        $\ww = \A \vv_m - \beta_{m-1} \vv_{m-1}$\;
    }
    $\alpha_m = \vv_m^\T \ww$\;
    $\ww = \ww - \alpha_m \vv_m$\;
    \tcc{reorthogonalization}
    $\ww = \ww - \V_m \V_m^\T \ww$\;
    $\beta_m = \|\ww\|_2$\;
    $\vv_{m+1} = \frac{\ww}{\beta_m}$\;
}
$\Tb_p =
\begin{bmatrix}
\alpha_1,   & \beta_1,    &             &               & \\
\beta_1,    & \alpha_2,   & \beta_2,    &               & \\
            & \beta_2,    & \alpha_3,   &               & \\
            &             &             & \ddots        & \beta_{p-1} \\
            &             &             & \beta_{p-1}   & \alpha_{p}
\end{bmatrix}$\;
$\{\theta_m\}_{m=1}^p, \{\y_m\}_{m=1}^p = \textrm{eig}(\Tb_p)$\;
\Return $\{\theta_m\}_{m=1}^p, \{\y_m\}_{m=1}^p$\;
\end{algorithm}

\subsection{Complexity of \textbf{\textsc{SlowLanczos}}}
Assume without loss of generality we are computing the spectrum of the train Hessian. Each of the $p$ iterations of the algorithm requires a single Hessian-vector multiplication, incurring $O(Np)$ complexity. The complexity due to all Hessian-vector multiplications is therefore $O(N p^2)$. The $m$'th iteration also requires a reorthogonalization step, which computes the inner product of $m$ vectors of length $p$ and costs $O(mp)$ complexity. Summing this over the iterations, $m=1,\dots,p$, the complexity incurred due to reorthogonalization is $O(p^3)$. The total runtime complexity of the algorithm is therefore $O(N p^2 + p^3)$. As for memory requirements, the algorithm constructs a basis $\V_p \in \R^{p \times p}$ and as such its memory complexity is $O(p^2)$. Since $p$ is in the order of magnitude of millions, both time and memory complexity make \textsc{SlowLanczos} impractical.

\begin{algorithm}[h!]
\caption{$\textbf{\textsc{FastLanczos}}(\A, M)$}
\label{alg:FastLanczos}
\KwIn{Linear operator $\A \in \R^{p \times p}$ with spectrum in the range $[-1,1]$.}
\myinput{Number of iterations $M$.}
\KwResult{Eigenvalues and eigenvectors of the tridiagonal matrix $\Tb_m$.}
\For{$m=1,\dots,M$}{
    \eIf{$m == 1$}{
        sample $\vv \sim \mathcal{N}(0,\I)$\;
        $\vv = \frac{\vv}{\|\vv\|_2}$\;
        $\vv_\nextt = \A \vv$\;
    }{
        $\vv_\nextt = \A \vv - \beta_{m-1} \vv_\prevv$\;
    }
    $\alpha_m = \vv_\nextt^\T \vv$\;
    $\vv_\nextt = \vv_\nextt - \alpha_m \vv$\;
    $\beta_m = \|\vv_\nextt\|_2$\;
    $\vv_\nextt = \frac{\vv_\nextt}{\beta_m}$\;
    $\vv_\prevv = \vv$\;
    $\vv = \vv_\nextt$\;
}
$\Tb_M =
\begin{bmatrix}
\alpha_1,   & \beta_1,    &             &               & \\
\beta_1,    & \alpha_2,   & \beta_2,    &               & \\
            & \beta_2,    & \alpha_3,   &               & \\
            &             &             & \ddots        & \beta_{M-1} \\
            &             &             & \beta_{M-1}   & \alpha_M
\end{bmatrix}$\;
$\{\theta_m\}_{m=1}^M, \{\y_m\}_{m=1}^M = \textrm{eig}(\Tb_M)$\;
\Return $\{\theta_m\}_{m=1}^M, \{\y_m\}_{m=1}^M$\;
\end{algorithm}

\begin{algorithm}[h!]
\caption{$\textbf{\textsc{LanczosApproxSpec}}$\newline$(\A, M, K, \nvec, \kappa)$}
\label{alg:LanczosApproxSpec}
\KwIn{Linear operator $\A \in \R^{p \times p}$ with spectrum in the range $[-1,1]$.}
\myinput{Number of iterations $M$.}
\myinput{Number of points $K$.}
\myinput{Number of repetitions $\nvec$.}
\KwResult{Density of the spectrum of $\A$ evaluated at $K$ evenly distributed points in the range $[-1,1]$.}
\For{$l=1,\dots,\nvec$}{
    $\{\theta_m^l\}_{m=1}^M, \{\y_m^l\}_{m=1}^M = \textsc{FastLanczos}(\A,M)$\;
}
$\{t_k\}_{k=1}^K = \textrm{linspace}(-1,1,K)$\;
\For{$k=1,\dots,K$}{
$\sigma = \frac{2}{(M-1)\sqrt{8 \log(\kappa)}}$\;
$\phi_k = \frac{1}{\nvec} \sum_{l=1}^{\nvec} \sum_{m=1}^M {y_m^l[1]}^2 g_\sigma(t - \theta_m^l)$
}
\Return $\{\phi_k\}_{k=1}^K$\;
\end{algorithm}

\subsection{Spectral density estimation via \textbf{\textsc{FastLanczos}}} \label{sec:FastLanczos}
As a first step towards making Lanczos suitable for the problem we attack in this paper, we remove the reorthogonalization step in Algorithm \ref{alg:SlowLanczos}. This allows us to save only three terms--$\vv_\prevv$, $\vv$ and $\vv_\nextt$--instead of the whole matrix $\V_m \in \R^{p \times m}$ (see Algorithm \ref{alg:FastLanczos}). This greatly reduces the memory complexity of the algorithm at the cost of a nuisance that is discussed in Section \ref{sec:nuisance}. Moreover, it removes the $O(p^3)$ term from the runtime complexity.

In light of the runtime complexity analysis in the previous subsection, it is clear that running Lanczos for $p$ iterations is impractical. Realizing that, the authors of \citet{lin2016approximating} proposed to run the algorithm for $M \ll p$ iterations and compute an approximation to the spectrum based on the eigenvalues $\{\theta_m\}_{m=1}^M$ and eigenvectors $\{\y_m\}_{m=1}^M$ of $\Tb_M$. Denoting by $y_m[1]$ the first element in $\y_m$, their proposed approximation was $\hat{\phi}(t) = \sum_{m=1}^M y_m[1]^2 g_\sigma(t - \theta_m^l)$, where $g_\sigma(t - \theta_m^l)$ is a Gaussian with width $\sigma$ centered at $\theta_m^l$. Intuitively, instead of computing the true spectrum $\phi(t) = \frac{1}{p} \sum_{i=1}^p \delta(t - \lambda_i)$, their algorithm computes only $M \ll p$ eigenvalues and replaces each with a Gaussian bump. They further proposed to improve the approximation by starting the algorithm from several different starting vectors, $\vv_1^l, l=1,\dots,\nvec$, and averaging the results. We summarize \textsc{FastLanczos} in Algorithm \ref{alg:FastLanczos} and \textsc{LanczosApproxSpec} in Algorithm \ref{alg:LanczosApproxSpec}.

\subsection{Complexity of \textbf{\textsc{FastLanczos}}}
Each of the $M$ iterations requires a single Hessian-vector multiplication. As previously mentioned, this product requires $O(N p)$ complexity. The total runtime complexity of the algorithm is therefore $O(M N p)$ (for $n_{vec}=1$). Although the complexity might seem equivalent to that of training a model from scratch for $M$ epochs, this is not the case. The batch size used for training a model is usually limited ($128$ in our case) so as to no deteriorate the model's generalization. Such limitations do not apply to \textsc{FastLanczos}, which can utilize the largest possible batch size that fits into the GPU memory ($1024$ in our case). As for memory requirements; we only save three vectors and as such the memory complexity is merely $O(p)$.

\subsection{Reorthogonalization} \label{sec:nuisance}
Under exact arithmetic, the Lanczos algorithm constructs an orthonormal basis. However, in practice the calculations are performed in floating point arithmetic, resulting in loss of orthogonality. This is why the reorthogonalization step in Algorithm \ref{alg:SlowLanczos} was introduced in the first place. From our experience, we did not find the lack of reorthogonalization to cause any issue, except for the known phenomenon of appearance of ``ghost'' eigenvalues--multiple copies of eigenvalues, which are unrelated to the actual multiplicities of the eigenvalues. Despite these, in all the toy examples we ran on synthetic data, we found that our method approximates the spectrum well, as is shown in Figure \ref{synthetic}.

\begin{figure}[h]
    \begin{subfigure}[t]{0.48\textwidth}
        \centering
        \includegraphics[width=1\textwidth]{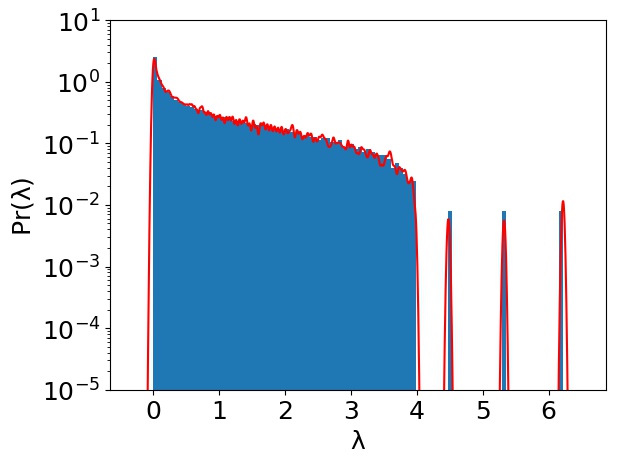}
        \caption{\textbf{Verification of \textsc{FastLanczos}.} Approximating the spectrum of a matrix $\Y \in \R^{2000 \times 2000}$, sampled from the distribution $\Y = \X + \frac{1}{2000} \Zb \Zb^\T$, where $X_{1,1} = 5$, $X_{2,2} = 4$, $X_{3,3} = 3$, $X_{i,j} = 0$ elsewhere and the entries of $\Zb \in \R^{2000 \times 2000}$ are standard normally distributed.}
    \end{subfigure}%
    ~\hspace{0.02\textwidth}%
    \begin{subfigure}[t]{0.48\textwidth}
        \centering
        \includegraphics[width=1\textwidth]{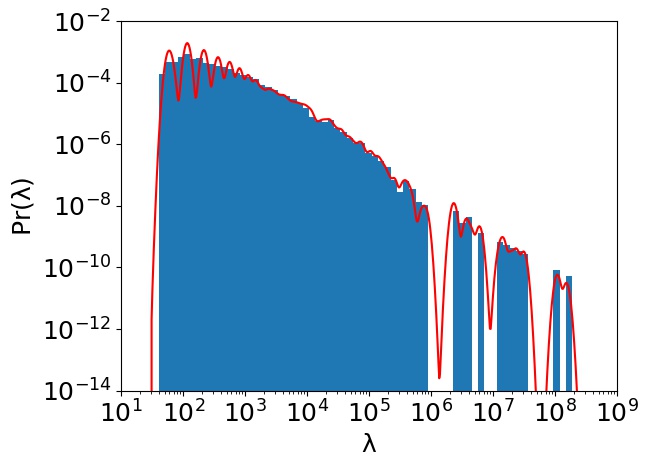}
        \caption{\textbf{Verification of \textsc{FastLanczos} for approximating the log-spectrum.} Approximating the log-spectrum of a matrix $\Y \in \R^{1000 \times 1000}$, sampled from the distribution $\Y = \frac{1}{1000} \Zb \Zb^\T$, where the entries of $\Zb \in \R^{500 \times 1000}$ are distributed i.i.d Pareto with index $\alpha=1$. This type of matrices are known to have a power law spectral density.}
    \end{subfigure}
    \caption{\textbf{Verification of spectrum approximation on synthetic data.} The eigenvalues obtained from eigenvalue decomposition are plotted in blue color in a histogram with $100$ bins. Our spectral approximation is plotted on top as a red line. In both the left and the right plots we average over $\nvec=10$ initial vectors.}
    \label{synthetic}
\end{figure}

\subsection{Normalization}
As a first step towards approximating the spectrum of a large matrix, we renormalize its range to $[-1,1]$. This can be done using any method that allows to approximate the maximal and minimal eigenvalue of a matrix--for example, the power method. In this work we follow the method proposed in \citet{lin2016approximating}. This normalization has the benefit of allowing us to set $\sigma$ to a fixed number, which does not depend on the specific spectrum approximated. We summarize the procedure in Algorithm \ref{alg:Normalization}.

\begin{algorithm}[h!]
\caption{$\textbf{\textsc{Normalization}}(\A, M_0, \tau)$}
\label{alg:Normalization}
\KwIn{Linear operator $\A \in \R^{p \times p}$.}
\myinput{Number of iterations $M_0$.}
\myinput{Margin percentage $\tau$.}
\KwResult{Linear operator $\A \in \R^{p \times p}$ with spectrum in the range $[-1,1]$.}
\tcc{approximate minimal and maximal eigenvalues}
$\{\theta_m\}_{m=1}^{M_0}, \{\y_m\}_{m=1}^{M_0} = \textsc{FastLanczos}(\A,M_0)$\;
$\lambda_{\min} = \theta_1 - \| (\A - \theta_1 \I) \y_1 \|$\;
$\lambda_{\max} = \theta_{M_0} + \| (\A - \theta_{M_0} \I) y_{M_0} \|$\;
\tcc{add margin}
$\Delta = \tau (\lambda_{\max} - \lambda_{\min})$\;
$\lambda_{\min} = \lambda_{\min} - \Delta$\;
$\lambda_{\max} = \lambda_{\max} + \Delta$\;
\tcc{normalized operator}
$c = \frac{\lambda_{\min} + \lambda_{\max}}{2}$\;
$d = \frac{\lambda_{\max} - \lambda_{\min}}{2}$\;
\Return $\frac{\A - c \I}{d}$\;
\end{algorithm}

\subsection{Spectral density estimation of $\boldmath{f(A)}$} \label{sec:log}
Figure \ref{bulk_dist_regular} approximates the spectrum of $\Eb$, showing that it is approximately linear on a log-log plot. A better idea would have been to approximate the spectrum of $\log(|\Eb|)$ in the first place (the absolute value is due to $\Eb$ being symmetric about the origin), since this would lead to a more precise estimate. Mathematically speaking, this amounts to approximating the measure $\Pr(\log(|\lambda|)) d\log(|\lambda|)$ instead of $\Pr(\lambda) d\lambda$. Using change of measure arguments, we have
\begin{align*}
    \Pr(\log(|\lambda|)) d\log(|\lambda|) \nonumber
    = \Pr(\log(|\lambda|)) \frac{d \log(|\lambda|)}{d |\lambda|} d |\lambda| \nonumber
    = \Pr(\log(|\lambda|)) \frac{1}{|\lambda|} d |\lambda|.
\end{align*}
Following the ideas presented in Section \ref{sec:FastLanczos}, the above can be approximated as 
\begin{equation*}
    \sum_{m=1}^M y_m[1]^2 \frac{1}{\theta_m^l} g_\sigma(|\lambda| - \log (\theta_m^l)).
\end{equation*}
Implementation-wise, all that is required is to replace $\theta_m^l$ with $\log(\theta_m^l)$ in Algorithm \ref{alg:LanczosApproxSpec}, scale the Gaussian bumps by $\frac{1}{\theta_m^l}$, and to apply $\log$ on $|\lambda_{\min}|$ and $|\lambda_{\max}|$ before adding the margin in Algorithm \ref{alg:Normalization}. In practice, we apply $f=\log(|\lambda| + \epsilon)$, where $\epsilon$ is a small constant added for numerical stability. Figure \ref{bulk_dist_log} shows the outcome of such procedure. The idea of approximating the log of the spectrum (or any function of it) is inspired by a recent work of \citet{ubaru2017fast}, which suggests a method for approximating $\Tr{f(\A)}$ using Lanczos and comments that similar ideas could be used for approximating functions of matrix spectra.

\begin{figure}[h!]
    \begin{subfigure}[t]{0.45\textwidth}
      \centering
      \includegraphics[width=1\textwidth]{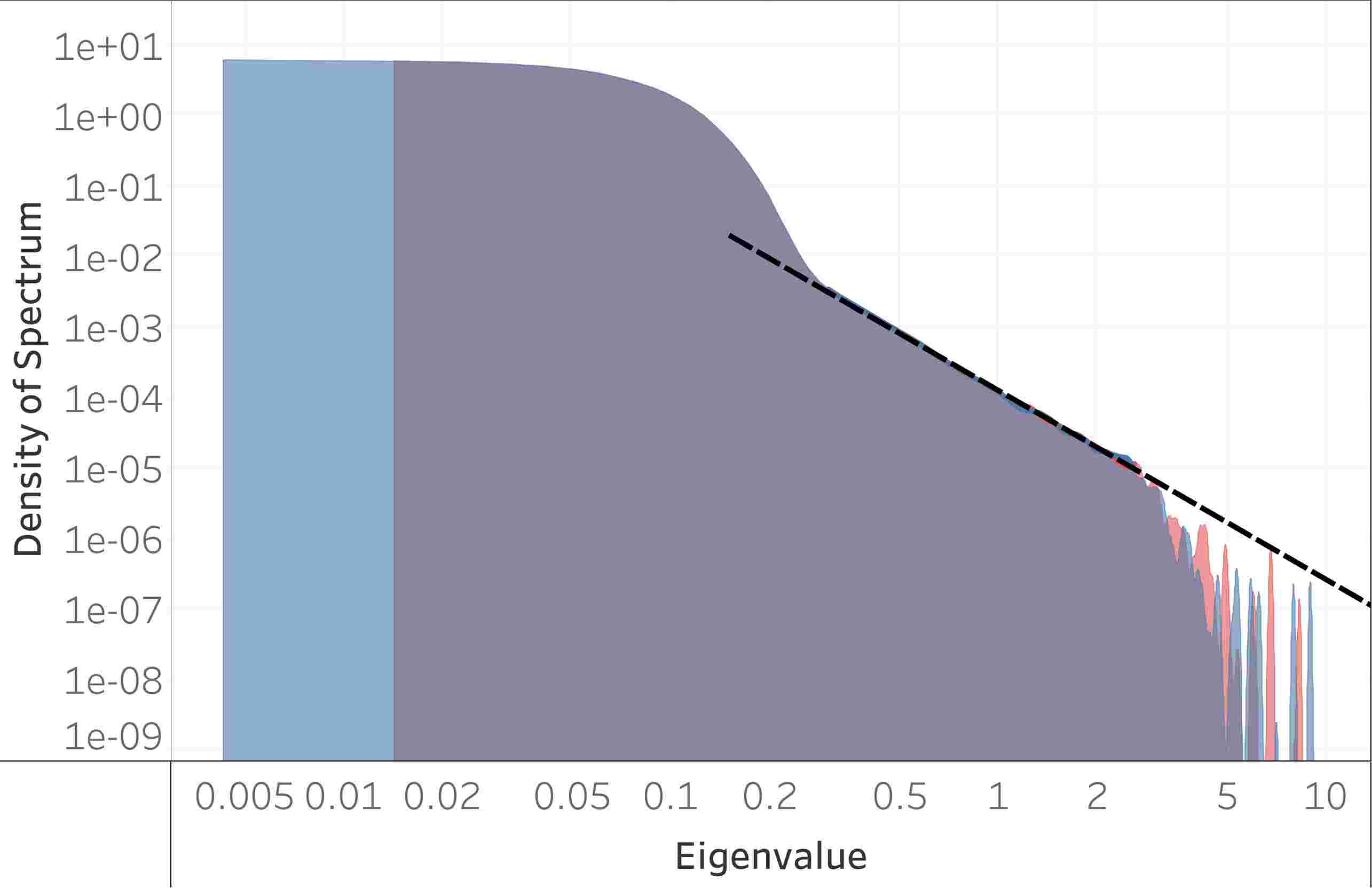}
      \caption{Spectrum of $\Eb$ on a logarithmic x-axis scale}
      \label{bulk_dist_regular}
    \end{subfigure}%
    \hspace{0.01\textwidth}~
    \begin{subfigure}[t]{0.45\textwidth}
      \centering
      \includegraphics[width=1\textwidth]{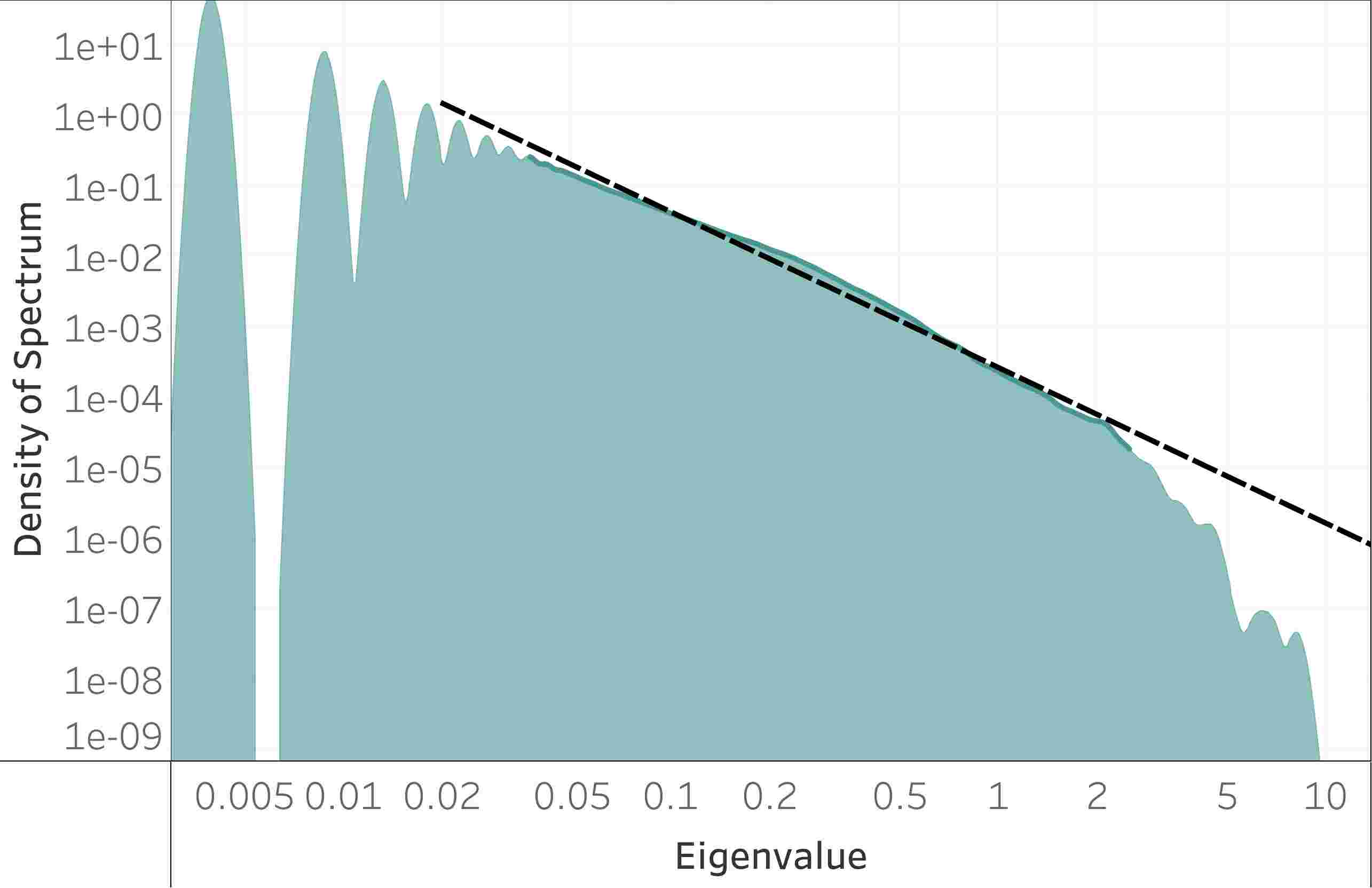}
      \caption{Spectrum of $\log(\Eb)$}
      \label{bulk_dist_log}
    \end{subfigure}
    \caption{\textbf{Tail properties of $\Eb$.} Spectrum of the test $\Eb$ for VGG11 trained on MNIST sub-sampled to 1351 examples per class. On the left we approximate the spectrum of $\Eb$ and plot it on a logarithmic x-axis. The positive eigenvalues of $\Eb$ are plotted in red and the absolute value of the negative ones in blue. Notice how the spectrum is almost perfectly symmetric about the origin. Fitting a power law trend on part of the spectrum results in a fit $\phi = 1.2{\times}10^{-4} \ |\lambda|^{-2.7}$ with an $R^2$ of $0.99$. On the right we approximate the spectrum of $\log(\Eb)$. Fitting a power law trend results in a fit $\phi = 2.6{\times}10^{-4} \ |\lambda|^{-2.2}$ with an $R^2$ of $0.99$. These spectra can not originate from Wigner's semicircle law, nor other classical RMT distributions}
    \label{benefits_log}
\end{figure}

\begin{algorithm}[h!]
\caption{$\textbf{\textsc{SubspaceIteration}}(\A,C,T)$}
\label{alg:SubspaceIteration}
\KwIn{Linear operator $\A \in \R^{p \times p}$.}
\myinput{Rank $C$.}
\myinput{Number of iterations $T$.}
\KwResult{Eigenvalues $\{\lambda_c\}_{c=1}^C$.}
\myresult{Eigenvectors $\{\vv_c\}_{c=1}^C$.}
\For{$c=1,\dots,C$}{
    sample $\vv_c \sim \mathcal{N}(0,\I)$\;
    $\vv_c = \frac{\vv_c}{\|\vv_c\|_2}$\;
}
$\Q = \textrm{QR}(\V)$\;
\For{$t=1,\dots,T$}{
    $\V = \A \Q$\;
    $\Q = \textrm{QR}(\V)$\;
}
\For{$c=1,\dots,C$}{
    $\lambda_c = \|\vv_c\|_2$
}
\Return $\{\lambda_c\}_{c=1}^C$, $\{\vv_c\}_{c=1}^C$
\end{algorithm}

\begin{figure}[h!]
  \centering
  \includegraphics[width=1\textwidth]{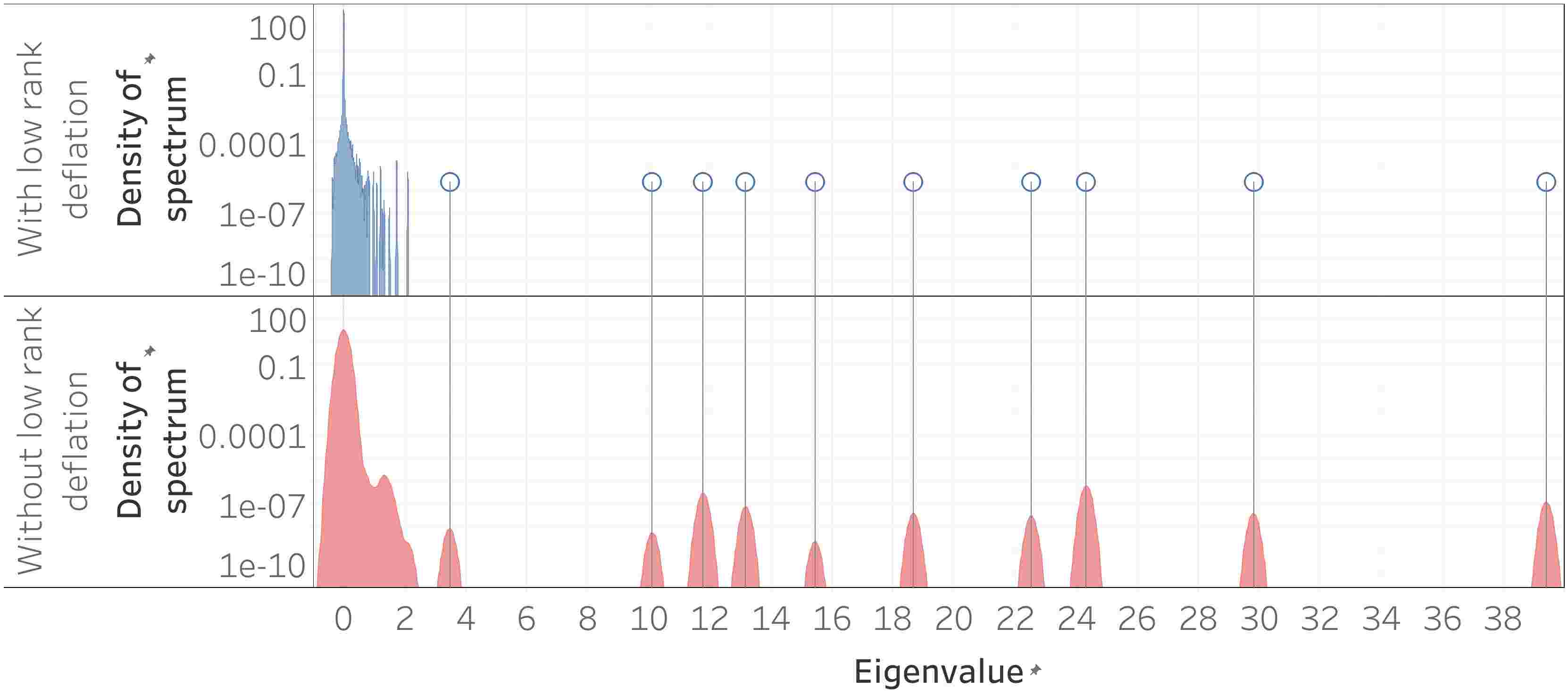}
  \caption{\textbf{Benefits of low rank deflation  using subspace iteration.} Spectrum of the train Hessian for ResNet18 trained on MNIST with 136 examples per class. Top panel: \textsc{SubspaceIteration} followed by \textsc{LanczosApproxSpec}. Bottom panel: \textsc{LanczosApproxSpec} only. Notice how the top eigenvalues at the top panel align with the outliers in the bottom panel. Low rank deflation allows for precise detection of outlier location and improved `resolution' of the bulk distribution.}
  \label{benefits_deflation}
\end{figure}

\subsection{\textbf{\textsc{SubspaceIteration}}}
The spectrum of the Hessian follows a bulk-and-outliers structure. Moreover, the number of outliers is approximately equal to $C$, the number of classes in the classification problem. It is therefore natural to extract the top $C$ outliers using, for example, the subspace iteration algorithm, and then to apply \textsc{LanczosApproxSpec} on a rank $C$ deflated operator to approximate the bulk. We demonstrate the benefits of \textsc{SubspaceIteration} in Figure \ref{benefits_deflation} and summarize its steps in Algorithm \ref{alg:SubspaceIteration}. The runtime complexity of \textsc{SubspaceIteration} is $O(T C^2 N p)$, $T$ being the number of iterations, which is $C^2$ times higher than that of \textsc{FastLanczos}.

\section{Experimental details} \label{sec:experimental_details}
\subsection{Training networks}
We present here results from training the VGG11 \citep{simonyan2014very} and ResNet18 \citep{he2016deep} architectures on the MNIST \citep{lecun2010mnist}, FashionMNIST \citep{xiao2017fashion}, CIFAR10 and CIFAR100 \citep{krizhevsky2009learning} datasets. We use stochastic gradient descent with $0.9$ momentum, $5{\times}10^{-4}$ weight decay and $128$ batch size. We train for $200$ epochs in the case of MNIST and FashionMNIST and $350$ in the case of CIFAR10 and CIFAR100, annealing the initial learning rate by a factor of $10$ at $1/3$ and $2/3$ of the number of epochs. For each dataset and network, we sweep over $100$ logarithmically spaced initial learning rates in the range $[0.25,0.0001]$ and pick the one that results in the best test error in the last epoch. For each dataset and network, we repeat the previous experiments on $20$ training sample sizes logarithmically spaced in the range $[10, 5000]$.

We also train an eight-layer multilayer perceptrons (MLP) with $2048$ neurons in each hidden layer on the same datasets. We use the same hyperparameters, except we train for $350$ epochs for all datasets and optimize the initial learning rate over $25$ logarithmically spaced values.

The massive computational experiments reported here were run painlessly using ClusterJob and ElastiCluster \citep{clusterjob,MMCEP17,Monajemi19}.

\subsection{Analyzing the spectra}
For each operator $\A$, we begin by computing $\textsc{Normalization} (\A, M_0{=}32, \tau{=}0.05)$. We then approximate the spectrum using $\textsc{LanczosApproxSpec} (A, M, K{=}1024, \nvec{=}1, \kappa{=}3)$, where $M \in \{128,256\}$. Finally, we denormalize the spectrum into its original range (which is not $[-1,1]$). Optionally, we apply the above steps on a rank-$C$ deflated operator obtained using \\ $\textsc{SubspaceIteration} (\A, C, T{=}128)$. Optionally, we compute the log-spectrum, in which case we use $M{=}2048$ iterations and a different value of $\kappa$.

\subsection{Removing sources of randomness}
The methods we employ in this paper--including Lanczos and subspace iteration--assume deterministic linear operators. As such, we train our networks without preprocessing the input data using random flips or crops. Moreover, we replace dropout layers \citep{srivastava2014dropout} with batch normalization ones \citep{ioffe2015batch} in the VGG architecture. The batch normalization layers are always set to ``test mode''.

\vskip 0.2in
\bibliography{sample}

\begin{thebibliography}{85}
\providecommand{\natexlab}[1]{#1}
\providecommand{\url}[1]{\texttt{#1}}
\expandafter\ifx\csname urlstyle\endcsname\relax
  \providecommand{\doi}[1]{doi: #1}\else
  \providecommand{\doi}{doi: \begingroup \urlstyle{rm}\Url}\fi

\bibitem[Alain et~al.(2019)Alain, Roux, and Manzagol]{alain2019negative}
Guillaume Alain, Nicolas~Le Roux, and Pierre-Antoine Manzagol.
\newblock Negative eigenvalues of the hessian in deep neural networks.
\newblock \emph{arXiv preprint arXiv:1902.02366}, 2019.

\bibitem[Amari(1998)]{amari1998natural}
Shun-Ichi Amari.
\newblock Natural gradient works efficiently in learning.
\newblock \emph{Neural computation}, 10\penalty0 (2):\penalty0 251--276, 1998.

\bibitem[Andreassen and Dyer(2020)]{andreassen2020asymptotics}
Anders Andreassen and Ethan Dyer.
\newblock Asymptotics of wide convolutional neural networks.
\newblock \emph{arXiv preprint arXiv:2008.08675}, 2020.

\bibitem[Ba et~al.(2016)Ba, Grosse, and Martens]{ba2016distributed}
Jimmy Ba, Roger Grosse, and James Martens.
\newblock Distributed second-order optimization using kronecker-factored
  approximations.
\newblock 2016.

\bibitem[B{\"o}hning(1992)]{bohning1992multinomial}
Dankmar B{\"o}hning.
\newblock Multinomial logistic regression algorithm.
\newblock \emph{Annals of the institute of Statistical Mathematics},
  44\penalty0 (1):\penalty0 197--200, 1992.

\bibitem[Chatzimichailidis et~al.(2019)Chatzimichailidis, Pfreundt, Gauger, and
  Keuper]{chatzimichailidis2019gradvis}
Avraam Chatzimichailidis, Franz-Josef Pfreundt, Nicolas~R Gauger, and Janis
  Keuper.
\newblock Gradvis: Visualization and second order analysis of optimization
  surfaces during the training of deep neural networks.
\newblock \emph{arXiv preprint arXiv:1909.12108}, 2019.

\bibitem[Collobert and Bengio(2004)]{collobert2004gentle}
Ronan Collobert and Samy Bengio.
\newblock A gentle hessian for efficient gradient descent.
\newblock In \emph{2004 IEEE International Conference on Acoustics, Speech, and
  Signal Processing}, volume~5, pages V--517. IEEE, 2004.

\bibitem[Dauphin et~al.(2014)Dauphin, Pascanu, Gulcehre, Cho, Ganguli, and
  Bengio]{dauphin2014identifying}
Yann~N Dauphin, Razvan Pascanu, Caglar Gulcehre, Kyunghyun Cho, Surya Ganguli,
  and Yoshua Bengio.
\newblock Identifying and attacking the saddle point problem in
  high-dimensional non-convex optimization.
\newblock In \emph{Advances in neural information processing systems}, pages
  2933--2941, 2014.

\bibitem[Dinh et~al.(2017)Dinh, Pascanu, Bengio, and Bengio]{dinh2017sharp}
Laurent Dinh, Razvan Pascanu, Samy Bengio, and Yoshua Bengio.
\newblock Sharp minima can generalize for deep nets.
\newblock \emph{arXiv preprint arXiv:1703.04933}, 2017.

\bibitem[Drabold and Sankey(1993)]{drabold1993maximum}
David~A Drabold and Otto~F Sankey.
\newblock Maximum entropy approach for linear scaling in the electronic
  structure problem.
\newblock \emph{Physical review letters}, 70\penalty0 (23):\penalty0 3631,
  1993.

\bibitem[Ducastelle and Cyrot-Lackmann(1970)]{ducastelle1970moments}
Francois Ducastelle and Fran{\c{c}}oise Cyrot-Lackmann.
\newblock Moments developments and their application to the electronic charge
  distribution of d bands.
\newblock \emph{Journal of physics and chemistry of solids}, 31\penalty0
  (6):\penalty0 1295--1306, 1970.

\bibitem[Dyer and Gur-Ari(2019)]{dyer2019asymptotics}
Ethan Dyer and Guy Gur-Ari.
\newblock Asymptotics of wide networks from feynman diagrams.
\newblock \emph{arXiv preprint arXiv:1909.11304}, 2019.

\bibitem[Fort and Ganguli(2019)]{fort2019emergent}
Stanislav Fort and Surya Ganguli.
\newblock Emergent properties of the local geometry of neural loss landscapes.
\newblock \emph{arXiv preprint arXiv:1910.05929}, 2019.

\bibitem[Fort and Jastrz{\k{e}}bski(2019)]{fort2019large}
Stanislav Fort and Stanis{\l}aw Jastrz{\k{e}}bski.
\newblock Large scale structure of neural network loss landscapes.
\newblock \emph{arXiv preprint arXiv:1906.04724}, 2019.

\bibitem[Fort and Scherlis(2019)]{fort2019goldilocks}
Stanislav Fort and Adam Scherlis.
\newblock The goldilocks zone: Towards better understanding of neural network
  loss landscapes.
\newblock In \emph{Proceedings of the AAAI Conference on Artificial
  Intelligence}, volume~33, pages 3574--3581, 2019.

\bibitem[Geiger et~al.(2019)Geiger, Spigler, Jacot, and
  Wyart]{geiger2019disentangling}
Mario Geiger, Stefano Spigler, Arthur Jacot, and Matthieu Wyart.
\newblock Disentangling feature and lazy training in deep neural networks.
\newblock \emph{arXiv preprint arXiv:1906.08034}, 2019.

\bibitem[Geiger et~al.(2020)Geiger, Jacot, Spigler, Gabriel, Sagun, d’Ascoli,
  Biroli, Hongler, and Wyart]{geiger2020scaling}
Mario Geiger, Arthur Jacot, Stefano Spigler, Franck Gabriel, Levent Sagun,
  St{\'e}phane d’Ascoli, Giulio Biroli, Cl{\'e}ment Hongler, and Matthieu
  Wyart.
\newblock Scaling description of generalization with number of parameters in
  deep learning.
\newblock \emph{Journal of Statistical Mechanics: Theory and Experiment},
  2020\penalty0 (2):\penalty0 023401, 2020.

\bibitem[Ghorbani et~al.(2019)Ghorbani, Krishnan, and
  Xiao]{ghorbani2019investigation}
Behrooz Ghorbani, Shankar Krishnan, and Ying Xiao.
\newblock An investigation into neural net optimization via hessian eigenvalue
  density.
\newblock In \emph{International Conference on Machine Learning}, pages
  2232--2241, 2019.

\bibitem[Granziol et~al.()Granziol, Garipov, Zohren, Vetrov, Roberts, and
  Wilson]{Granziol2019}
Diego Granziol, Timur Garipov, Stefan Zohren, Dmitry Vetrov, Stephen Roberts,
  and Andrew~Gordon Wilson.
\newblock The deep learning limit: are negative neural network eigenvalues just
  noise?

\bibitem[Grosse and Martens(2016)]{grosse2016kronecker}
Roger Grosse and James Martens.
\newblock A kronecker-factored approximate fisher matrix for convolution
  layers.
\newblock In \emph{International Conference on Machine Learning}, pages
  573--582, 2016.

\bibitem[Grosse and Salakhudinov(2015)]{grosse2015scaling}
Roger Grosse and Ruslan Salakhudinov.
\newblock Scaling up natural gradient by sparsely factorizing the inverse
  fisher matrix.
\newblock In \emph{International Conference on Machine Learning}, pages
  2304--2313, 2015.

\bibitem[Gur-Ari et~al.(2018)Gur-Ari, Roberts, and Dyer]{gur2018gradient}
Guy Gur-Ari, Daniel~A Roberts, and Ethan Dyer.
\newblock Gradient descent happens in a tiny subspace.
\newblock \emph{arXiv preprint arXiv:1812.04754}, 2018.

\bibitem[He et~al.(2016)He, Zhang, Ren, and Sun]{he2016deep}
Kaiming He, Xiangyu Zhang, Shaoqing Ren, and Jian Sun.
\newblock Deep residual learning for image recognition.
\newblock In \emph{Proceedings of the IEEE conference on computer vision and
  pattern recognition}, pages 770--778, 2016.

\bibitem[Hochreiter and Schmidhuber(1997)]{hochreiter1997flat}
Sepp Hochreiter and J{\"u}rgen Schmidhuber.
\newblock Flat minima.
\newblock \emph{Neural Computation}, 9\penalty0 (1):\penalty0 1--42, 1997.

\bibitem[Horn and Johnson(2012)]{horn2012matrix}
Roger~A Horn and Charles~R Johnson.
\newblock \emph{Matrix analysis}.
\newblock Cambridge university press, 2012.

\bibitem[Huberty and Olejnik(2006)]{huberty2006applied}
Carl~J Huberty and Stephen Olejnik.
\newblock \emph{Applied MANOVA and discriminant analysis}, volume 498.
\newblock John Wiley \& Sons, 2006.

\bibitem[Ioffe and Szegedy(2015)]{ioffe2015batch}
Sergey Ioffe and Christian Szegedy.
\newblock Batch normalization: Accelerating deep network training by reducing
  internal covariate shift.
\newblock \emph{arXiv preprint arXiv:1502.03167}, 2015.

\bibitem[Jacot et~al.(2018)Jacot, Gabriel, and Hongler]{jacot2018neural}
Arthur Jacot, Franck Gabriel, and Cl{\'e}ment Hongler.
\newblock Neural tangent kernel: Convergence and generalization in neural
  networks.
\newblock In \emph{Advances in neural information processing systems}, pages
  8571--8580, 2018.

\bibitem[Jacot et~al.(2019{\natexlab{a}})Jacot, Gabriel, and
  Hongler]{jacot2019asymptotic}
Arthur Jacot, Franck Gabriel, and Cl{\'e}ment Hongler.
\newblock The asymptotic spectrum of the hessian of dnn throughout training.
\newblock \emph{arXiv preprint arXiv:1910.02875}, 2019{\natexlab{a}}.

\bibitem[Jacot et~al.(2019{\natexlab{b}})Jacot, Gabriel, and
  Hongler]{jacot2019freeze}
Arthur Jacot, Franck Gabriel, and Cl{\'e}ment Hongler.
\newblock Freeze and chaos for dnns: an ntk view of batch normalization,
  checkerboard and boundary effects.
\newblock \emph{arXiv preprint arXiv:1907.05715}, 2019{\natexlab{b}}.

\bibitem[Jastrz{\k{e}}bski et~al.(2017)Jastrz{\k{e}}bski, Kenton, Arpit,
  Ballas, Fischer, Bengio, and Storkey]{jastrzkebski2017three}
Stanis{\l}aw Jastrz{\k{e}}bski, Zachary Kenton, Devansh Arpit, Nicolas Ballas,
  Asja Fischer, Yoshua Bengio, and Amos Storkey.
\newblock Three factors influencing minima in sgd.
\newblock \emph{arXiv preprint arXiv:1711.04623}, 2017.

\bibitem[Jastrz{\k{e}}bski et~al.(2018{\natexlab{a}})Jastrz{\k{e}}bski, Kenton,
  Ballas, Fischer, Bengio, and Storkey]{jastrzebski2018relation}
Stanis{\l}aw Jastrz{\k{e}}bski, Zachary Kenton, Nicolas Ballas, Asja Fischer,
  Yoshua Bengio, and Amos Storkey.
\newblock On the relation between the sharpest directions of dnn loss and the
  sgd step length.
\newblock \emph{arXiv preprint arXiv:1807.05031}, 2018{\natexlab{a}}.

\bibitem[Jastrz{\k{e}}bski et~al.(2018{\natexlab{b}})Jastrz{\k{e}}bski, Kenton,
  Ballas, Fischer, Bengio, and Storkey]{jastrzkebski2018relation}
Stanis{\l}aw Jastrz{\k{e}}bski, Zachary Kenton, Nicolas Ballas, Asja Fischer,
  Yoshua Bengio, and Amos Storkey.
\newblock On the relation between the sharpest directions of dnn loss and the
  sgd step length.
\newblock 2018{\natexlab{b}}.

\bibitem[Jastrz{\k{e}}bski et~al.(2020)Jastrz{\k{e}}bski, Szymczak, Fort,
  Arpit, Tabor, Cho, and Geras]{jastrzebski2020break}
Stanis{\l}aw Jastrz{\k{e}}bski, Maciej Szymczak, Stanislav Fort, Devansh Arpit,
  Jacek Tabor, Kyunghyun Cho, and Krzysztof Geras.
\newblock The break-even point on optimization trajectories of deep neural
  networks.
\newblock \emph{arXiv preprint arXiv:2002.09572}, 2020.

\bibitem[Jiang et~al.(2018)Jiang, Krishnan, Mobahi, and
  Bengio]{jiang2018predicting}
Yiding Jiang, Dilip Krishnan, Hossein Mobahi, and Samy Bengio.
\newblock Predicting the generalization gap in deep networks with margin
  distributions.
\newblock In \emph{International Conference on Learning Representations}, 2018.

\bibitem[Jiang et~al.(2019)Jiang, Neyshabur, Mobahi, Krishnan, and
  Bengio]{jiang2019fantastic}
Yiding Jiang, Behnam Neyshabur, Hossein Mobahi, Dilip Krishnan, and Samy
  Bengio.
\newblock Fantastic generalization measures and where to find them.
\newblock \emph{arXiv preprint arXiv:1912.02178}, 2019.

\bibitem[Karakida et~al.(2019{\natexlab{a}})Karakida, Akaho, and
  Amari]{karakida2019pathological}
Ryo Karakida, Shotaro Akaho, and Shun-ichi Amari.
\newblock Pathological spectra of the fisher information metric and its
  variants in deep neural networks.
\newblock \emph{arXiv preprint arXiv:1910.05992}, 2019{\natexlab{a}}.

\bibitem[Karakida et~al.(2019{\natexlab{b}})Karakida, Akaho, and
  Amari]{karakida2019universal}
Ryo Karakida, Shotaro Akaho, and Shun-ichi Amari.
\newblock Universal statistics of fisher information in deep neural networks:
  Mean field approach.
\newblock In \emph{The 22nd International Conference on Artificial Intelligence
  and Statistics}, pages 1032--1041, 2019{\natexlab{b}}.

\bibitem[Keskar et~al.(2016)Keskar, Mudigere, Nocedal, Smelyanskiy, and
  Tang]{keskar2016large}
Nitish~Shirish Keskar, Dheevatsa Mudigere, Jorge Nocedal, Mikhail Smelyanskiy,
  and Ping Tak~Peter Tang.
\newblock On large-batch training for deep learning: Generalization gap and
  sharp minima.
\newblock \emph{arXiv preprint arXiv:1609.04836}, 2016.

\bibitem[Krizhevsky and Hinton(2009)]{krizhevsky2009learning}
Alex Krizhevsky and Geoffrey Hinton.
\newblock Learning multiple layers of features from tiny images.
\newblock Technical report, Citeseer, 2009.

\bibitem[Kylasa et~al.(2019)Kylasa, Roosta, Mahoney, and Grama]{kylasa2019gpu}
Sudhir Kylasa, Fred Roosta, Michael~W Mahoney, and Ananth Grama.
\newblock Gpu accelerated sub-sampled newton's method for convex classification
  problems.
\newblock In \emph{Proceedings of the 2019 SIAM International Conference on
  Data Mining}, pages 702--710. SIAM, 2019.

\bibitem[Lanczos(1950)]{lanczos1950iteration}
Cornelius Lanczos.
\newblock \emph{An iteration method for the solution of the eigenvalue problem
  of linear differential and integral operators}.
\newblock United States Governm. Press Office Los Angeles, CA, 1950.

\bibitem[LeCun et~al.(1998)LeCun, Bottou, Orr, and Muller]{lecun1998efficient}
Y~LeCun, L~Bottou, G~Orr, and K~Muller.
\newblock Efficient backprop in neural networks: Tricks of the trade (orr, g.
  and m{\"u}ller, k., eds.).
\newblock \emph{Lecture Notes in Computer Science}, 1524\penalty0
  (98):\penalty0 111, 1998.

\bibitem[LeCun et~al.(2010)LeCun, Cortes, and Burges]{lecun2010mnist}
Yann LeCun, Corinna Cortes, and CJ~Burges.
\newblock Mnist handwritten digit database.
\newblock \emph{AT\&T Labs [Online]. Available: http://yann. lecun.
  com/exdb/mnist}, 2:\penalty0 18, 2010.

\bibitem[LeCun et~al.(2012)LeCun, Bottou, Orr, and
  M{\"u}ller]{lecun2012efficient}
Yann~A LeCun, L{\'e}on Bottou, Genevieve~B Orr, and Klaus-Robert M{\"u}ller.
\newblock Efficient backprop.
\newblock In \emph{Neural networks: Tricks of the trade}, pages 9--48.
  Springer, 2012.

\bibitem[Lewkowycz et~al.(2020)Lewkowycz, Bahri, Dyer, Sohl-Dickstein, and
  Gur-Ari]{lewkowycz2020large}
Aitor Lewkowycz, Yasaman Bahri, Ethan Dyer, Jascha Sohl-Dickstein, and Guy
  Gur-Ari.
\newblock The large learning rate phase of deep learning: the catapult
  mechanism.
\newblock \emph{arXiv preprint arXiv:2003.02218}, 2020.

\bibitem[Li et~al.(2018)Li, Farkhoor, Liu, and Yosinski]{li2018measuring}
Chunyuan Li, Heerad Farkhoor, Rosanne Liu, and Jason Yosinski.
\newblock Measuring the intrinsic dimension of objective landscapes.
\newblock \emph{arXiv preprint arXiv:1804.08838}, 2018.

\bibitem[Li et~al.(2019)Li, Gu, Zhou, Chen, and Banerjee]{li2019hessian}
Xinyan Li, Qilong Gu, Yingxue Zhou, Tiancong Chen, and Arindam Banerjee.
\newblock Hessian based analysis of sgd for deep nets: Dynamics and
  generalization.
\newblock \emph{arXiv preprint arXiv:1907.10732}, 2019.

\bibitem[Lin et~al.(2016)Lin, Saad, and Yang]{lin2016approximating}
Lin Lin, Yousef Saad, and Chao Yang.
\newblock Approximating spectral densities of large matrices.
\newblock \emph{SIAM review}, 58\penalty0 (1):\penalty0 34--65, 2016.

\bibitem[Ma et~al.(2020)Ma, Montague, Ye, Yao, Gholami, Keutzer, and
  Mahoney]{ma2020inefficiency}
Linjian Ma, Gabe Montague, Jiayu Ye, Zhewei Yao, Amir Gholami, Kurt Keutzer,
  and Michael~W Mahoney.
\newblock Inefficiency of k-fac for large batch size training.
\newblock In \emph{AAAI}, pages 5053--5060, 2020.

\bibitem[Mahajan et~al.(2018)Mahajan, Girshick, Ramanathan, He, Paluri, Li,
  Bharambe, and van~der Maaten]{mahajan2018exploring}
Dhruv Mahajan, Ross Girshick, Vignesh Ramanathan, Kaiming He, Manohar Paluri,
  Yixuan Li, Ashwin Bharambe, and Laurens van~der Maaten.
\newblock Exploring the limits of weakly supervised pretraining.
\newblock In \emph{Proceedings of the European Conference on Computer Vision
  (ECCV)}, pages 181--196, 2018.

\bibitem[Mahoney and Martin(2019)]{mahoney2019traditional}
Michael Mahoney and Charles Martin.
\newblock Traditional and heavy tailed self regularization in neural network
  models.
\newblock In \emph{International Conference on Machine Learning}, pages
  4284--4293, 2019.

\bibitem[Martens and Grosse(2015)]{martens2015optimizing}
James Martens and Roger Grosse.
\newblock Optimizing neural networks with kronecker-factored approximate
  curvature.
\newblock In \emph{International conference on machine learning}, pages
  2408--2417, 2015.

\bibitem[Martin and Mahoney(2018)]{martin2018implicit}
Charles~H Martin and Michael~W Mahoney.
\newblock Implicit self-regularization in deep neural networks: Evidence from
  random matrix theory and implications for learning.
\newblock \emph{arXiv preprint arXiv:1810.01075}, 2018.

\bibitem[Monajemi and Donoho(2015)]{clusterjob}
H.~Monajemi and D.~L. Donoho.
\newblock Clusterjob: An automated system for painless and reproducible massive
  computational experiments.
\newblock \url{https://github.com/monajemi/clusterjob}, 2015.

\bibitem[Monajemi et~al.(2017)Monajemi, Donoho, and Stodden]{MMCEP17}
H.~Monajemi, D.~L. Donoho, and V.~Stodden.
\newblock Making massive computational experiments painless.
\newblock \emph{Big Data (Big Data), 2016 IEEE International Conference on},
  February 2017.

\bibitem[Monajemi et~al.(2019)Monajemi, Murri, Yonas, Liang, Stodden, and
  Donoho]{Monajemi19}
H.~Monajemi, R.~Murri, E.~Yonas, P.~Liang, V.~Stodden, and D.L. Donoho.
\newblock Ambitious data science can be painless.
\newblock \emph{arXiv:1901.08705}, 2019.

\bibitem[Oymak et~al.(2019)Oymak, Fabian, Li, and
  Soltanolkotabi]{oymak2019generalization}
Samet Oymak, Zalan Fabian, Mingchen Li, and Mahdi Soltanolkotabi.
\newblock Generalization guarantees for neural networks via harnessing the
  low-rank structure of the jacobian.
\newblock \emph{arXiv preprint arXiv:1906.05392}, 2019.

\bibitem[Papyan(2018)]{papyan2018full}
Vardan Papyan.
\newblock The full spectrum of deep net hessians at scale: Dynamics with sample
  size.
\newblock \emph{arXiv preprint arXiv:1811.07062}, 2018.

\bibitem[Papyan(2019)]{papyan2019measurements}
Vardan Papyan.
\newblock Measurements of three-level hierarchical structure in the outliers in
  the spectrum of deepnet hessians.
\newblock In \emph{International Conference on Machine Learning}, pages
  5012--5021, 2019.

\bibitem[Pascanu and Bengio(2013)]{pascanu2013revisiting}
Razvan Pascanu and Yoshua Bengio.
\newblock Revisiting natural gradient for deep networks.
\newblock \emph{arXiv preprint arXiv:1301.3584}, 2013.

\bibitem[Pennington and Bahri(2017)]{pennington2017geometry}
Jeffrey Pennington and Yasaman Bahri.
\newblock Geometry of neural network loss surfaces via random matrix theory.
\newblock In \emph{Proceedings of the 34th International Conference on Machine
  Learning-Volume 70}, pages 2798--2806. JMLR. org, 2017.

\bibitem[Pennington and Worah(2018)]{pennington2018spectrum}
Jeffrey Pennington and Pratik Worah.
\newblock The spectrum of the fisher information matrix of a
  single-hidden-layer neural network.
\newblock In \emph{Advances in Neural Information Processing Systems}, pages
  5410--5419, 2018.

\bibitem[Pfahler and Morik(2019)]{pfahler2019evolution}
Lukas Pfahler and Katharina Morik.
\newblock Evolution of eigenvalue decay in deep networks.
\newblock 2019.

\bibitem[Ramasesh et~al.(2020)Ramasesh, Dyer, and Raghu]{ramasesh2020anatomy}
Vinay~V Ramasesh, Ethan Dyer, and Maithra Raghu.
\newblock Anatomy of catastrophic forgetting: Hidden representations and task
  semantics.
\newblock \emph{arXiv preprint arXiv:2007.07400}, 2020.

\bibitem[Sagun et~al.(2016)Sagun, Bottou, and LeCun]{sagun2016eigenvalues}
Levent Sagun, Leon Bottou, and Yann LeCun.
\newblock Eigenvalues of the hessian in deep learning: Singularity and beyond.
\newblock \emph{arXiv preprint arXiv:1611.07476}, 2016.

\bibitem[Sagun et~al.(2017)Sagun, Evci, Guney, Dauphin, and
  Bottou]{sagun2017empirical}
Levent Sagun, Utku Evci, V~Ugur Guney, Yann Dauphin, and Leon Bottou.
\newblock Empirical analysis of the hessian of over-parametrized neural
  networks.
\newblock \emph{arXiv preprint arXiv:1706.04454}, 2017.

\bibitem[Simonyan and Zisserman(2014)]{simonyan2014very}
Karen Simonyan and Andrew Zisserman.
\newblock Very deep convolutional networks for large-scale image recognition.
\newblock \emph{arXiv preprint arXiv:1409.1556}, 2014.

\bibitem[Srivastava et~al.(2014)Srivastava, Hinton, Krizhevsky, Sutskever, and
  Salakhutdinov]{srivastava2014dropout}
Nitish Srivastava, Geoffrey Hinton, Alex Krizhevsky, Ilya Sutskever, and Ruslan
  Salakhutdinov.
\newblock Dropout: a simple way to prevent neural networks from overfitting.
\newblock \emph{The Journal of Machine Learning Research}, 15\penalty0
  (1):\penalty0 1929--1958, 2014.

\bibitem[Thomas et~al.(2019)Thomas, Pedregosa, van Merri{\"e}nboer, Mangazol,
  Bengio, and Roux]{thomas2019information}
Valentin Thomas, Fabian Pedregosa, Bart van Merri{\"e}nboer, Pierre-Antoine
  Mangazol, Yoshua Bengio, and Nicolas~Le Roux.
\newblock Information matrices and generalization.
\newblock \emph{arXiv preprint arXiv:1906.07774}, 2019.

\bibitem[Turek(1988)]{turek1988maximum}
I~Turek.
\newblock A maximum-entropy approach to the density of states within the
  recursion method.
\newblock \emph{Journal of Physics C: Solid State Physics}, 21\penalty0
  (17):\penalty0 3251, 1988.

\bibitem[Ubaru et~al.(2017)Ubaru, Chen, and Saad]{ubaru2017fast}
Shashanka Ubaru, Jie Chen, and Yousef Saad.
\newblock Fast estimation of tr(f(a)) via stochastic lanczos quadrature.
\newblock \emph{SIAM Journal on Matrix Analysis and Applications}, 38\penalty0
  (4):\penalty0 1075--1099, 2017.

\bibitem[Van~Loan and Pitsianis(1993)]{van1993approximation}
Charles~F Van~Loan and Nikos Pitsianis.
\newblock Approximation with kronecker products.
\newblock In \emph{Linear algebra for large scale and real-time applications},
  pages 293--314. Springer, 1993.

\bibitem[Verma et~al.(2018)Verma, Lamb, Beckham, Najafi, Mitliagkas, Courville,
  Lopez-Paz, and Bengio]{verma2018manifold}
Vikas Verma, Alex Lamb, Christopher Beckham, Amir Najafi, Ioannis Mitliagkas,
  Aaron Courville, David Lopez-Paz, and Yoshua Bengio.
\newblock Manifold mixup: Better representations by interpolating hidden
  states.
\newblock \emph{arXiv preprint arXiv:1806.05236}, 2018.

\bibitem[Wang et~al.(2019)Wang, Grosse, Fidler, and Zhang]{wang2019eigendamage}
Chaoqi Wang, Roger Grosse, Sanja Fidler, and Guodong Zhang.
\newblock Eigendamage: Structured pruning in the kronecker-factored eigenbasis.
\newblock \emph{arXiv preprint arXiv:1905.05934}, 2019.

\bibitem[Wang and Zhang(2019)]{wang2019utilizing}
Jialei Wang and Tong Zhang.
\newblock Utilizing second order information in minibatch stochastic variance
  reduced proximal iterations.
\newblock \emph{J. Mach. Learn. Res.}, 20:\penalty0 42--1, 2019.

\bibitem[Wheeler and Blumstein(1972)]{wheeler1972modified}
John~C Wheeler and Carl Blumstein.
\newblock Modified moments for harmonic solids.
\newblock \emph{Physical Review B}, 6\penalty0 (12):\penalty0 4380, 1972.

\bibitem[Wu et~al.(2017{\natexlab{a}})Wu, Mansimov, Grosse, Liao, and
  Ba]{wu2017scalable}
Yuhuai Wu, Elman Mansimov, Roger~B Grosse, Shun Liao, and Jimmy Ba.
\newblock Scalable trust-region method for deep reinforcement learning using
  kronecker-factored approximation.
\newblock In \emph{Advances in neural information processing systems}, pages
  5279--5288, 2017{\natexlab{a}}.

\bibitem[Wu et~al.(2017{\natexlab{b}})Wu, Mansimov, Grosse, Liao, and
  Ba]{wu2017second}
Yuhuai Wu, Elman Mansimov, Roger~B Grosse, Shun Liao, and Jimmy Ba.
\newblock Second-order optimization for deep reinforcement learning using
  kronecker-factored approximation.
\newblock In \emph{NIPS}, 2017{\natexlab{b}}.

\bibitem[Xiao et~al.(2017)Xiao, Rasul, and Vollgraf]{xiao2017fashion}
Han Xiao, Kashif Rasul, and Roland Vollgraf.
\newblock Fashion-mnist: a novel image dataset for benchmarking machine
  learning algorithms.
\newblock \emph{arXiv preprint arXiv:1708.07747}, 2017.

\bibitem[Xu et~al.(2019)Xu, Roosta, and Mahoney]{xu2019newton}
Peng Xu, Fred Roosta, and Michael~W Mahoney.
\newblock Newton-type methods for non-convex optimization under inexact hessian
  information.
\newblock \emph{Mathematical Programming}, pages 1--36, 2019.

\bibitem[Xu et~al.(2020)Xu, Roosta, and Mahoney]{xu2020second}
Peng Xu, Fred Roosta, and Michael~W Mahoney.
\newblock Second-order optimization for non-convex machine learning: An
  empirical study.
\newblock In \emph{Proceedings of the 2020 SIAM International Conference on
  Data Mining}, pages 199--207. SIAM, 2020.

\bibitem[Yao et~al.(2019)Yao, Gholami, Keutzer, and Mahoney]{yao2019pyhessian}
Zhewei Yao, Amir Gholami, Kurt Keutzer, and Michael Mahoney.
\newblock Pyhessian: Neural networks through the lens of the hessian.
\newblock \emph{arXiv preprint arXiv:1912.07145}, 2019.

\bibitem[Ye et~al.(2018)Ye, Huang, Fang, Li, and Zhang]{ye2018hessian}
Haishan Ye, Zhichao Huang, Cong Fang, Chris~Junchi Li, and Tong Zhang.
\newblock Hessian-aware zeroth-order optimization for black-box adversarial
  attack.
\newblock \emph{arXiv preprint arXiv:1812.11377}, 2018.

\bibitem[Zhang et~al.(2019)Zhang, Bengio, and Singer]{zhang2019all}
Chiyuan Zhang, Samy Bengio, and Yoram Singer.
\newblock Are all layers created equal?
\newblock \emph{arXiv preprint arXiv:1902.01996}, 2019.

\end{thebibliography}

\end{document}